\documentclass[twoside]{article}
\pdfoutput=1
\usepackage[accepted]{aistats2023_old}
\usepackage[postaistats]{optional}
\usepackage[colorlinks=True,citecolor=blue,urlcolor=blue,pagebackref=true,backref=true,linkcolor=blue]{hyperref}

\renewcommand*{\backrefalt}[4]{%
    \ifcase #1 \footnotesize{(Not cited.)}%
    \or        \footnotesize{(Cited on page~#2.)}%
    \else      \footnotesize{(Cited on pages~#2.)}%
    \fi}

\usepackage[round]{natbib}

\usepackage{amsmath,amsfonts,bm}

\def\eqref#1{equation~\ref{#1}}

\def\ceil#1{\lceil #1 \rceil}
\def\floor#1{\lfloor #1 \rfloor}
\def\1{\bm{1}}

\def\eps{{\epsilon}}

\DeclareMathAlphabet{\mathsfit}{\encodingdefault}{\sfdefault}{m}{sl}
\SetMathAlphabet{\mathsfit}{bold}{\encodingdefault}{\sfdefault}{bx}{n}

\newcommand{\E}{\mathbb{E}}

\newcommand{\R}{\mathbb{R}}

\newcommand{\Var}{\mathrm{Var}}

\DeclareMathOperator*{\argmin}{arg\,min}

\usepackage{amssymb,amsthm}

\usepackage{url,nicefrac}
\usepackage[capitalize]{cleveref}  
\usepackage{autonum}
\usepackage[ruled,algo2e]{algorithm2e}%
\usepackage{enumitem}

\usepackage[usestackEOL]{stackengine}
\usepackage{booktabs} %
\usepackage{graphicx} %
\usepackage{mathtools}
\usepackage{etoc} %

\usepackage{thmtools,thm-restate}
\usepackage[capitalize]{cleveref}  
\usepackage{autonum}

\newtheorem{mytheorem}{Theorem}

\newtheorem{myproposition}{Proposition}

\newtheorem{mycorollary}{Corollary}

\newtheorem{mylemma}{Lemma}

\newtheorem{mydefinition}{Definition}

\newtheorem{myremark}{Remark}

\crefname{appendix}{App.}{Apps.}
\crefname{subsubsubappendix}{App.}{Apps.}
\crefname{equation}{}{}
\crefname{lemma}{Lem.}{Lems.}
\crefname{theorem}{Thm.}{Thms.}
\crefname{Corollary}{Cor.}{Cors.}
\crefname{algorithm}{Alg.}{Algs.}
\crefname{algocf}{Alg.}{Algs.} %
\crefname{section}{Sec.}{Secs.}
\crefname{table}{Tab.}{Tabs.}
\crefname{remark}{Rem.}{Rems.}
\crefname{example}{Ex.}{Exs.}
\crefname{definition}{Def.}{Defs.}
\crefname{Proposition}{Prop.}{Props.}
\crefname{mytheorem}{Thm.}{Thms.}
\crefname{myremark}{Rem.}{Rems.}
\crefname{mylemma}{Lem.}{Lems.}
\crefname{mydefinition}{Def.}{Defs.}
\crefname{myproposition}{Prop.}{Props.}
\crefname{mycorollary}{Cor.}{Cors.}
\crefname{enumi}{}{}
\crefname{name}{}{} %
\Crefformat{footnote}{#2\footnotemark[#1]#3}
\DeclareMathAlphabet{\mathbsf}{OT1}{cmss}{bx}{n}%
\DeclareMathAlphabet{\mathssf}{OT1}{cmss}{m}{sl}%

\newcommand{\vsep}{\vspace{-2mm}}

\newcommand{\numperm}{\mc B}

\newcommand{\uptag}{\trm{up}}

\newcommand{\defeq}{\triangleeq}

\newcommand{\para}[1]{\noindent\tbf{#1}\ \ }

\newcommand{\set}[1]{\mc{#1}}
\newcommand{\ball}{\mathbb{B}}
\newcommand{\indicator}{\mbf 1}

\newcommand{\snorm}[1]{\Vert #1 \Vert}
\newcommand{\sinfnorm}[1]{\snorm{#1}_\infty}

\renewcommand{\natural}{\mbb{N}}

\newcommand{\real}{\ensuremath{\mathbb{R}}}

\newcommand{\sless}[1]{\stackrel{#1}{\leq}}

\newcommand{\seq}[1]{\stackrel{#1}{=}}

\newcommand{\x}{x}
\newcommand{\y}{y}
\newcommand{\z}{z}
\renewcommand{\l}{\ell}
\newcommand{\vareps}{\varepsilon}

\DeclareMathOperator{\mmd}{MMD}
\def\Matern{Mat\'ern\xspace}

\newcommand{\kernel}{\mbf{k}}
\newcommand{\hkernel}{\mbf{h}}

\newcommand{\knorm}[1]{\Vert{#1}\Vert_{\kernel}}
\newcommand{\kinfsin}[1][\kernel]{\Vert{#1}\Vert_{\infty,\mrm{in}}}
\newcommand{\ksinfsin}{\kinfsin[\kersplit]}

\newcommand{\tail}[1][\kernel]{\tau_{#1}}
\newcommand{\klip}[1][\kernel]{L_{#1}}
\newcommand{\rttag}{\mrm{rt}}
\newcommand{\ksqrt}[1][\kernel]{{#1}_{\rttag}}

\newcommand{\ltwonorm}[1]{\norm{#1}_{2}}

\newcommand{\sgparam}[1][i]{\sigma_{#1}}
\newcommand{\vmax}[1][i]{\mathfrak{b}_{#1}}

\newcommand{\coreset}[1][j]{\mathcal{S}^{(#1)}}

\newcommand{\brackets}[1]{\left[ #1 \right]}

\newcommand{\parenth}[1]{\left( #1 \right)}

\newcommand{\bigparenth}[1]{\big( #1 \big)}

\newcommand{\braces}[1]{\left\{ #1 \right \}}
\newcommand{\sbraces}[1]{\{ #1  \}}
\newcommand{\abss}[1]{\left| #1 \right |}

\newcommand{\angles}[1]{\left\langle #1 \right \rangle}

\newcommand{\inv}{^{-1}}

\newcommand{\lam}{\lambda}
\newcommand{\Lam}{\Lambda}

\def\balign#1\ealign{\begin{align}#1\end{align}}
\def\baligns#1\ealigns{\begin{align*}#1\end{align*}}
\def\balignat#1\ealign{\begin{alignat}#1\end{alignat}}
\def\balignats#1\ealigns{\begin{alignat*}#1\end{alignat*}}
\def\bitemize#1\eitemize{\begin{itemize}#1\end{itemize}}
\def\benumerate#1\eenumerate{\begin{enumerate}#1\end{enumerate}}

\newenvironment{talign*}
 {\csname align*\endcsname}
 {\endalign}
\newenvironment{talign}
 {\csname align\endcsname}
 {\endalign}

\def\balignst#1\ealignst{\begin{talign*}#1\end{talign*}}
\def\balignt#1\ealignt{\begin{talign}#1\end{talign}}
\newcommand{\qtext}[1]{\quad\text{#1}\quad} 
\newcommand{\stext}[1]{\text{ #1 }} 

\let\originalleft\left
\let\originalright\right
\renewcommand{\left}{\mathopen{}\mathclose\bgroup\originalleft}
\renewcommand{\right}{\aftergroup\egroup\originalright}

\def\Matern{Mat\'ern\xspace}

\def\tinycitep*#1{{\tiny\citep*{#1}}}
\def\tinycitealt*#1{{\tiny\citealt*{#1}}}
\def\tinycite*#1{{\tiny\cite*{#1}}}
\def\smallcitep*#1{{\scriptsize\citep*{#1}}}
\def\smallcitealt*#1{{\scriptsize\citealt*{#1}}}
\def\smallcite*#1{{\scriptsize\cite*{#1}}}

\def\mbf#1{\mathbf{#1}}
\def\mbb#1{\mathbb{#1}}
\def\mc#1{\mathcal{#1}}
\def\mrm#1{\mathrm{#1}}
\def\trm#1{\textrm{#1}}
\def\tbf#1{\textbf{#1}}

\def\wtil#1{\widetilde{#1}}
\def\mfk#1{\mathfrak{#1}}

\def\reals{\mathbb{R}} %
\def\R{\mathbb{R}}
\def\Q{\mathbb{Q}}
\def\N{\mathbb{N}}
\def\<{\left\langle} %
\def\>{\right\rangle}

\def\iff{\Leftrightarrow}

\def\defeq{\triangleq} %
\def\half{\frac{1}{2}}
\def\quarter{\frac{1}{4}}

\newcommand{\textfrac}[2]{{\textstyle\frac{#1}{#2}}}
\def\norm#1{\left\|{#1}\right\|} %
\newcommand{\twonorm}[1]{\norm{#1}_2} %
\newcommand{\infnorm}[1]{\norm{#1}_{\infty}} %
\newcommand{\opnorm}[1]{\norm{#1}_{\mathrm{op}}} %
\def\staticnorm#1{\|{#1}\|} %
\newcommand{\staticinfnorm}[1]{\staticnorm{#1}_\infty} %
\def\indic#1{\mbb{I}\left[{#1}\right]} %
\def\E{\mbb{E}} %

\def\P{\mbb{P}} %

\def\Var{\mrm{Var}} %

\def\indep{\perp\!\!\!\perp} %
\newcommand{\iid}{\textrm{i.i.d.}\ }
\newcommand{\dist}{\sim}
\newcommand{\distiid}{\overset{\textup{\tiny\iid}}{\dist}}

\providecommand{\argmin}{\mathop\mathrm{arg min}}

\def\supp#1{\mathrm{supp}({#1})}

\ifdefined\nonewproofenvironments\else
\ifdefined\ispres\else
\newenvironment{proof-sketch}{\noindent\textbf{Proof Sketch}
  \hspace*{1em}}{\qed\bigskip\\}
\newenvironment{proof-idea}{\noindent\textbf{Proof Idea}
  \hspace*{1em}}{\qed\bigskip\\}
\newenvironment{proof-of-lemma}[1][{}]{\noindent\textbf{Proof of Lemma {#1}}
  \hspace*{1em}}{\qed\\}
\newenvironment{proof-of-theorem}[1][{}]{\noindent\textbf{Proof of Theorem {#1}}
  \hspace*{1em}}{\qed\\}
\newenvironment{proof-attempt}{\noindent\textbf{Proof Attempt}
  \hspace*{1em}}{\qed\bigskip\\}

\newcommand{\cset}{\mc{S}}

\newcommand{\inputcoreset}{\cset_{\mrm{in}}}
\newcommand{\outputcoreset}{\cset_{\mrm{out}}}
\newcommand{\ktcoreset}{\cset_{\mrm{KT}}}

\newcommand{\rmin}{\mfk{S}}
\newcommand{\rminpn}[1][\inputcoreset]{\rmin_{#1}}

\newcommand{\shiftparam}{a}
\newcommand{\kgaussparam}{v}

\newcommand{\err}{\mathfrak{N}}%

\newcommand{\kersplit}[1][]{\kernel_{{#1}}'}%

\newcommand{\kt}{\textsc{KT}\xspace}

\newcommand{\ktsplit}{\hyperref[algo:ktsplit]{\color{black}{\textsc{kt-split}}}\xspace}

\newcommand{\ktswap}{\hyperref[algo:ktswap]{\color{black}{\textsc{kt-swap}}}\xspace}

\newcommand{\halve}{\textsc{Halve}\xspace}
\newcommand{\osname}{compression level\xspace}
\newcommand{\ossymb}{\ensuremath{\mfk{g}}\xspace}
\newcommand{\compress}{\textsc{Compress}\xspace}
\newcommand{\ktcompress}{\textsc{KT-Compress}\xspace}
\newcommand{\compresssub}{\textsc{C}}

\newcommand{\genmat}{\mbf{M}}

\newcommand{\xs}[1][m]{\mbb{X}_{#1}}
\newcommand{\ys}[1][n]{\mbb{Y}_{#1}}

\newcommand{\nullhypo}{\mc H_{0}}
\newcommand{\althypo}[1][1]{\mc H_{#1}}
\newcommand{\threshold}[1][\alpha]{t_{#1}}

\newcommand{\block}{B}
\newcommand{\mmdblock}{\mmd_{\block}}
\newcommand{\mmdincomplete}{\mmd_{\trm{inc}}}

\newcommand{\hatxs}[1][m]{\hat{\mathbb{X}}_{#1}}
\newcommand{\hatys}{\hat{\mathbb{Y}}_n}
\newcommand{\hatzs}[1][]{\hat{\mbb{Z}}_{#1}}

\newcommand{\ctt}{\textup{CTT}\xspace}
\newcommand{\deltactt}{\Delta_{\ctt}}
\newcommand{\actt}{\textup{ACTT}\xspace}
\newcommand{\deltaactt}{\Delta_{\actt}}
\newcommand{\lctt}{\textup{LR-CTT}\xspace}
\newcommand{\feat}[1][r]{\Phi_{#1}} %
\newcommand{\featmmd}[1][r]{\mmd_{\feat[#1]}} %
\newcommand{\featerr}[1][r]{\eps_{\feat[#1]}} %
\newcommand{\deltalctt}{\Delta_{\lctt}}

\newcommand{\cttname}{{Compress Then Test}\xspace}
\newcommand{\acttname}{{Aggregated \ctt}\xspace}

\newcommand{\lcttname}{{Low-Rank \ctt}\xspace}
\newcommand{\tmmd}{\textsc{CoresetMMD}\xspace}
\newcommand{\ltmmd}{\textup{LR}-\textsc{CoresetMMD}\xspace}

\newcommand{\nout}{n_{\mrm{out}}}

\newcommand{\ncref}[1]{\cref{#1}: \nameref*{#1}} %

\newcommand{\pcref}[1]{Proof of \ncref{#1}} %

\newcommand{\error}[1][\kernel]{\mathbsf{R}_{#1}}
\newcommand{\errorsplit}[1][]{\error[{\kernel_{#1}},{\kersplit[#1]}]}

\newcommand{\sblock}{s} %
\newcommand{\speedup}{\ensuremath{200\times}\xspace}

\newcommand{\speeduptwenty}{\ensuremath{20\times}\xspace}
\newcommand{\klam}[1][\lam]{\kernel_{#1}}
\newcommand{\kslam}[1][\lam]{{\kersplit[#1]}}

\newcommand{\aggtag}{\textsc{agg}}
\newcommand{\ktagg}{\ktcompress-\aggtag\xspace}
\newcommand{\kagg}{\kersplit}
\newcommand{\kagginf}{\kinfsin[{\kagg}]}
\newcommand{\epsagg}{\vareps_{\aggtag}}
\newcommand{\ckk}[1][\kernel,\kersplit]{C_{#1}}
\newcommand{\mkk}[1][\kernel,\kersplit]{\mathfrak{M}_{#1}}
\newcommand{\ackk}{C_{\kernel_{\lam},\kersplit}}
\newcommand{\amkk}{\mfk{M}_{\kernel_{\lam},\kersplit}}

\newcommand{\cnew}{\mfk{a}}

\newcommand{\opthalvefour}{\textsc{OptHalve4}\xspace}

\newcommand{\cttalg}{
  \begin{algorithm2e}[h]
\caption{\cttname, $\deltactt$}
\label{algo:ctt}
\SetAlgoLined
  \DontPrintSemicolon
\small{
  \KwIn{Samples ($\xs$, $\ys$),  
  \# coresets $\sblock$,
  \osname $\ossymb$,
  kernels ($\kernel,\kersplit$),
  failure prob.~$\delta$,  \# replicates $\numperm$, level $\alpha$} 
  \BlankLine
  Partition $\xs$ into $\sblock_m =  \frac{\sblock m}{m+n}$ equal-sized bins 
  $ ( \xs^{(i)} )_{i=1}^{\sblock_m}$ \\
  Partition $\ys$ into $\sblock_n =  \frac{\sblock n}{m+n}$ equal-sized bins 
  $ ( \ys^{(i)} )_{i=1}^{\sblock_n}$ \\
  \BlankLine
  // Identify coreset of size
  $2^\ossymb\sqrt{\frac{m+n}{\sblock}}$
  for each bin\\
  \For{$i=1, \dots, \sblock_m$}
  {$\hatxs^{(i)} \leftarrow \ktcompress( \xs^{(i)}, \ossymb, \kernel, \kersplit, \delta)$}
  \For{$i=1, \dots, \sblock_n$}
  {$\hatys^{(i)} \leftarrow \ktcompress( \ys^{(i)}, \ossymb, \kernel, \kersplit, \delta)$}
    \BlankLine
    // Compute \tmmd test statistic \\

\makebox[\linewidth]{$M_{\numperm+1} \gets \mmd_{\kernel}(\hatxs,\hatys)\qtext{for}$  \hfill\refstepcounter{equation}\llap{(\theequation)} \label{tmmd}} \\[.2\baselineskip]
$\hatxs \coloneqq\textsc{Concat}((\hatxs^{(i)})_{i=1}^{\sblock_m} )\qtext{and} \hatys \coloneqq\textsc{Concat}((\hatys^{(i)})_{i=1}^{\sblock_n} )$\\
    \BlankLine
    // Simulate null by randomly permuting the $\sblock$ coresets $\numperm$ times \\ %
    \For{$b=1,\dots,\numperm$}
    {
    ($\hatxs^b, \hatys^b) \gets \textsc{PermuteCoresets}(\hatxs, \hatys, \sblock)$
    
    $M_b \!\gets\! \mmd_{\kernel}(\hatxs^{b}, \hatys^{b})$ 
    }
    \BlankLine
    // Threshold test statistic\\
    $R \gets$ position of $M_{\numperm+1}$ in an increasing ordering of $(M_b)_{b=1}^{\numperm+1}$ with ties broken uniformly at random
    
    \lIf{  $R > b_{\alpha} \coloneqq \ceil{(1\!-\!\alpha)(\numperm\!+\!1)}$ }{
        \KwRet{ \textup{ 1 \,\, // reject null}}
    }
    \lElseIf{
    $R < b_{\alpha}$}{
        \KwRet{ \textup{ 0 \qquad\qquad\qquad\quad// accept null} }
    }
    \lElse{
        \KwRet{ %
        \textup{1 with prob.\  
        $p_\alpha = b_\alpha - (1-\alpha)(\numperm+1)$ or else $0$}%
        }
    }
    }
\end{algorithm2e}}

\newcommand{\lcttalg}{
  \begin{algorithm2e}[h]
\caption{\lcttname, $\deltalctt$}
\label{algo:lctt}
\SetAlgoLined
  \DontPrintSemicolon
\small{
  \KwIn{Samples ($\xs$, $\ys$),  
  coreset size factor $a$, 
  \osname $\ossymb$,
  kernels ($\kernel,\kersplit$),
  feature map $\feat$, 
  failure prob.~$\delta$,  
  \# permutation bins $\sblock$, \#~replicates $\numperm$, level $\alpha$} 
  \BlankLine
  Partition $\xs$ into 
  $\sblock_{m,r} = \frac{4^{\ossymb} a^2 m}{r^2}$
  equal-sized bins 
  $ ( \xs^{(i)} )_{i=1}^{\sblock_{m,r}}$ \\
  Partition $\ys$ into 
  $\sblock_{n,r} = \frac{4^{\ossymb} a^2 n}{r^2}$
  equal-sized bins 
  $ ( \ys^{(i)} )_{i=1}^{\sblock_{n,r}}$ \\
  \BlankLine
  // Identify coreset of size
  $\frac{r}{a}$
  for each bin\\
  \For{$i=1, \dots, \sblock_{m,r}$}
  {$\hatxs^{(i)} \leftarrow \ktcompress( \xs^{(i)}, \ossymb, \kernel, \kersplit, \delta)$}
  \For{$i=1, \dots, \sblock_{n,r}$}
  {$\hatys^{(i)} \leftarrow \ktcompress( \ys^{(i)}, \ossymb, \kernel, \kersplit, \delta)$}
    \BlankLine
    // Compute \ltmmd test statistic \\

\makebox[\linewidth]{%
$M_{\numperm+1} 
    \gets 
\featmmd(\hatxs,\hatys)
\qtext{for}$  \hfill\refstepcounter{equation}\llap{(\theequation)} \label{ltmmd}} \\[.2\baselineskip]
$\hatxs \coloneqq\textsc{Concat}((\hatxs^{(i)})_{i=1}^{\sblock_{m,r}} )\stext{and} \hatys \coloneqq\textsc{Concat}((\hatys^{(i)})_{i=1}^{\sblock_{n,r}} )$\\
    \BlankLine
    // Simulate null by randomly permuting $\sblock$ coresets $\numperm$ times \\ %
    \For{$b=1,\dots,\numperm$}
    {
    ($\hatxs^b, \hatys^b) \gets \textsc{PermuteCoresets}(\hatxs, \hatys, \sblock)$
    
    $M_b \!\gets\! \featmmd(\hatxs^{b}, \hatys^{b})$ 
    }

    \BlankLine
    // Threshold test statistic\\
    $R \gets$ position of $M_{\numperm+1}$ in an increasing ordering of $(M_b)_{b=1}^{\numperm+1}$ with ties broken uniformly at random
    
    \lIf{  $R > b_{\alpha} \coloneqq \ceil{(1\!-\!\alpha)(\numperm\!+\!1)}$ }{
        \KwRet{ \textup{ 1 \,\, // reject null}}
    }
    \lElseIf{
    $R < b_{\alpha}$}{
        \KwRet{ \textup{ 0 \qquad\qquad\qquad\quad// accept null} }
    }
    \lElse{
        \KwRet{ %
        \textup{1 with prob.\  
        $p_\alpha = b_\alpha - (1-\alpha)(\numperm+1)$ or else $0$}%
        }
    }
    }
\end{algorithm2e}}

\newcommand{\acttalg}{%
  \begin{algorithm2e}[h]
\caption{\acttname, $\deltaactt$}
\label{algo:actt}
\SetAlgoLined
  \DontPrintSemicolon
\small{
  \KwIn{Samples ($\xs$, $\ys$),  
  \# coresets $\sblock$,
  \osname $\ossymb$,
  kernels $(\kernel_{\lambda},\kersplit[\lam])_{\lambda \in \Lambda}$, importance weights $(w_{\lambda})_{\lambda \in \Lambda}$,
  failure prob.~$\delta$,  \# replicates ($\numperm_1$, $\numperm_2$, $\numperm_3$), level $\alpha$}
  \BlankLine
  Partition $\xs$ into $\sblock_m =  \frac{\sblock m}{m+n}$ equal-sized bins 
  $ ( \xs^{(i)} )_{i=1}^{\sblock_m}$ \\
  Partition $\ys$ into $\sblock_n =  \frac{\sblock n}{m+n}$ equal-sized bins 
  $ ( \ys^{(i)} )_{i=1}^{\sblock_n}$ \\
  \BlankLine
  // Identify coresets of size
  $2^\ossymb\sqrt{\frac{m+n}{\sblock}}$
  using sum of kernels\\
  $\kernel \gets \sum_{\lam\in\Lambda} \kernel_\lam;\quad
  \kersplit \gets \sum_{\lam\in\Lambda} \kersplit[\lam]
  $\\
  \For{$i=1, \dots, \sblock_m$}
  {$\hatxs^{(i)} \leftarrow \ktcompress( \xs^{(i)}, \ossymb, \kernel, \kersplit, \delta)$}
  \For{$i=1, \dots, \sblock_n$}
  {$\hatys^{(i)} \leftarrow \ktcompress( \ys^{(i)}, \ossymb, \kernel, \kersplit, \delta)$}
  \BlankLine
    // Compute \tmmd for each candidate parameter $\lam$\\
    $\hatxs \coloneqq\textsc{Concat}((\hatxs^{(i)})_{i=1}^{\sblock_m} );\quad \hatys \coloneqq\textsc{Concat}((\hatys^{(i)})_{i=1}^{\sblock_n} )$\\
    \lFor{$\lambda \in \Lambda$}{\ $M_{\lambda} \gets \mmd_{\kernel_{\lam}}(\hatxs,\hatys)$}
  \BlankLine
    // Simulate null for each $\lam$ by randomly permuting $\sblock$ coresets\\
    \For{\textup{$\ell = 1,2$ \textbf{and} $b=1,\dots,\numperm_{\ell}$}}
    {
    ($\hatxs^{\ell,b}, \hatys^{\ell,b}) \gets \textsc{PermuteCoresets}(\xs, \ys, \sblock)$

    \lFor{$\lambda \in \Lambda$}{$M_{b,\lambda,\ell} \gets \mmd_{\kernel_{\lam}}(\hatxs^{b},\hatys^{b})$} %

    }
    \lFor{$\lambda \in \Lambda$}{\ Sort $(M_{b,\lambda,1})_{b=1}^{\numperm_1}$ increasingly into $(M_{(b),\lambda,1})_{b=1}^{\numperm_1}$}
    \BlankLine
    // Estimate largest rejection threshold for each $M_\lam$ statistic that ensures aggregated test size $\leq \alpha$ \\
    $u_{\min} \gets 0$ \textup{\textbf{and}} $u_{\max} \gets \min_{\lambda \in \Lambda} w_{\lambda}^{-1}$\\
    \For{$i=1,\dots,\numperm_3$}{
    $u \gets\! \frac{u_{\min}+u_{\max}}{2}$;\ \  \lFor{$\lam\in\Lam$}{$b_{u,\lam}\! \gets\! \ceil{(\numperm_1\!+\!1)(1\!-\!u w_{\lambda})}$}
    $P_u \gets\!\frac{1}{\numperm_2}\!\sum_{b=1}^{\numperm_2} \!\indicator[\max_{\lambda \in \Lambda} (M_{b,\lambda,2}\!-\!M_{(b_{u,\lam}),\lambda,1})\!>\!0]$
    
    \lIf{$P_u \leq \alpha$}{$u_{\min} \gets u$ \textup{\textbf{else}} $u_{\max} \gets u$}
    }
    // Reject null if any test statistic $M_\lam$ exceeds its threshold\\
    $\hat{u}_{\alpha} \gets u_{\min}$;\quad \lFor{$\lam\in\Lam$}{$b'_{\alpha,\lam}\gets \ceil{(\numperm_1+1)(1-\hat{u}_{\alpha} w_{\lambda})}$}
    \lIf{$M_{\lambda} > M_{(b'_{\alpha,\lam}),\lambda,1}$ \textbf{\textup{for some}} $\lambda \in \Lambda$}{\KwRet{\textup{1 (reject null)}}}\lElse{\KwRet{\textup{0 (accept null)} }}
    }

\end{algorithm2e}}
\newcommand{\gaussfig}{
\begin{figure*}[t!]
    \centering
    \begin{tabular}{cc}
    \includegraphics[width=0.48\textwidth]{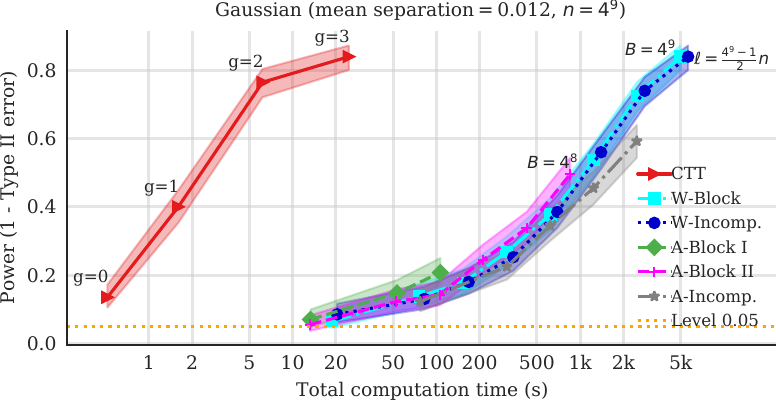}
    &%
    \includegraphics[width=0.48\textwidth]{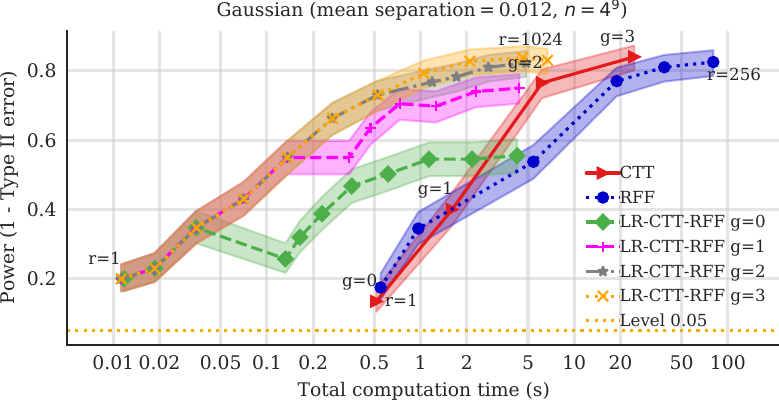}
    \\ \vsep
    \includegraphics[width=0.48\textwidth]{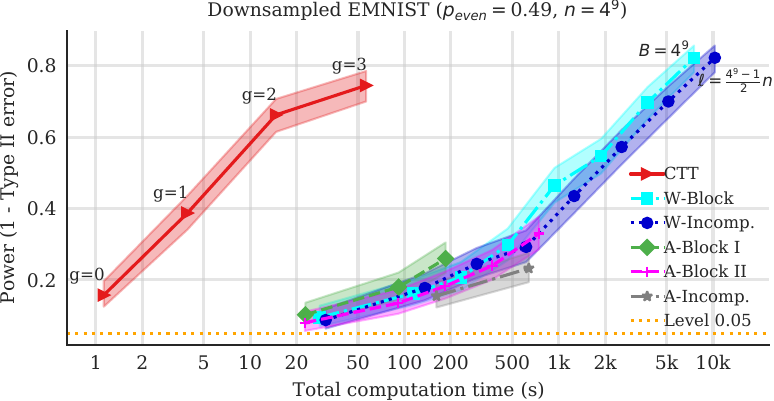}
    &
    \includegraphics[width=0.48\textwidth]{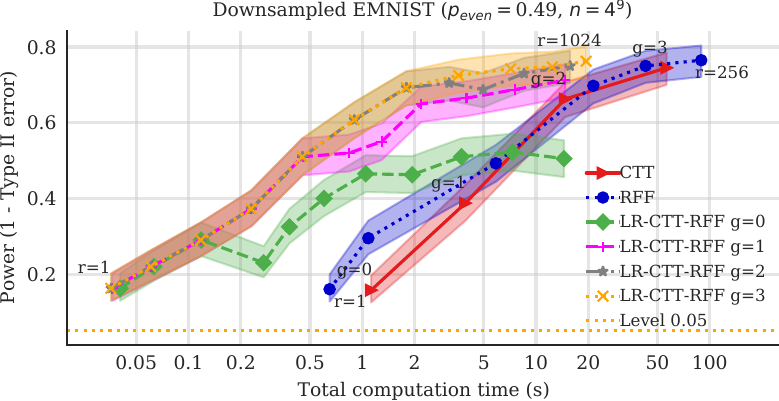}
    \end{tabular}
    \caption{\tbf{Time-power trade-off curves} in the \textsc{Gaussian} and \textsc{EMNIST} experimental settings comparing 
    (\emph{left}) \ctt to five state-of-the-art approximate MMD tests based on subsampling 
    and (\emph{right}) \lctt to the state-of-the-art low-rank MMD test based on random Fourier features (RFF).
    }
    \label{fig:gaussians_EMNIST}
\end{figure*}
}

\newcommand{\blobs}{

\begin{figure}[h!]
    \centering
    \includegraphics[width=0.48\textwidth]{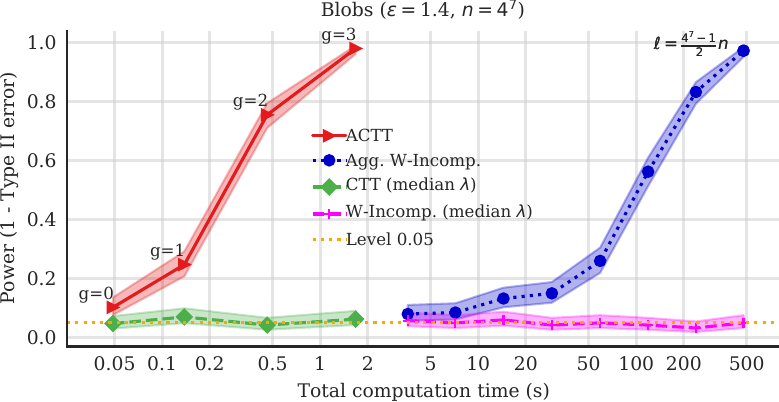}
    \includegraphics[width=0.48\textwidth]{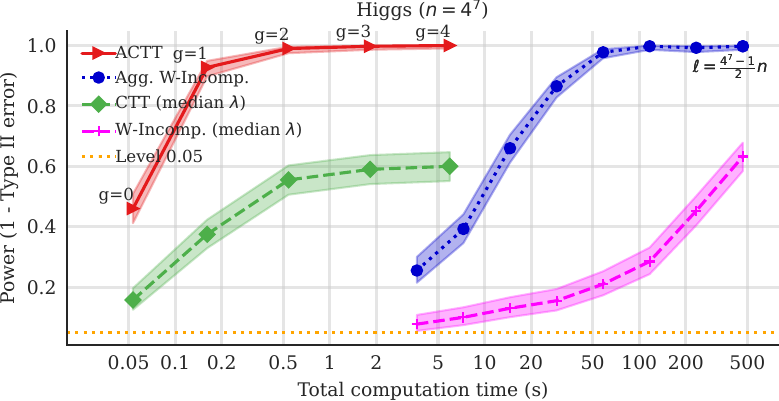}
    \includegraphics[width=0.48\textwidth]{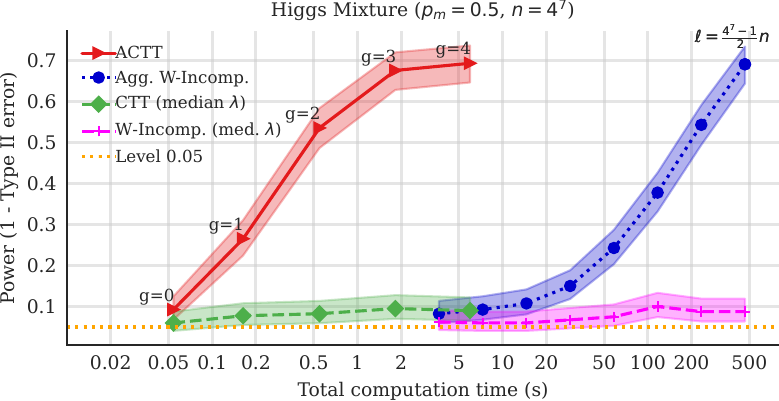}
    \caption{\tbf{Time-power trade-off curves} 
     for \actt and the state-of-the-art incomplete MMD aggregation test in the  {\textsc{Blobs}} and \textsc{Higgs} experimental settings.
    }
    \label{fig:blobs}
\end{figure}}
\graphicspath{{../AISTATS/images/},{AISTATS/images/},{images/}}

\begin{document}
\makeatletter
\patchcmd{\@algocf@start}%
  {-1.5em}%
  {0pt}%
  {}{}%
\makeatother

\setlength{\algomargin}{0pt}

\twocolumn[\aistatstitle{Compress Then Test: Powerful Kernel Testing in Near-linear Time}
\aistatsauthor{Carles Domingo-Enrich \And Raaz Dwivedi \And  Lester Mackey}
\aistatsaddress{Courant Institute of Mathematical Sciences\\ NYU\\ \texttt{cd2754@nyu.edu} \And  Harvard University\\ MIT\\ \texttt{raaz@mit.edu} \And Microsoft Research New England\\ \texttt{lmackey@microsoft.com}}]

\etocdepthtag.toc{mtchapter}
\etocsettagdepth{mtchapter}{subsection}
\etocsettagdepth{mtappendix}{none}

\begin{abstract}
Kernel two-sample testing provides a powerful framework for distinguishing any pair of distributions based on $n$ sample points. However, existing kernel tests either run in $n^2$ time or sacrifice undue power to improve runtime. To address these shortcomings, we introduce Compress Then Test (CTT), a new framework for high-powered  kernel testing based on sample compression. CTT cheaply approximates an expensive test by compressing each $n$ point sample into a small but provably high-fidelity coreset. For standard kernels and subexponential distributions, CTT inherits the statistical behavior of a quadratic-time test---recovering the same optimal detection boundary---while running in near-linear time. We couple these advances with cheaper permutation testing, justified by new power analyses; 
improved time-vs.-quality guarantees for low-rank approximation; and a fast aggregation procedure for identifying especially discriminating kernels. In our experiments with real and simulated data, CTT and its extensions provide 20--200x speed-ups over state-of-the-art approximate MMD tests with no loss of power.
\end{abstract}

\section{Introduction}
\label{sec:intro}
Kernel two-sample tests based on the maximum mean discrepancy \citep[MMD,][]{gretton2012akernel} can distinguish any pair of distributions given only a sufficiently large sample from each.
However, standard MMD tests have prohibitive running times that scale quadratically in the sample size $n$.
\citet{gretton2012akernel,zaremba2013btest,yamada2018post,schrab2022efficient} introduced faster approximate MMD tests based on subsampling, but each suffers from a fundamental time-quality trade-off barrier: for any pair of distributions, quadratic time is required to match the discrimination power of a standard MMD test (see \cref{thm:rates_other_tests}).
Our first contribution is a new subsampling approach called {\cttname} (\ctt) that accelerates testing by first compressing each sample.
In \cref{sec:ctt}, we prove that this approach pierces the aforementioned barrier, matching the quality of quadratic-time tests in near-linear time for subexponential distributions.
Along the way, we develop refined analyses of permutation tests, establishing their discriminating power even when permutations are restricted to preserve group structure and relatively few (e.g., $39$) permutations are employed.
In our experiments with both real and synthetic data, the CTT time-quality trade-off curves dominate those of state-of-the-art subsampling approaches, providing \speedup speed-ups.

\citet{zhao2015fastmmd} introduced an alternative, {low-rank} approach to fast approximate MMD testing that replaces the target kernel with a $\Theta(nr)$ time approximation based on $r$ random Fourier features \citep[RFFs,][]{rahimi2008random}.
This method often performs well in practice, but the guarantees of \citet{rahimi2008random,zhao2015fastmmd,sutherland2015error,sriperumbudur2015optimal} require $\Omega(n^2)$ random features and hence $\Omega(n^3)$ time to match the power of a standard MMD test.
By compressing before performing low-rank approximation, our second contribution, {\lcttname} (\lctt),  allows a user to harness any effective low-rank approximation without sacrificing the improved time-quality guarantees of CTT.  
In our experiments, this hybrid test offers the best performance of all, outpacing both the \ctt and RFF tests.

Finally, in the spirit of \citet{schrab2021mmd}, we develop {\acttname} (\actt) tests that improve power by rapidly identifying the most discriminating kernel in a collection of candidates.  In our experiments, \actt offers 100-\speedup speed-ups over the state-of-the-art efficient aggregated tests of \citet{schrab2022efficient}.

\section{Kernel Two-sample Testing}
\label{sub:two_sample}
As a standing assumption, suppose that we observe $\xs \defeq (X_i)_{i=1}^m$ and $\ys \defeq (Y_j)_{j=1}^n$, two independent sequences of datapoints drawn \iid from unknown probability measures $\P$ and $\Q$ respectively.
In two-sample testing, our goal is to decide whether the null hypothesis $\nullhypo: \P = \Q$ or the alternative hypothesis $\althypo : \P \neq \Q$ is correct.
A test $\Delta$ is a binary function of $\xs$ and $\ys$ such that the null hypothesis $\nullhypo$ is rejected if and only if $\Delta(\xs,\ys) = 1$. The \emph{size} or Type I error of the test is the probability that the null hypothesis is rejected when it is true, i.e., the probability $\mathrm{Pr}[\Delta(\xs,\ys) = 1]$ when $\P=\Q$. A test is said to have \emph{level} $\alpha \in (0,1)$ if its Type I error is bounded by $\alpha$ for all probability distributions, i.e.,  $\sup_{\P=\Q} %
\mathrm{Pr}[\Delta(\xs,\ys) = 1] \leq \alpha$. 
The Type II error of a test for a specific choice of $\P \neq \Q$ is the probability that the null hypothesis is accepted, i.e., $\mathrm{Pr}[\Delta(\xs,\ys) = 0]$. For a given level $\alpha$, our aim is to build a test with Type II error as small as possible for alternatives $\Q$ that are not too similar to $\P$. If  $\mathrm{Pr}[\Delta(\xs,\ys) = 0]=\beta$, we say that the test has \emph{power} $1-\beta$ against the alternative $\Q$.

Kernel two-sample tests are popular because they can distinguish any pair of distributions given sufficiently large samples and a {characteristic kernel} $\kernel$ \citep{gretton2012akernel}.
A \emph{characteristic kernel} is any positive-definite function $\kernel(x,y)$ \citep[Def.~4.15]{steinwart2008support} satisfying $\E_{X\sim\P}\kernel(X,x) \neq \E_{Y\sim\Q}\kernel(Y,x)$ for some $x$ whenever $\P\neq\Q$.
Common examples include Gaussian, \Matern, B-spline, inverse multiquadric (IMQ), sech, and Wendland's compactly supported kernels on $\R^d$ \citep{dwivedi2021kernel}. 
Kernel two-sample tests take the form 
$\Delta(\xs\!,\!\ys) \!\!=\!\! \indicator[T(\xs,\!\ys) \!\!>\!\! \threshold]$ where the test statistic $T(\xs,\!\ys)$ is an estimate of the squared 
\emph{maximum mean discrepancy} (MMD) between $\P$ and $\Q$, 
\begin{talign}
	\mmd^2_{\kernel}(\P, \Q)&\defeq %
	\E_{X, X'\distiid\P} \kernel(X,X') + \E_{Y, Y'\distiid\Q} \kernel(Y,Y') \\ 
	&\qquad - 2 \E_{X\sim\P\indep Y\sim\Q} \kernel(X,Y),
	\label{eq:kernel_mmd_distance}
\end{talign}
and $\threshold$ is a threshold chosen to ensure that the test has either finite-sample or asymptotic level $\alpha$.

\para{Quadratic-time or complete MMD tests}
The standard MMD test statistics 
defined in \citet{gretton2012akernel} each require $\Theta(m^2 + n^2)$ kernel evaluations and hence computation that grows quadratically in the sample sizes.
For example, \citet[Sec.~4.1]{gretton2012akernel} defines the squared sample MMD test statistic,
\begin{talign}
&\mmd^2_{\kernel}\!(\xs,\!\ys) 
    \!\defeq\! 
\frac{1}{m^2}\sum_{1\leq i, i' \leq m} \kernel(X_i, X_{i'})  \\
    &\ +\! 
\frac{1}{n^2}\sum_{1 \leq j, j' \leq n} \kernel(Y_j, Y_{j'})
    \!-\!\! 
\frac{2}{m n}\sum_{i=1}^{m}\!\! \sum_{j=1}^{n} \kernel(X_i,Y_j).
 \label{eq:biased}
\end{talign}
\citet[Lem.~6]{gretton2012akernel} also presents two unbiased estimators of $\mmd_{\kernel}^2(\P,\Q)$ as test statistics:
\begin{talign} 
&\mmd_u^2\!(\xs,\!\ys) \!\defeq\! \frac{\sum_{1\leq i \neq i' \leq m} \kernel(X_i,\! X_{i'}) }{m(m\!-\!1)} \\
&\qquad+\! \frac{\sum_{1 \leq j \neq j' \leq n} \kernel(Y_j, \!Y_{j'})}{n(n-1)} 
\!-\!\! \frac{2\sum_{i=1}^{m}\!\! \sum_{j=1}^{n} \kernel(X_i,\!Y_j)}{m n},
\label{eq:unbiased}
\\
 &\mmd^2_{\uptag}(\xs[n],\!\ys) \!\defeq\! \frac{\sum_{1 \leq i \neq j \leq n} \hkernel(X_i,\! X_{j},\! Y_i,\! Y_{j})}{n(n-1)}, \label{eq:unbiased_equal}
\end{talign}
where $\hkernel(x,x'\!,y,y')\!\! =\! \!\kernel(x,x') \!+\! \kernel(y,y') \!-\! \kernel(x,y') \!-\! \kernel(x',y)$. The estimator \cref{eq:unbiased_equal} differs from the estimator \cref{eq:unbiased} as it omits the diagonal cross-terms and is  defined only when $m=n$. 

\para{Block MMD tests} 
To improve computational cost through subsampling, \cite{zaremba2013btest} introduced block MMD tests, or \emph{B-tests} for short, that average $\frac{n}{\block}$ independent instances of the quadratic estimator \cref{eq:unbiased_equal}, each with sample size $\block$, i.e., 
\begin{talign} \label{eq:tauBdef}
&\mmdblock^2(\xs[n], \ys) \!\defeq\! \frac{\block}{n} \sum_{i=1}^{\frac{n}{\block}} \eta_i(\xs[n], \ys) \stext{with}\\
&\eta_i(\xs[n], \ys) \!\defeq\! \frac{1}{\block(\block-1)}\sum_{j,k=(i\!-\!1)\block+1, j \neq k}^{i\block} \hkernel(X_j, X_k, Y_j, Y_k).
\end{talign}
Consequently, the statistic computation takes time $\Theta(n \block)$. 
Moreover, when $\frac{n}{\block} \!\to\! \infty$, $\sqrt{\frac{n}{B}}\mmdblock^2$ has a Gaussian limit under the null that can be estimated to set $t_\alpha$. Previously, \citet[Sec.~6]{gretton2012akernel} studied a particular instantiation of this test with $\block \!=\! 2$.

\para{Incomplete MMD tests} 
\citet{yamada2018post} introduced an alternative $\Theta(\ell)$ time subsampling approximation based on \emph{incomplete} MMD test statistics, 
$\mmdincomplete^2(\xs[n],\ys)\! \defeq\! \frac{1}{\l} \sum_{(i,j) \in \mathcal{D}} \hkernel(X_i\!,X_j,\!Y_i,\!Y_j)$, with $\mathcal{D}$ a collection of $\l$ ordered index pairs.
\citet{yamada2018post} sampled pairs uniformly with replacement and set $t_\alpha$ using the Gaussian limit of $\sqrt{\l} \mmdincomplete^2$ as $\l\to\infty$.
\citet{schrab2022efficient} instead used  deterministically pre-selected index pairs and a wild bootstrap setting of $t_\alpha$ described below.

\para{Low-rank RFF tests} 
\citet{zhao2015fastmmd} proposed a complementary speed-up for MMD testing based on a low-rank MMD approximation of the form
\begin{talign}
&\mmd_{\feat}^2(\xs,\!\ys)
    \defeq \label{low-rank-mmd}\\
&\twonorm{\frac{1}{m}\sum_{i=1}^m\!\feat(x_i)\!-\!\frac{1}{n}\sum_{i=1}^n\!\feat(y_i)}^2
\end{talign}
where $\feat$ maps each sample point to an $r$-dimensional feature vector.
Specifically, \citeauthor{zhao2015fastmmd} chose $r$ RFFs to unbiasedly estimate $\mmd_{\kernel}^2$ in $\Theta((m+n)r)$ time.

\para{Permutation tests} 
For any of the aforementioned test statistics, one can alternatively set $t_\alpha$ using the following permutation approach to obtain a test with non-asymptotic level $\alpha$ \citep{romano2005exact,fromont2012kernels}.
Let $\mathbb{U}$ be the concatenation of $\xs$ and $\ys$.
For each permutation $\sigma$ of the indices $\{1,\dots,m+n\}$, define the permuted samples $\xs^{\sigma} = (U_{\sigma(i)})_{i=1}^m, \ys^{\sigma} = (U_{\sigma(m+j)})_{j=1}^n$ and the permuted statistic as $T^{\sigma} \defeq T(\xs^{\sigma},\ys^{\sigma})$. Sample $\numperm$ uniformly random permutations $(\sigma_b)_{b=1}^{\numperm}$ to obtain the values $T_b \defeq T^{\sigma_b}$ and sort them in increasing order $(T_{(b)})_{b=1}^{\numperm}$. Finally, set $t_\alpha = T_{(\ceil{(1-\alpha)(\numperm+1)})}$.

\para{Wild bootstrap tests}
Similarly, when $m=n$, the following wild bootstrap approach employed by 
\citet{fromont2012kernels} yields a non-asymptotic level $\alpha$ by exchangeability and \citet[Lem.~1]{romano2005exact}.
For each vector $\epsilon \in \{\pm 1\}^n$, define 
$T^{\epsilon} \defeq T(\mathbb{X}_n^{\epsilon},\ys^{\epsilon})$
where $\mathbb{X}_n^{\epsilon_{b}},\ys^{\epsilon_{b}}$ are constructed from $\xs[n]$ and $\ys[n]$ by swapping $X_i$ and $Y_i$ if $\epsilon_i = -1$.
Sample  $\numperm$ \iid vectors $(\epsilon_b)_{b=1}^{\numperm}$ uniformly from $\{\pm1\}^n$, compute the values $T_b \defeq T^{\epsilon_{b}}$, and finally set $t_\alpha$ as in the permutation approach.
\newcommand{\kinfnorm}{\staticinfnorm{\kernel}}
\newcommand{\ksplitinfnorm}{\staticinfnorm{\kersplit}}

\section{Compress Then Test}
\label{sec:ctt}

This section introduces \cttname (\ctt), a new framework for testing with sample compression.
\ctt relies on a new test statistic, \tmmd, that we describe and analyze in \cref{sub:tmmd}.
\cref{sub:ctt} then provides an analysis of the complete \ctt procedure detailed in \cref{algo:ctt}.
\vspace{-0mm}
\cttalg
\subsection{MMD compression with \tmmd}
\label{sub:tmmd}
At the heart of our testing strategy lies  \tmmd \cref{tmmd}, a new, inexpensive estimate for $\mmd_{\kernel}(\P, \Q)$ that builds atop the \ktcompress algorithm, a strategy introduced by \citet[Ex.~4]{shetty2022distribution} to compress a given point sequence (see \cref{sec:ktcompress} for background on \ktcompress). 
Given a coreset count $\sblock$, a target compression level $\ossymb$, and an auxiliary kernel function $\kersplit$ used by \ktcompress, \tmmd partitions each input sample into bins of size $\frac{m+n}{\sblock}$, compresses each bin into a smaller coreset of points using \ktcompress, concatenates the coresets to form the compressed approximations $\hatxs$ and $\hatys$ of size
$2^{\ossymb} m\sqrt{\frac{\sblock}{m+n}}$ and $2^{\ossymb} n\sqrt{\frac{\sblock}{m+n}}$ 
respectively, and finally computes the MMD estimate $\mmd_{\kernel}(\hatxs,\hatys)$.

As we show in \cref{proof-compression_guarantee}, this strategy offers the following strong approximation error guarantees, expressed in terms of the \ktcompress \emph{error inflation factor} $\errorsplit / 2^\ossymb$.

\begin{lemma}[Quality of \tmmd] \label{thm:compression_guarantee}
The 
\tmmd estimate~\cref{tmmd} %
satisfies\footnote{Unless otherwise specified, all of our results refer to an arbitrary setting of an algorithm's input arguments.}
\begin{talign} \label{eq:mmd_diff_x_y}
&
|\mmd_\kernel( \xs, \ys) \!- \! \mmd_\kernel( \hatxs,\hatys ) |
\\
&\quad\leq \! \frac{\errorsplit(\xs, \frac{m}{\sblock_m}, \delta,\ossymb)}{2^{\ossymb} \sqrt{m}} \!+  \! \frac{\errorsplit(\ys, \frac{n}{\sblock_n},\delta,\ossymb)}{2^{\ossymb} \sqrt{n}},
\end{talign}
with probability at least $1\!-\!\delta$ conditional on $(\xs,\ys)$, and
\begin{talign}
    \label{eq:mmd_diff_p_q}
    &\!\!\!|\mmd_\kernel(\P, \Q) \! - \! \mmd_\kernel( \hatxs, \hatys ) | 
    \\ 
    &\!\!\!\leq\! \frac{\errorsplit(\P, \frac{m}{\sblock_m}, \delta,\ossymb)}{2^{\ossymb} \sqrt{m}}\!\!+\! \!  \frac{\errorsplit(\Q, \frac{n}{\sblock_n},\delta,\ossymb)}{2^{\ossymb} \sqrt{n}}  \!\!+ \!c_{\delta} \bigparenth{\sqrt{\!\frac{\kinfnorm}{m}}\!\!+\!\!\sqrt{\!\frac{\kinfnorm}{n}}},
\end{talign}
with probability at least %
$1\!-\!3\delta$ for $c_{\delta}\!\defeq\! 2\!+\!\sqrt{2\log(\frac{%
2}{\delta}})$.
\end{lemma}
\begin{remark}[Beyond \iid data]
Our proof shows that the guarantee \cref{eq:mmd_diff_x_y} holds more generally for any point sequences $(\xs,\ys)$ generated independently of the randomness in \tmmd.
\end{remark}
\begin{remark}[Beyond \ktcompress]
    \tmmd and \ctt are compatible with any  compression scheme. In particular, when an  alternative compression algorithm is used in place of \ktcompress in \cref{algo:ctt}, the  conclusions of \cref{thm:compression_guarantee,thm:uniform_separation} can be straightforwardly generalized to accommodate the quality guarantees of that alternative.
\end{remark}

The first guarantee of \cref{thm:compression_guarantee} bounds the compression error introduced by substituting the compressed points $(\hatxs,\hatys)$ for $(\xs,\ys)$, while the second accounts also for the  $\Theta(\frac{1}{\sqrt{m}} \!+\! \frac{1}{\sqrt{n}})$  random fluctuations of the quadratic-time statistic $\mmd_\kernel(\xs,\ys)$ around the population estimand $\mmd_\kernel(\P,\Q)$ \citep{gretton2012akernel}.
In either case, we find that \tmmd offers an order $O(\frac{1}{\sqrt{m}} \!+\! \frac{1}{\sqrt{n}}$ approximation---the same order as the quadratic-time $\mmd_\kernel(\xs,\ys)$ estimate---up to the inflation factor $(1\!+\!{\errorsplit}/{2^\ossymb})$.

The value $\errorsplit(\P,m,\delta,\ossymb)$ depends on the choice of the auxiliary kernel $\kersplit$ and the tail decay of $\P$ (see \cref{sub:inflation_factor} for details).\footnote{The related value $\errorsplit(\xs,m,\delta,\ossymb)$ is $\errorsplit(\cdot,m,\delta,\ossymb)$ applied to the empirical distribution over $\xs$.}
Two standard choices for $\kersplit$ are the target kernel $\kernel$ itself~\citep{dwivedi2022generalized} or a \emph{square-root kernel} $\ksqrt$ satisfying $\kernel(x,y) = \int \ksqrt(x,z)\ksqrt(y,z) dz$.
As detailed in \citet{dwivedi2021kernel}, convenient square-root (or square-root dominating) kernels are available for a variety of popular kernels including Gaussian, \Matern, B-spline, inverse multiquadric (IMQ), sech, and Wendland’s compactly supported $\kernel$.

\cref{table:error_tails} summarizes how $\errorsplit(\P,m,\delta,\ossymb)$ varies with $\kersplit$ and $\P$.  For example, when $\P$, $\Q$, and $\kersplit=\ksqrt$ are compactly supported, $\errorsplit(\P,m,\delta,\ossymb) = O((\log\frac{m}{\delta})^2)$ and hence the compression error \cref{eq:mmd_diff_x_y} of \cref{thm:compression_guarantee} becomes
\begin{talign}
O(\frac{(\log\frac{m}{\delta})^2}{2^{\ossymb} \sqrt{m}}  \!+ \!  \frac{(\log\frac{n}{\delta})^2}{2^{\ossymb} \sqrt{n}})  
    =
o(\frac{1}{\sqrt{m}}  \!+ \!  \frac{1}{\sqrt{n}}),
\end{talign}
when $\ossymb = \log_2(\omega(\log^2(\frac{m \vee n}{\delta})))$.
More generally, the \tmmd compression error is asymptotically negligible relative to the usual error of $\mmd_\kernel(\xs,\ys)$ whenever $\ossymb = \log_2 (\omega(\errorsplit(\P,m\vee n,\delta,\ossymb)))$.
For example, if  $\P$, $\Q$, and 
$\kersplit=\ksqrt$ have 
have subexponential tails then, for some constant $c>0$, the choice $\ossymb \ge c\log_2\log (m\vee n) $ yields $o(\frac{1}{\sqrt{m}}  \!+ \!  \frac{1}{\sqrt{n}})$ compression error. By \cref{table:error_tails}, the same result holds when $\kersplit =\kernel$ is analytic.  Together, these results cover all of the aforementioned popular kernels.

We next turn our attention to the running time of \tmmd. By \citet[Ex.~4]{shetty2022distribution}, the runtime of each $\ktcompress(\xs^{(i)},\ossymb,\kernel,\kersplit,\delta)$ call is dominated by $O(4^\ossymb \frac{m+n}{\sblock}(\log_4 (\frac{m+n}{\sblock}) - \ossymb))$ kernel evaluations.
Since $\mmd_\kernel(\hatxs,\hatys)$ can be computed using $O(4^{\ossymb} \sblock (m + n))$ kernel evaluations once $\hatxs$ and $\hatys$ are available, the total runtime of \tmmd is 
\begin{talign}\label{tmmd-runtime}
O(4^{\ossymb} (m+n)(\sblock + \log_4 (\frac{m+n}{\sblock})- \ossymb)).
\end{talign}
Notably, this runtime is $O((m+n)\log_4^{c+1} (m+n))$, \emph{near-linear} in $m+n$, whenever $\sblock=O(\log_4(m+n))$ and $\ossymb\leq c\log_4 \log (m+n)$, as in the subexponential and compact-support settings previously considered.

\begin{table}[t]
    \centering
  \resizebox{0.48\textwidth}{!}
  {
    {
    \renewcommand{\arraystretch}{2}
    \begin{tabular}{ccc}
        \toprule
          \Centerstack{\bf Tails of $\P$}
        & \Centerstack{\bf Choice of $\kersplit$} 
        & %
        $\errorsplit(\P, m,\delta,\ossymb)$
        \\[0mm]
        \midrule 

        \Centerstack{ Compact}
        & \Centerstack{Compact $\ksqrt$ }
        & $(\log\frac{m}{\delta})^2$
        \\[0mm]
        
        \Centerstack{ Subexponential}
        & \Centerstack{Analytic $\kernel$}
        & \Centerstack{$(\log\frac{m}{\delta})^{\frac{3d+5}{2}}$}
         
        \\[0mm]
        
        \Centerstack{ Subexponential}
        & \Centerstack{Subexponential $\ksqrt$ } 
        & $c_{m,\delta}  (\log\frac{m}{\delta})^{\frac{d+5}{2}}$
        \\[0mm]
        
        \Centerstack{ $\rho$-Heavy-tailed} 
        &\Centerstack{$\rho$-Heavy-tailed $\ksqrt$} 
        & $(\frac{m}{\delta})^{\frac{d}{2\rho}}  (\log\frac{m}{\delta})^{\frac{5}{2}}$
        \\[1ex] \bottomrule \hline
    \end{tabular}
    }
    }
    \caption{\tbf{Error inflation due to compression.} We report the scaling of $\errorsplit$ in \cref{thm:compression_guarantee} up to constants depending on $d$ under various assumptions on $\kersplit$ and the tail decay of $\P$ (see \cref{sub:proof_of_tab_1} for the proof). 
    Here $c_{m,\delta} \defeq \sqrt{\log\log\frac{m}{\delta}}$.
    } 
 \label{table:error_tails}
\end{table}

\begin{table}[t]
    \centering
  \resizebox{0.48\textwidth}{!}
  {
    {
    \renewcommand{\arraystretch}{2}
    \begin{tabular}{ccc}
        \toprule
        \bf Test name
        & \bf MMD separation
        & \bf Runtime
        \\[0mm]
        \midrule 

        \Centerstack{\bf \ctt \\ \scriptsize{(ours, \cref{thm:uniform_separation})}} 
        &  $\frac{\errorsplit(\P,m,\delta,\ossymb)}{2^{\ossymb}\sqrt{m}} \!+\! m^{-\half}$ 
        & $4^{\ossymb} m \log m$ 
        \\ [0mm]
        
        \Centerstack{\bf Complete MMD \\ {\scriptsize{\citep{gretton2012akernel}}}}
        & $m^{-\half}$ 
        & $m^2$

        \\ [0mm]
        \Centerstack{\bf Block MMD \\ {\scriptsize{\citep{zaremba2013btest}}}}
        & $(Bm)^{-\quarter}$ 
        & $Bm$  

        \\ [0mm]
        \Centerstack{\bf Incomp. MMD \\ \scriptsize{\citep{yamada2018post}}} 
        & $\l^{-\quarter}$
        & $\l$
        \\[1ex] \bottomrule \hline
    \end{tabular}
    }
    }
    \caption{\tbf{Detectable $\mmd(\P,\Q)$ separation vs.\ runtime} for complete and approximate MMD tests.
    For subexponential $(\P,\Q)$, \ctt can detect 
    $m^{-\half}$ MMD separation in near-linear time, while the complete, block, and incomplete tests require quadratic time. See \cref{sub:ctt} for more details.
    } 
 \label{table:power_runtime}
\end{table}

\subsection{\cttname}
\label{sub:ctt}
We are now prepared to discuss our complete \ctt procedure defined in \cref{algo:ctt}.
\ctt begins by computing the \tmmd test statistic  described in \cref{sub:tmmd} but then reuses the coresets to carry out a special form of the permutation test.  Rather than permuting all $m+n$ points as in the standard permutation tests of \cref{sub:two_sample}, \ctt keeps each coreset intact and only permutes the order of the $\sblock$ coresets when setting the test threshold.  

The advantages of coreset reuse are threefold. First, the compression step can be carried out just once, irrespective of the number of permutations employed.
Second, the same kernel evaluations used to compute the initial test statistic \cref{tmmd} can be reused to compute every permuted \tmmd.
Indeed, when forming the initial test statistic, it suffices to store the $\sblock^2$ sufficient statistics 
\begin{talign}
\label{eq:coreset-sufficient-stats}
&a_{ij}
    =
\frac{1}{|\hatzs^{(i)}||\hatzs^{(j)}|}\sum_{z\in\hatzs^{(i)}, z'\in\hatzs^{(j)}}\kernel(z,z')
    \qtext{for} \\
&(\hatzs^{(1)},\dots,\hatzs^{(\sblock)})
    \defeq
(\hatxs^{(1)},\dots,\hatxs^{(s_m)}, 
\hatys^{(1)}, \dots,\hatys^{(s_n)})
\end{talign}
since each permuted \tmmd can be written as
\begin{talign}
\mmd^2_\kernel(\hatxs^b, \hatys^b) =  \sum_{i,j = 1}^\sblock\textfrac{1 - 2 |\indic{i\leq s_m} - \indic{j\leq s_m}|}{s^2}  a_{\sigma(i)\sigma(j)}
\end{talign}
for some permutation $\sigma$ over the coreset indices $\{1,\dots, \sblock\}$.
Hence, the total running time of \ctt is simply the running time \cref{tmmd-runtime} of a single \tmmd call  plus $O(\sblock^2\numperm)$ arithmetic operations.

Finally, by keeping each coreset intact, \ctt
ensures that every coreset permutation $(\hatxs^b, \hatys^b)$ accurately approximates an analogous full-sample permutation that keeps each of the 
$ ( \xs^{(i)} )_{i=1}^{\sblock_m}$ and  
  $ ( \ys^{(i)} )_{i=1}^{\sblock_n}$
bins intact and only permutes the order of the $\sblock$ bins.
One of the main contributions of this work is showing that such restricted permutation procedures provide high power even when $\sblock$ is set to a small value.  However, before turning to power, we next show that \ctt has a size exactly equal to the nominal level $\alpha$ for all sample sizes and all data distributions $\P$.

\begin{proposition}[Finite-sample exactness of \ctt] \label{thm:validity}
For any distribution $\P$, \cttname (\cref{algo:ctt}) has size (Type I error) exactly equal to the nominal level $\alpha$, i.e.,
    \begin{align}
        \mathrm{Pr}[\deltactt(\xs, \ys) = 1] = 
        \alpha \qtext{whenever} \P=\Q.
    \end{align}
\end{proposition}
\vspace{-4mm}
\begin{remark}[Exchangeability]
Our proof in \cref{proof-validity} does not require datapoint independence and rather holds under the weaker condition that the point sequence $(X_1,\dots, X_m, Y_1,\dots, Y_n)$ is \emph{exchangeable} under the null, i.e., the null distribution of this point sequence is invariant under permutation.
\end{remark}

Our proof of \cref{thm:validity}, based on exchangeability, parallels the size-no-larger-than-level proofs of \citet{schrab2021mmd, albert2022adaptive} but includes a more detailed treatment of the case $R = b_{\alpha}$ to ensure the exactness of the Type I error, as in \citet{hoeffding1952large}.

We now provide a complementary upper bound on the Type II error of \ctt (or equivalently, a lower bound on its power) under suitable assumptions on the MMD separation between $\P$ and $\Q$.

\begin{theorem}[Power of \ctt]
\label{thm:uniform_separation}
Suppose \cttname (\cref{algo:ctt}) is run with $m \leq n$, level $\alpha$, replication count $\numperm \geq \frac{1}{\alpha}- 1$, coreset count %
$\sblock_m \geq \frac{32}{9} \log(\frac{2e}{\gamma})$  for
    $\gamma \defeq \frac{\alpha}{4e} (\frac{\tilde\beta}{4})^{\frac{1}{\floor{\alpha (\numperm\!+\!1)}}}$ and $\tilde{\beta} \defeq \frac{\beta}{1+\beta/2}$.
Then \ctt 
has power 
\begin{talign} \label{eq:power_condition_main}
    \Pr[\deltactt(\mathbb{X}_m,\mathbb{Y}_n)=1] &\geq 1\!-\!\beta 
    \end{talign}
whenever $c' \mmd_{\kernel}(\P,\Q)/\sqrt{\log({1/}{\gamma})}$ is greater than
\begin{talign}
    &2c_{\nicefrac{\tilde{\beta}}{20\sblock}}\sqrt{\frac{\kinfnorm}{m}}  \!+\! \frac{\errorsplit(\P, \frac{m}{\sblock_m}, \frac{\tilde{\beta}}{20\sblock_m},\ossymb)+\errorsplit(\Q, \frac{m}{\sblock_m}, \frac{\tilde{\beta}}{20\sblock_n},\ossymb)}{2^{\ossymb} \sqrt{m}}
\end{talign}
for %
$c'$ a universal constant and $c_{\delta}$ as defined in \cref{thm:compression_guarantee}.
\end{theorem}

\begin{remark}[Valid parameter values]
   The \ctt compression level $\mathfrak{g}$ is an integer in $\{0,\dots,\log_4 (\frac{m+n}{s})\}$---a larger value provides more power but increases runtime. The failure probability $\delta$ and level $\alpha$ take arbitrary values in $(0,1)$, while the coreset count $s \leq m+n$ and replicate count $\mathcal{B}$ are positive integers.
\end{remark}

The proof of \cref{thm:uniform_separation} in \cref{proof-uniform_separation} contains several novel arguments that may be of independent interest.
First, using novel techniques based on order statistics, we show that $\numperm \geq \frac{1}{\alpha}\!-\! 1$ permutations suffice to obtain a powerful permutation test.  
Our arguments can be straightforwardly adapted to strengthen the analogous results for the complete \citep[Thm.~5]{schrab2021mmd} and incomplete \citep[Thm.~5.2]{schrab2022efficient} permutation tests.
Compared with the $\numperm \geq \frac{3}{\alpha^2} (\log(\frac{8}{\beta}) \!+\! \alpha(1\!-\!\alpha))$ requirement of \citet{schrab2021mmd, schrab2022efficient}, our requirement  eliminates all dependence on the target power $1-\beta$ and improves the $\alpha$ dependence by a quadratic factor.
Put in practical terms, by \cref{thm:uniform_separation}, $\numperm \geq 19$ permutations suffice for powerful permutation testing at level $\alpha=0.05$ while $\numperm \geq 2613$ were previously required to guarantee power greater than the level. 
Second, we show that to obtain a powerful permutation test, one need not permute all $m+n$ datapoints; rather, it suffices to permute $\sblock$ bins where the number $\sblock$ can be chosen \emph{independently} of the sample sizes.
In the end, \cref{thm:uniform_separation} implies that \ctt with a small number of coresets and permutations can detect distributional discrepancies of order $\frac{1}{\sqrt{m}}$---the same detection threshold enjoyed by the quadratic-time MMD tests \citep[Thm.~13]{gretton2012optimal}----up to the inflation factor $(1+\errorsplit/2^{\ossymb})$.
Since $\errorsplit/2^{\ossymb}=o(1)$ whenever $\ossymb =  \log_2 (\omega(\errorsplit))$ and the runtime of \ctt is dominated by a single \tmmd computation, by setting $\kersplit$, $\ossymb$, and $\sblock$ as discussed in \cref{sub:tmmd}, \ctt can recover the quadratic-time detection threshold in \emph{near-linear} $O((m+n)\log_4^{c+1}(m+n))$ time for subexponential $(\P,\Q)$ and subquadratic time for heavy-tailed $(\P,\Q)$ with $\rho > 2d$ moments.

Our next result shows that such runtime improvements \emph{cannot} be achieved by the  state-of-the-art block and incomplete MMD tests of \cref{sub:two_sample}, as each requires quadratic time (i.e., $B=\Omega(m)$ or $\l = \Omega(m^2)$) to match the order $\frac{1}{\sqrt{m}}$ detection threshold of a complete MMD test.

\begin{proposition}[Power upper bounds for complete, block, and incomplete MMD tests] \label{thm:rates_other_tests}
For any nominal level $\alpha \in (0,1)$ and target Type II error $\beta \in (0,1)$, there exists a constant $c_{\alpha,\beta}$ such that the following power upper bounds hold for all sample sizes $m$. %
\begin{enumerate}[label=(\alph*),leftmargin=*]
\itemsep0em
    \item \label{thm:rate_complete} %
    Asymptotic complete test: $\Pr[\Delta_{\trm{up}}\!(\mathbb{X}_m,\mathbb{Y}_m) \!=\! 1] \!<\! 1\!-\!\beta$ if $\mmd(\P,\!\Q) \!\leq\! \frac{c_{\alpha,\beta}}{\sqrt{m}}$.
    \item \label{thm:rate_block} 
    Asymptotic block test:  $\Pr[\Delta_{B}\!(\mathbb{X}_m,\mathbb{Y}_m) \!=\! 1] \!<\! 1\!-\!\beta$ if $\mmd(\P,\!\Q) \!\leq\! \frac{c_{\alpha,\beta}}{(B m)^{1/4}}$ and $B,\!\frac{m}{B} \!\to\! \infty$.
    \item \label{thm:rate_incomplete} %
    Asymptotic incomplete test: $\Pr[\Delta_{\trm{inc}}(\mathbb{X}_m,\!\mathbb{Y}_m)\!=\!1] \!<\! 1\!-\!\beta$ if $\mmd(\P,\Q) \!\leq\! \frac{c_{\alpha,\beta}}{\l^{1/4}}$, and $\frac{\l}{m}\!\to \!c\!>\!0$.
\end{enumerate}
\end{proposition}
The proof of \cref{thm:rates_other_tests} in \cref{proof_of_rates_other_tests} uses the asymptotic distribution of each statistic under the null and alternative hypotheses \citep[as derived by][]{gretton2007kernel,gretton2009afast,zaremba2013btest,yamada2018post} to upper bound the power (and hence lower bound the Type II error) of each test.
Moreover, the proof reveals that these detectable $\mmd(\P,\Q)$ separation rates are tight. For example, there also exists a constant $c'_{\alpha,\beta}>c_{\alpha, \beta}$ such that $\mathrm{Pr}[\Delta_{up}(\mathbb{X}_m,\mathbb{Y}_m) = 1] > 1\!-\!\beta$ whenever $\mmd(\P,\Q) \geq c'_{\alpha,\beta}/\sqrt{m}$.
\cref{table:power_runtime} summarizes the trade-off between detectable MMD separation and runtime for the complete and approximate MMD tests
and highlights the improved trade-off offered by \ctt. %

In particular, the time-power trade-off of \ctt improves significantly under the favorable settings of \cref{table:error_tails} (e.g., for compact $\P$ or subexponential $\P$ and analytic $\kernel$ in lower dimensions).
While the improvements need not be as large for heavier-tailed distributions, less smooth kernels, and higher dimensions, even the worst-case trade-offs of CTT 
are no worse than prior methods' as 
$\errorsplit(\P,m,\delta,\ossymb) = O\left(2^{\ossymb/2}m^{1/4}\sqrt{\log(\frac{1}{\delta})}\right)$ by \citet[Rem.~2]{dwivedi2021kernel}.
That is, for arbitrary distributions, dimensions,  and kernels, a user can comfortably use \ctt as a drop-in replacement for the block and incomplete tests, as we should expect no worse power-time trade-off curves.  
That said, there is some overhead associated with compression, so a user may find the block and incomplete tests to be faster for small sample sizes.

\section{CTT Extensions} \label{sec:ctt_extensions}
In this section, we develop two extensions of \ctt: first, 
a fast and powerful way to exploit an accurate low-rank kernel approximation 
and, second, 
a fast and powerful 
aggregation procedure for identifying a particularly discriminating kernel from amongst a collection of candidates.
\subsection{\lcttname}\label{sec:lctt}

Our first extension, called \lcttname (\cref{algo:lctt}), allows the user to exploit an accurate low-rank kernel approximation 
without sacrificing the provable time-power trade-off improvements of \ctt.
Specifically, we consider $\Theta(nr)$-time low-rank $\mmd_{\feat}$ approximations of the form \cref{low-rank-mmd} with $\feat$
selected so that the approximation error 
\begin{talign}
\featerr^2(\xs,\ys) = 
\sup_{x,y\in\xs\cup\ys} |\kernel(x,y) - \feat(x)^\top\feat(y)|
\end{talign}
is small. 
For example, 
\citet[Thm.~1]{sriperumbudur2015optimal}
show that 
$\featerr(\xs,\ys) = O(r^{-1/4})$ 
and hence that
$|\mmd_{\kernel}^2(\xs,\ys)-\mmd_{\feat}^2(\xs,\ys)|\sless{\cref{eq:mmd_approx_error}} 4\featerr^2(\xs, \ys) = O(r^{-1/2})$
with high probability when 
$\feat$ consists of $r$ random Fourier features and $(\P,\Q)$ are compactly supported. However, since computing $\featmmd(\xs,\ys)$ requires $\Theta((m+n)r)$ feature evaluations, this analysis requires $\Omega(m^3)$ time to match the order $\frac{1}{\sqrt{m}}$ detection threshold of a complete MMD test.
Our following result, proved in \cref{proof-lctt_guarantees}, shows that appropriate compression prior to low-rank approximation yields comparable power guarantees in just $O(4^\ossymb (m+n)\log r)$ time. 

\begin{theorem}[\lctt exactness and power]
\label{lctt_guarantees}
\lcttname (\cref{algo:lctt}) has size exactly equal to the level $\alpha$ for all $\P$.
If the replication count $\numperm \geq \frac{1}{\alpha}- 1$, the permutation bin count 
$\sblock \geq \frac{m+n}{m}\frac{32}{9} \log(\frac{2e}{\gamma})$  for
    $(\gamma,\tilde{\beta})$ as in \cref{thm:uniform_separation}, and $m \leq n$, 
then \lctt 
has power
\begin{talign} \label{eq:power_condition_main}
    \Pr[\deltalctt(\mathbb{X}_m,\mathbb{Y}_n)=1] &\geq 1\!-\!\beta 
    \end{talign}
when, for a universal constant $c'$ and $c_{\delta}$ defined in \cref{thm:compression_guarantee}, %
\begin{talign}
&c'\mmd_{\kernel}(\P,\Q)/\sqrt{\log({1/}{\gamma})}
\geq         
    2c_{\nicefrac{\tilde{\beta}}{20\sblock_r}}\sqrt{\frac{\kinfnorm}{m}} + \\
    & \frac{\errorsplit(\P, \frac{m}{\sblock_{m,r}}, \frac{\tilde{\beta}}{20\sblock_{m,r}},\ossymb)+\errorsplit(\Q, \frac{m}{\sblock_{m,r}}, \frac{\tilde{\beta}}{20\sblock_{n,r}},\ossymb)}{2^{\ossymb} \sqrt{m}}
    +
    \featerr(\hatxs,\!\hatys).
\end{talign}
\end{theorem}

\lcttalg
\vspace{-2mm}

Specifically, to form a {low-rank} \tmmd test statistic \cref{ltmmd}, \lcttname (\lctt, \cref{algo:lctt}) divides each sample into $\sblock_{m,r}$ or $\sblock_{n,r}$ equal-sized bins, forms a coreset for each bin using \ktcompress with kernels $(\kernel,\kersplit)$, and computes the low-rank approximation $\featmmd$ using only the concatenated coreset points $\hatxs$ and $\hatys$.
Then, just as in \cref{algo:ctt}, \lctt selects an appropriate test statistic threshold albeit now manually partitioning  $(\hatxs,\hatys)$ into $\sblock$ coreset bins and permuting those bins.
All told, the \lctt runtime is dominated by 
$O(4^\ossymb (m+n)(\log_4(\frac{2r}{a})-\ossymb))$ 
kernel evaluations, 
$O(4^\ossymb (m+n) a)$ feature evaluations, and $O(\sblock^2 \numperm)$ arithmetic operations.
Importantly, when $a=O(\log r)$, the \emph{logarithmic} dependence on the rank $r$ means that, by \cref{lctt_guarantees}, \lctt can recover the order $\frac{1}{\sqrt{m}}$ detection threshold of a complete MMD test in \emph{near-linear} time for subexponential $(\P,\Q)$ and subquadratic time for heavy-tailed $(\P,\Q)$ with $\rho > 2d$ moments, even when the approximation error $\featerr$ has slow (e.g., order $r^{-1/4}$) decay.

\subsection{\acttname}
\label{sub:fast_kernel}
Each of the tests considered so far assumes that a suitable kernel $\kernel$ has been pre-selected by the user.
However, because the discriminating power of a kernel varies with the pair of distributions under consideration, it can be challenging to identify a single suitable kernel a priori.
As a result, a variety of strategies have been introduced for automatically selecting discriminating kernels for MMD tests~\citep[see, e.g.,][]{gretton2012optimal,sutherland2017generative,liu2020learning,kubler2020learning}. 
We highlight in particular the \emph{aggregated MMD tests} of \citet{schrab2021mmd} which combine complete MMD tests with varying kernels into a single test with power comparable to the best individual test. 
Since these complete aggregated tests run in quadratic time, \citet{schrab2022efficient} recently introduced incomplete aggregated tests that trade off computation time and power exactly as in the single-kernel setting (see \cref{table:power_runtime}).

In \cref{algo:actt} of \cref{sec:proof_agg_validity_power}, we extend our 
\cttname 
framework 
to form a more efficient aggregated test that we call \acttname (\actt).
Like past aggregated tests, \cref{algo:actt} takes as input \emph{any} indexed collection of kernels $(\kernel_{\lambda})_{\lambda \in \Lambda}$ and  accommodates nonnegative importance weights $(w_{\lambda})_{\lambda \in \Lambda}$ with $\sum_{\lam\in\Lambda}w_\lam \leq 1$ reflecting prior beliefs about the suitability of each kernel.
Like \cref{algo:ctt}, \actt then proceeds to partition $\xs$ and $\ys$ into bins and to form a coreset for each bin using a parallel collection of auxiliary \ktcompress kernels $(\kersplit[\lambda])_{\lambda \in \Lambda}$ scaled so that %
$\sup_{z}\!\! | \kersplit[\lam](z, z)| = 1$.
However, instead of forming a separate coreset for each candidate kernel, as one might if one were running a \ctt test separately for each $\klam$, \actt saves additional computation by forming a single coreset per bin using the combination kernels $\kernel = \sum_{\lam\in\Lam}\klam$ and $\kersplit = \sum_{\lam\in\Lam}\kersplit[\lam]$.
These shared coresets are used to compute a \tmmd test statistic $M_\lam$ for each $\klam$, $\numperm_1$ permuted \tmmd statistics to estimate the null distribution for each $\klam$, and $\numperm_2$ permuted \tmmd statistics to estimate the size of the aggregated test.
Finally, exactly as in \citet[Alg.~1]{schrab2021mmd}, \actt selects a suitable rejection threshold for each test statistic $M_\lam$ and rejects the null whenever at least one $M_\lam$ exceeds its threshold.
The total cost of \actt is at most $|\Lam|$ times that of single-kernel \ctt (with $\numperm=\numperm_1$)
plus the cost of $O(|\Lam|(\numperm_1\log \numperm_1 + \numperm_2\numperm_3))$ arithmetic operations due to sorting and selecting thresholds.
\cref{thm:agg_validity_power}, proved in \cref{sec:proof_agg_validity_power}, shows that \actt is \emph{valid}, i.e., it has Type I error $\leq\alpha$ for all sample sizes and generating distributions, and that its power is comparable to that of the best $\klam$-\ctt test run with compression level 
$\ossymb - \log_2{|\Lam|}$.
Moreover, by \cref{thm:uniform_separation}, each $\klam$-\ctt test has power comparable to a complete $\klam$ test when 
$\ossymb = \log_2(|\Lam|\omega(\errorsplit[\lam]))$.
Therefore, by setting $\kersplit$, $\ossymb$, and $\sblock$ as discussed in \cref{sub:tmmd}, \actt with 
${|\Lam|  }=O(1)$ 
can recover the detection threshold of the best quadratic-time $\klam$ test in \emph{near-linear} $O((m+n)\log_4^{c+1}(m+n))$ time for subexponential $(\P,\Q)$ and subquadratic time for heavy-tailed $(\P,\Q)$ with $\rho > 2d$ moments.

\begin{theorem}[$\actt$ validity and power] \label{thm:agg_validity_power}
For any distribution $\P$, $\actt$ (\cref{algo:actt}) has non-asymptotic level $\alpha$, i.e.,
\begin{talign} \label{eq:agg_validity}
    \Pr[\deltaactt(\xs, \ys) = 1]
    \leq \alpha
    \stext{whenever} \P=\Q.
\end{talign}
Morever, with $m\leq n$,  $\alpha \!\in\! (0,\frac1e)$, and  replicate counts $\numperm_1 \!\geq\! (\max_{\lambda \in \Lambda} w_{\lambda}^{-2}) \frac{12}{\alpha^2}(\log(\frac{8}{\beta}) \!+\! \alpha(1\!-\!\alpha))$, $\numperm_2 \!\geq\! \frac{8}{\alpha^2} \log(\frac{2}{\beta})$, and $\numperm_3 \!\geq\! \log_2(\frac{4}{\alpha} \min_{\lambda \in \Lambda} w^{-1}_{\lambda})$,  \actt has power
\begin{talign}
\label{eq:agg_power}
    \Pr[\deltaactt(\xs, \ys) = 1]
    \geq 1-\beta
\end{talign}
whenever there exists a $\lam \in \Lam$ satisfying
\begin{talign}
\label{eq:mmd_agg_power}
\mmd_{\kernel_{\lam}}(\P,\Q) \!\!\geq \! 
    c'\sqrt{\log(\frac{1}{\gamma_{\lam}})}\,
    \epsagg(\frac{\beta/(10\sblock)}{4+\beta})
\end{talign}
and $\sblock_{m}\!\geq\! \frac{32}{9}\log(\frac{2e}{\gamma_{\lam}})$,  where $\gamma_{\lam} \!\defeq\! \frac{\alpha w_{\lam}}{8e} (\frac{\beta}{8+2\beta})^{\frac{1}{\floor{\alpha w_{\lam}(\numperm_1\!+\!1)/2}}}$, $c'$ is a universal constant, and
\begin{align}
\label{eq:eps_agg}
\epsagg(\delta) &\defeq  2c_{\delta} \sqrt{\textfrac{\sinfnorm{\klam}}{m}} \\
    &+ c_{\Lam}\! \max_{\lam\in \Lam} 
    {\textfrac{\error[\klam,\kersplit](\P, \frac{m}{\sblock_m}, \delta,\ossymb)+\error[\klam,\kersplit](\Q, \frac{m}{\sblock_m}, \delta,\ossymb)}{2^{\ossymb} \sqrt{m}}}
\end{align}
for $c_{\delta}$  as in \cref{thm:compression_guarantee} and $c_{\Lam}\!\defeq\!2\sqrt{|\Lam|(1\!+\!\log(|\Lam|))}$ .
\end{theorem}

\gaussfig
\renewcommand{\vsep}{\vspace{1mm}}

\section{Experiments} \label{sec:experiments}

We now present seven experiments that illustrate the improved power-runtime trade-offs of \ctt, \lctt,  and \actt over state-of-the-art approximate MMD tests.
Unless otherwise noted, we use  $\alpha\!=\! 0.05$,  $m\!=\!n\!=\!4^9$, $s=32$,  $\delta=\half$, and Gaussian $\kersplit=\kernel$.
We report average rejection rates over $400$ independent replications of each experiment with 95\% \citet{wilson1927probable} confidence intervals.
See \cref{sec:app_experiments} for additional details and \href{https://github.com/microsoft/goodpoints}{github.com/microsoft/goodpoints} for open-source Python code recreating all experiments.

\para{\ctt experiments}
We evaluate \ctt in two settings with the kernel bandwidth set using the popular median heuristic \citep{chaudhuri2017themean}.
    In the \textsc{Gaussian} setting, $\P$ and $\Q$ are $10$-dimensional Gaussians with identity covariances; the means have Euclidean distance $0.012$ under the alternative and $0$ under the null.
    The \textsc{EMNIST} setting is similar to the one considered by \cite{kubler2020learning, schrab2021mmd}, where $\P$ and $\Q$ denote distributions on downsampled $7 \times 7$ images of the EMNIST dataset \citep{cohen2017EMNIST}---an extension of the MNIST dataset \citep{lecun2010mnist} that also includes letters. Under the alternative hypothesis, $\P$ denotes a $2$-mixture of uniform distributions based on parity of digits and letters with weight 0.49 (resp. 0.51) for even (resp. odd) parity, while $\Q$ puts equal weight 0.5 on both parities. Under the null hypothesis, we consider $\P = \Q = $ equally weighted mixture. We plot the test power results versus runtime in \cref{fig:gaussians_EMNIST} with \textsc{Gaussian} setting on top and \textsc{EMNIST} setting on the bottom.

\blobs

\cref{fig:gaussians_EMNIST} ({left}) shows that in both settings, the \ctt time-power trade-off curve uniformly dominates those of the state-of-the-art subsampling approximations of \cref{sub:two_sample}: the wild bootstrap block (W-Block) and incomplete (W-Incomp.) tests and the asymptotic block (A-Block I and II) incomplete  (A-Incomp.) tests. In particular, the \ctt test with $\ossymb \!=\! 3$ achieves the same power as the wild bootstrap quadratic-time tests (W-block with $B \!=\! 4^9$ and W-Incomp. with $\ell\!=\! \frac{(4^9\!-\!1)n}{2}$) while providing a $\speedup$ speed-up. While \ctt and the wild bootstrap tests are guaranteed to have Type I error controlled by $\alpha$ (\cref{fig:gaussians_EMNIST_size_exact} in \cref{sec:app_experiments}), the asymptotic tests violate their level constraint for large $B$ or $\ell$ as the asymptotic approximation is poor for such settings. 
As a result, \cref{fig:gaussians_EMNIST} displays only those points that respect the level constraint in the power plots. For consistency, we used $\mathcal{B} = 39$ replicates for all the non-asymptotic tests. 

\para{\lctt experiments} 
In the same settings, \cref{fig:gaussians_EMNIST} ({right}) compares \ctt, the state-of-the-art low-rank RFF test of \cref{sub:two_sample}, and  
\lctt with RFF $\feat$ 
and $a = {r/}{(4^g2^{\mathbb{I}[r> 4^{\ossymb+1}]})}$. 
We use $\mathcal{B}=39$ permutations to set the threshold for each test. 
We find that \ctt and RFF produce comparable trade-off curves despite their distinct and complementary approximation strategies
and that the combined \lctt test with $\ossymb \geq 2$ consistently yields the best performance, with $5$--$\speeduptwenty$ speed-ups over \ctt or RFF alone.
\para{\actt experiments}
We compare our \actt procedure in {two} different settings with the aggregated wild bootstrap incomplete test (W-Incomp.)  of \citet{schrab2022efficient} and $m=n=4^7$. 
For the \textsc{Blobs} experiment of \citet[Fig.~1]{gretton2012optimal}, and \citet[Fig.~2]{sutherland2017generative}, $\P$ and $\Q$ are two-dimensional $3 \times 3$ grids of Gaussian mixture components with a grid spacing of $10$. Each mixture component has identity covariance in $\P$, while for $\Q$ the ratio of eigenvalues for their covariance matrix is $\epsilon$ with diagonal entries set to 1; the null hypothesis corresponds to $\epsilon = 1$.  %
We consider the bandwidth set $\Lambda = \{ 2^i \lambda_0 \}_{i=0}^{-4}$ for $\klam(x,y) = e^{-\twonorm{x - y}^2/(2\lam^2)}$, where $\lambda_0$ denotes the median heuristic  bandwidth, and uniform weights $w_\lambda = 1/|\Lambda|$. We plot the results in \cref{fig:blobs} and observe that \actt provides a uniform gain in the power-runtime curve over the aggregated WB incomplete test---a $100\times$ to \speedup-speed up. %

We perform the same comparison with the same configurations on the \textsc{Higgs} experiment, a variation of the setting considered by \citet{liu2020learning}, which took the data from \citet{baldi2014searching}. While the original dataset has samples with 27 covariates belonging to two different classes (0 and 1), \citet{liu2020learning} considers only four covariates of those, and we only use the first two of the four. We consider two settings for the alternative distribution: one in which $\P$ is sampled from the class 0 and $\Q$ is sampled from the class 1 (\textsc{Higgs}: \cref{fig:blobs}, \emph{middle}) and a more challenging one in which $\P$ is sampled from the class 0 and $\Q$ is sampled from each class with equal probability (\textsc{Higgs Mixture}: \cref{fig:blobs}, \emph{bottom}). We observe a $100\times$ to \speedup-speed up over the aggregated WB incomplete test.
\section{Connections and Conclusions}
This paper introduced \ctt, a new framework for kernel testing with compression; \lctt, a test that combines low-rank approximation and compression for added scalability; and \actt, a fast and powerful procedure for aggregating kernel tests. 
While we have shown that \ctt, \lctt, and \actt offer better power-time trade-offs than state-of-the-art approximate MMD tests, we highlight that there are other approaches to fast non-parametric testing based on alternative test statistics
\citep[see, e.g.,][]{chwialkowski2015fast,jitkrittum2016interpretable,kirchler2020two,shekhar2022a}. 
A natural follow-up question is whether compression techniques can also improve the power-time trade-offs of those tests. A second opportunity for future work is to extend the \ctt framework to other inferential tasks like independence and goodness-of-fit testing or kernel regression.

\subsubsection*{Acknowledgements}
We thank Marco Letizia for highlighting an error in an earlier version of \cref{fig:blobs}.
{%
\bibliography{biblio}}

\clearpage\newpage
\appendix
\onecolumn
\pagestyle{fancy}
\lhead{}
\chead{}
\rhead{}
\lfoot{}
\cfoot{\vspace{\baselineskip}\hyperlink{toc}{\thepage}}
\rfoot{}
\etoctocstyle{1}{Appendix}
\etocdepthtag.toc{mtappendix}
\etocsettagdepth{mtchapter}{none}
\etocsettagdepth{mtappendix}{section}
\etocsettagdepth{mtappendix}{subsection}
\addtocontents{toc}{\protect\hypertarget{toc}{}}
{\small\tableofcontents}
\newpage
\section{Background on \ktcompress}
\label{sec:ktcompress}
This section details the \ktcompress algorithm of \citet[Ex.~4]{shetty2022distribution}.
In a nutshell, \ktcompress (\cref{algo:ktcompress}) takes as input a point sequence of size $n$, a compression level $\ossymb$, two kernel functions $(\kernel, \kersplit)$, and a failure probability $\delta$. It then combines the  \compress algorithm of \citet[Alg.~1]{shetty2022distribution} with the generalized kernel thinning (KT) algorithm of \citet[Alg.~1]{dwivedi2021kernel,dwivedi2022generalized} to output a coreset of $2^{\ossymb} \sqrt{n}$ input points that together closely approximate the input in terms of $\mmd_{\kernel}$. \ktcompress proceeds by calling the recursive procedure \compress, which uses KT with kernels $(\kernel, \kersplit)$ as an intermediate halving algorithm. The KT algorithm itself consists of two subroutines: (1) \ktsplit (\cref{algo:ktsplit}), which splits a given input point sequence into two equal halves with small approximation error in the $\kersplit$ reproducing kernel Hilbert space and (2) \ktswap (\cref{algo:ktswap}), which selects the best approximation amongst the \ktsplit coresets and a baseline coreset (that simply selects every other point in the sequence) and then iteratively refines the selected coreset by swapping out each element in turn for the non-coreset point that most improves $\mmd_{\kernel}$ error.
As in \citet[Rem.~3]{shetty2022distribution}, we symmetrize the output of KT by returning either the KT coreset or its complement with equal probability.

For this work, we develop a slight modification of the original \ktcompress algorithm, which we use in all experiments and in the released implementation.
When the compression level $\ossymb=0$ and the number of input points passed to \compress is $n=4$, instead of running the usual \compress algorithm, we run \opthalvefour (\cref{algo:opthalvefour})
which identifies the coreset of size two that optimally approximates the input point sequence in terms of $\mmd_{\kernel}$ and then returns either that coreset or its complement with equal probability.

  \begin{algorithm2e}[h!]
    \SetKwFunction{proc}{\textnormal{\textsc{\compress}}}
    \SetKwFunction{proctwo}{\textnormal{KT}}
\caption{{\ktcompress\ --\ } Identify coreset of size $2^\ossymb\sqrt{n}$}
\label{algo:ktcompress}
\SetAlgoLined
  \DontPrintSemicolon
\small{
  \KwIn{point sequence $ \inputcoreset$ of size $n$, \osname~$\ossymb$, kernels $(\kernel,\kersplit)$, failure probability $\delta$} 
  \BlankLine
    \KwRet{ \textup{  $\compress(\inputcoreset, \ossymb, \kernel,\kersplit, \frac{\delta}{n 4^{\ossymb+1} (\log_4 n - \ossymb)})$}}
        \\
    \hrulefill\\
    \SetKwProg{myproc}{function}{}{}
     \myproc{\proc{$\cset,\ossymb,\kernel,\kersplit,\delta$}:}{
     \lIf{  $|\cset| = 4^{\ossymb}$ }{
        \KwRet{ $\cset $ }{}
    }  
        Partition $\cset$ into four arbitrary subsequences $ \{ \cset_i \}_{i=1}^4 $ each of size $n/4$ \\ 
        \For{$i=1, 2, 3, 4$}
        {$ \wtil{ \cset_i } \leftarrow \compress( \cset_i, \ossymb, \kernel, \kersplit, \delta)$ \qquad  // run \compress recursively to return coresets of size $2^{\ossymb} \cdot \sqrt{\frac{|\cset|}{4}}$
        } 
        $ \wtil{\cset} \gets\textsc{Concatenate}( \wtil{\cset}_1, \wtil{\cset}_2,\wtil{\cset}_3, \wtil{\cset}_4)$ \quad \quad  // combine the coresets to obtain a coreset of size   $2 \cdot 2^{\ossymb} \cdot \sqrt{|\cset|}$ \\
        }\KwRet{ \textup{  $\kt(\wtil{\cset}, \kernel,\kersplit, |\wtil{\cset}|^2\delta )$ 
        \qquad\qquad\qquad\quad \ 
        // halve the coreset to size $2^{\ossymb} \sqrt{|\cset|}$}} via symmetrized kernel thinning\;
    \hrulefill\\
\myproc{\proctwo{$\cset,\kernel,\kersplit,\delta$}:}{
     // Identify kernel thinning coreset containing  $\floor{|\cset|/2}$ input points\\
     $\ktcoreset \gets \ktswap(\kernel,\ktsplit(\kersplit,\cset, \delta))$
     }
     \KwRet{\textup{$\ktcoreset$ with probability $\half$ and the complementary coreset $\cset\setminus\ktcoreset$ otherwise}}\;
  }
\end{algorithm2e}

\newcommand{\ksumonetwo}{\textup{\texttt{K12\_plus\_K43}}}
\newcommand{\ksumfourone}{\textup{\texttt{K41\_plus\_K23}}}
\newcommand{\ksumfourtwo}{\textup{\texttt{K42\_plus\_K13}}}

\begin{algorithm2e}[h!]
\caption{{\opthalvefour\ --\ } Optimal four-point halving} 
  \label{algo:opthalvefour}
 \SetAlgoLined\DontPrintSemicolon
\small
{
    \KwIn{kernel $\kernel$, point sequence $\inputcoreset = (x_i)_{i = 1}^4$}
        \BlankLine
    $\ksumonetwo \gets \kernel(x_1,x_2) + \kernel(x_4,x_3)$;\ \, 
    $\ksumfourone \gets \kernel(x_4,x_1) + \kernel(x_2,x_3)$;\ \, 
    $\ksumfourtwo \gets \kernel(x_4,x_2) + \kernel(x_1,x_3)$\\
        \BlankLine
    \If{ $\ksumonetwo < \ksumfourone$ }{
    \If{ $\ksumonetwo < \ksumfourtwo$ }{
    \KwRet{\textup{$(x_3,x_4)$ with probability $\half$ and $(x_1,x_2)$ otherwise}}
    }
    \KwRet{\textup{$(x_2,x_4)$ with probability $\half$ and $(x_1,x_3)$ otherwise}}
    }        
    \If{ $\ksumfourone < \ksumfourtwo$ }{
    \KwRet{\textup{$(x_1,x_4)$ with probability $\half$ and $(x_2,x_3)$ otherwise}}    
    }
    \KwRet{\textup{$(x_2,x_4)$ with probability $\half$ and $(x_1,x_3)$ otherwise}} 
}
\end{algorithm2e}

\setcounter{algocf}{4}
\renewcommand{\thealgocf}{\arabic{algocf}a}
\begin{algorithm2e}[h!]
  \SetKwFunction{proctwo}{\texttt{get\_swap\_params}}
\caption{{\large\textsc{kt-split}\ --\ } Divide points into candidate coresets of size $\floor{n/2}$} 
  \label{algo:ktsplit}
  \SetAlgoLined\DontPrintSemicolon
  \small
  {
  \KwIn{kernel $\kersplit$, point sequence $\inputcoreset = (x_i)_{i = 1}^n$, failure probability $\delta$}
\BlankLine
  {$\coreset[1], \coreset[2] \gets \braces{}$}\quad // Initialize empty coresets: $\coreset[1],\coreset[2]$ have size $i$ after round $i$ \\ 
  {$\sgparam[] \gets 0$}\qquad\quad\ \ \ \quad // Initialize swapping parameter \\
  \For{$i=1, 2, \ldots, \floor{n/2}$}
    {%
    // Consider two points at a time \\
    $(\x, \x') \gets (\x_{2i-1}, \x_{2i})$ \\
	 \BlankLine
     // Compute swapping threshold $\cnew$ \\ 
     $ \cnew, \sigma \gets $\proctwo{$\sigma, \vmax[], \frac{\delta}{n}$}
     with\ $\vmax[]^2 
      \!=\! \kersplit(\x,\x)\!+\!\kersplit(\x',\x')\!-\!2\kersplit(x,x')$
    \BlankLine
    // Assign one point to each coreset after probabilistic swapping \\
    $\theta\gets  \sum_{j=1}^{2i-2}(\kersplit
	 (\x_j, \x)-\kersplit(\x_j,\x')) 
	 - 2\sum_{\z\in\coreset[1]}(\kersplit(\z, \x)-\kersplit(\z,\x'))$ \\[2pt]
    $(x, x') \gets (x', x)$ \textit{ with probability }
		    $\min(1, \half (1-\frac{\theta}{\cnew})_+)$\\[2pt]
     $\coreset[1]\texttt{.append}(\x); 
		        \quad \coreset[2]\texttt{.append}(\x')$
    }
    \KwRet{$(\coreset[1],\coreset[2])$\textup{, candidate coresets of size $\floor{n/2}$}}\\
    \hrulefill\\
    \SetKwProg{myproc}{function}{}{}
     \myproc{\proctwo{$\sigma, \vmax[], \delta$}:}{
     $
			\cnew 
			    \gets \max(\vmax[] \sigma\sqrt{\smash[b]{2\log(2/\delta)}}, \vmax[]^2)$ \\
     $\sigma^2 \gets \sigma^2
	        \!+\! \vmax[]^2(1 \!+\! ({\vmax[]^2}{}\! - \!2\cnew){\sigma^2}{/\cnew^2})_+$\\
     }
     \KwRet{$(\cnew, \sigma)$}\;
  }
\end{algorithm2e} 

\setcounter{algocf}{4}
\renewcommand{\thealgocf}{\arabic{algocf}b}
\begin{algorithm2e}[h!]
\caption{{\large\textsc{kt-swap}\ --\ } Identify and refine the best candidate coreset} 
  \label{algo:ktswap}
 \SetAlgoLined\DontPrintSemicolon
\small
{
    \KwIn{kernel $\kernel$, point sequence $\inputcoreset = (x_i)_{i = 1}^n$, candidate coresets $(\coreset[1],\coreset[2])$}
        \BlankLine
    $\coreset[0] 
	    \!\gets\! \texttt{baseline\_coreset}(\inputcoreset, \texttt{size}\!=\!\floor{n/2})
	 $ \qquad \quad \ // Compare to baseline (e.g., standard thinning)
	 \BlankLine
	 $\ktcoreset \!\gets\! \coreset[ \ell^\star]
	 \text{ for }
	 \ell^\star 
	    \!\gets\! \argmin_{\ell \in \braces{0, 1, 2}} \mmd_{\kernel}(\inputcoreset, \coreset[\ell])$ 
	    \ \ // \textup{Select best coreset} \\
	    \BlankLine
	    // Swap out each point in $\ktcoreset$ for best alternative in $\inputcoreset$ while ensuring no point is repeated in $\ktcoreset$
     \\[1pt]
	   \For{$i=1, \ldots, \floor{n/2}$}{
	   \BlankLine
	    $\ktcoreset[i] \gets \argmin_{z\in \sbraces{\ktcoreset[i]} \cup (\inputcoreset\backslash \ktcoreset) }\mmd_{\kernel}(\inputcoreset, \ktcoreset \text{ with } \ktcoreset[i] = z)$
	   }
    \KwRet{$\ktcoreset$\textup{, refined coreset of size $\floor{n/2}$}}
}
\end{algorithm2e}
\renewcommand{\thealgocf}{\arabic{algocf}}

\section{\pcref{thm:compression_guarantee}}\label{proof-compression_guarantee}
We first provide a discussion on the error inflation factor and then prove the two claims in \cref{thm:compression_guarantee}. 

\subsection{On the \ktcompress error inflation factor}
\label{sub:inflation_factor}
Given a point sequence $\inputcoreset$, a positive integer $n$, and a scalar $\delta\in(0, 1)$, the inflation factor $\errorsplit(\inputcoreset, n, \delta,\ossymb)$ denotes the smallest scalar of the form
\begin{talign} \label{eq:errorsplit_def}
 \errorsplit^2(\inputcoreset, n, \delta,\ossymb) = %
    &256 (\log_4 n \!-\!\ossymb\!-\!1)
    (\ckk(\inputcoreset) + \mkk(\inputcoreset, \delta, 2^{\ossymb+1}\sqrt{n} ) \sqrt{\log( \frac{3 n (\log_4 n \!-\!\ossymb\!-\!1)}{\delta})} 
    )^2 
    \\ &\cdot (\sqrt{\log(n\!+\!1)} + \sqrt{\log(2/\delta)})^2,
\end{talign}
where $\ckk(\inputcoreset)$  and $\mkk(\inputcoreset, \delta, 2^{\ossymb+1}\sqrt{n})$ are any scalars satisfying the property that, on an event of probability at least $1-\frac{\delta}{2}$, every \ktcompress (\cref{algo:ktcompress}) call to $\kt$ with an input of size $\l$ (that is a subset of $\inputcoreset$) is 
$\kernel$-sub-Gaussian (see \citet[Def.~2]{shetty2022distribution}) with parameters
\begin{talign} \label{eq:subgaussian_params}
\shiftparam_{\l,n} = \frac{2C_{\kernel,\kersplit}(\inputcoreset)}{\l}
\qtext{and}
\kgaussparam_{\l,n} =  \frac{2\mathfrak{M}_{\kernel, \kersplit} (\inputcoreset,\delta, \l)}{\l}\sqrt{\log( 
    \frac{12n4^{\ossymb} (\log_4 n- \ossymb)}{\ell \delta})}.
\end{talign}
See \cref{rem:frak_M} for the valid values of $\ckk$, and $\mkk$ derived in prior work for standard choices of $\kersplit$. (Notably, $\mkk$ is non-decreasing in its last argument.)

The factor $\errorsplit(\Q, n, \delta, \ossymb)$ is the population analogue of $\errorsplit(\ys, n, \delta, \ossymb)$ and is defined as the smallest scalar of the form
\begin{talign}\label{eq:pop_errorsplit_def}
    \errorsplit^2(\Q, n, \delta, \ossymb) \defeq %
    &256 (\log_4 n \!-\!\ossymb\!-\!1)
    (\ckk(\Q, n, \delta) + \mkk(\Q, n, \delta, 2^{\ossymb+1}\sqrt{n}) \sqrt{\log( \frac{3 n (\log_4 n \!-\!\ossymb\!-\!1)}{\delta})} 
    )^2 
    \\ &\cdot (\sqrt{\log(n\!+\!1)} + \sqrt{\log(2/\delta)})^2,
\end{talign}
where $\ckk(\Q, n, \delta)$ and $\mkk(\Q,n,\delta, 2^{\ossymb+1}\sqrt{n})$ satisfy
\begin{talign}
\label{eq:M_P_def}
    \P\brackets{\ckk(\ys) \leq \ckk(\Q, n, \delta) \stext{and}\mkk(\ys,\delta, 2^{\ossymb+1}\sqrt{n})\leq \mkk(\Q,n,\delta, 2^{\ossymb+1}\sqrt{n})} \geq 1-\delta/2
\end{talign}
and $\ckk(\ys)$ and $\mkk(\ys, \delta, 2^{\ossymb+1}\sqrt{n})$ satisfy the sub-Gaussian property~\cref{eq:subgaussian_params} defined above when $\inputcoreset=\ys$.
See \cref{rem:frak_M_P} for upper bounds on the quantities $\ckk(\Q, n, \delta)$ and $\mkk(\Q,n,\delta, 2^{\ossymb+1}\sqrt{n})$ for standard choices of $\kersplit$ and \cref{sub:proof_of_tab_1} for how that translates to a scaling for $\errorsplit$ for the settings in \cref{table:error_tails}.

\subsection{Proof of claim \cref{eq:mmd_diff_x_y}} \label{sub:proof_mmd_diff_x_y}
We follow the notation of \citet[Apps.~A and C]{shetty2022distribution} and note that \ktcompress (\cref{algo:ktcompress}) is the \compress algorithm of \citet{shetty2022distribution} with, in the notation of \citet[Example~4]{shetty2022distribution}, $\halve = \kt(\frac{\l^2}{n 4^{\ossymb+1} (\beta_n+1)} \delta)$ and $\beta_n \defeq \log_4 n \!-\! \ossymb \!-\! 1$ for an input of size $\l$.

We associate with each algorithm $\mathrm{ALG}$ and each input point sequence $\inputcoreset$ of size $n$ and output $\outputcoreset$ of size $n_{out}$ the measure difference
\begin{talign}
    \phi_{\mathrm{ALG}}(\inputcoreset) = \frac{1}{n} \sum_{x \in \inputcoreset} \delta_x - \frac{1}{n_{out}} \sum_{x \in \outputcoreset} \delta_x,
\end{talign}
and the unnormalized measure difference
\begin{talign}
    \psi_{\mathrm{ALG}}(\inputcoreset) = n \cdot \phi_{\mathrm{ALG}}(\inputcoreset) = \sum_{x \in \inputcoreset} \delta_x - \frac{n}{n_{out}} \sum_{x \in \outputcoreset} \delta_x.
\end{talign}
\citet[Eqn.~18]{shetty2022distribution} %
show that for their $\compress$ algorithm, the following holds: %
\begin{talign} \label{eq:psi_equality}
\psi_{\mathrm{C}}(\inputcoreset) = \sqrt{n} 2^{-\ossymb-1} \sum_{i=0}^{\beta_n} \sum_{j=1}^{4^{i}} 2^{-i} \psi_{H}(\cset^{\mathrm{in}}_{i,j}),
\end{talign}
where $\beta_n \defeq \log_4 n - \ossymb - 1$, ${(\cset^{\mathrm{in}}_{i,j})}_{j \in [4^i]}$ are the $4^i$ coresets of size $n_i = 2^{\ossymb + 1 - i} \sqrt{n}$ resulting from $i$ recursive calls to the $\compress$ algorithm, and $\psi_{H}$ is the unnormalized measure difference for \halve. 
Substituting $\inputcoreset \gets \xs^{(r)}$, we get that for $r=1,\dots, \sblock$, $\psi_{\mathrm{C}}(\xs^{(r)}) = \sqrt{m/\sblock_m} \cdot 2^{-\ossymb-1} \sum_{i=0}^{\beta_{m/\sblock_m}} \sum_{j=1}^{4^{i}} 2^{-i} \psi_{H}(\mathbb{X}^{(r)}_{m,i,j})$, where $\mathbb{X}^{(r)}_{m,i,j}$ is defined analogously to $\cset^{\mathrm{in}}_{i,j}$.
If we let $\mathrm{C+}$ be the algorithm that maps $\xs$ to $\hatxs$ (and $\ys$ to $\hatys$), we obtain
\begin{talign}
    \psi_{\mathrm{C+}}(\xs^{(r)}) = \sqrt{m/\sblock_m} \cdot 2^{-\ossymb-1} \sum_{r=1}^{\sblock_m} \sum_{i=0}^{\beta_{m/\sblock_m}} \sum_{j=1}^{4^{i}} 2^{-i} \psi_{H}(\mathbb{X}^{(r)}_{m,i,j}).
\end{talign}
Similarly,
\begin{talign}
    \psi_{\mathrm{C+}}(\ys^{(r)}) = \sqrt{n/\sblock_n} \cdot 2^{-\ossymb-1} \sum_{r=1}^{\sblock_n} \sum_{i=0}^{\beta_{n/\sblock_n}} \sum_{j=1}^{4^{i}} 2^{-i} \psi_{H}(\mathbb{Y}^{(r)}_{n,i,j}).
\end{talign}
Following App. C.1 from \citet{shetty2022distribution}, if one numbers the elements of $\inputcoreset$ as $(\x_1,\dots, \x_n)$, and defines the $n \times n$ kernel matrix $\mbf{K} \defeq (\kernel(x_i, x_j))_{i,j=1}^{n}$, one obtains
\begin{talign}
    \label{eq:def_u_and_v}
    u_{k, j} \!\defeq\! \mbf{K}^{\half}  \sum_{i=1}^n e_i \parenth{\indicator(x_i\! \in \!\cset^{\mathrm{in}}_{ k , j}) \!- \! 2\!\cdot\! \indicator(x_i\! \in\! \cset^{\mrm{out}}_{k, j})},  \stext{and}
    u_{\compresssub} \!\defeq\! \sum_{k=0}^{\log_4 n \!-\!\ossymb\!-\!1} \sum_{j=1}^{4^{k}} w_{k, n} u_{k, j},
\end{talign}
where $w_{k, n} \!\defeq\!\frac{\sqrt{n}}{2^{\ossymb+1+k}}$.
    Then, we have
         \begin{talign}
         \label{eq:knorm_twonorm}
               &
               n^2 \cdot \mmd_\kernel^2\left( \inputcoreset , \cset_{ \compresssub} \right) = \twonorm{u_{\compresssub}}^2,
          \qtext{and}\\
          &\E[u_{k,j} \vert (u_{k', j'} : j' \in [4^{k'}], k' > k)] = 0
          \qtext{for} k = 0, \ldots, \log_4 n \!-\!\ossymb\!-\!2,
         \label{eq:zero_mean_vector}
         \end{talign}
         and $u_{k,j}$ for $ j \in [4^{k}]$ are conditionally independent given $(u_{k', j'} : j' \in [4^{k'}], k' > k)$. This follows easily from \cref{eq:psi_equality}. For any $u \in \R^n$ for arbitrary $n$, we also define 
         \begin{talign}
    \genmat_{u}
          &\!\defeq \!
          \begin{pmatrix}
            0 & u^{\top} \\
            u & \mbf{0}_{n \times n} 
        \end{pmatrix} \in  \reals^{ (n+1) \times (n+1)}.
        \label{eq:def_M}
    \end{talign}
    For any $u \in \real^n$, the matrix $\genmat_{u}$ satisfies
         \begin{talign}
            \opnorm{\genmat_{u}} =  \twonorm{u} = \lambda_{\max}(\genmat_{u}), \qtext{and} 
               \genmat_u^q  \preceq \twonorm{u}^{q} \mbf{I}_{n+1} \stext{ for all } q \in \natural.
     \label{eq:power_of_M}
        \end{talign}
    Defining the shorthand $\genmat_{k, n} \defeq \genmat_{w_{k, n} u_{k, j}}$, we find that
    \begin{talign}
          n \mmd_\kernel\left( \inputcoreset , \cset_{ \compresssub} \right) 
          \!=\! \twonorm{u_{\compresssub}}
          &\!=\! \lambda_{\max}(\genmat_{u_{\compresssub}})
          \!=\! \lambda_{\max}({\sum_{k=0}^{\log_4 n \!-\!\ossymb\!-\!1} \sum_{j=1}^{4^{k}} \genmat_{k, j}}),
          \label{eq:mcp_op}
    \end{talign}

Let $\mathbb{U} = (U_i)_{i=1}^{m+n}$ be the sequence obtained as the concatenation of $\xs$ and $\ys$.  Define the matrix $\mbf{K}_{\xs,\ys} \defeq (\kernel(U_i, U_j))_{i,j=1}^{m+n}$.
Substituting $\inputcoreset \gets \xs^{(r)}$, we can write
\begin{talign}
    \label{eq:def_u_and_v_xs}
    &u_{k, j, \xs}^{(r)} \!\defeq\! \mbf{K}_{\xs,\ys}^{\half}  \sum_{i=1}^n e_i \parenth{\indicator(x_i\! \in \!\mathbb{X}^{(r)}_{m,k,j}) \!- \! 2\!\cdot\! \indicator(x_i\! \in\! \mathbb{X}^{(r),\mathrm{out}}_{m,k,j})}, \\ &\stext{and}
    u_{\compresssub,\xs}^{(r)} \!\defeq\! \sum_{k=0}^{\log_4 (m/\sblock_m) \!-\!\ossymb\!-\!1} \sum_{j=1}^{4^{k}} w_{k, m/\sblock_m} u_{k, j,\xs}^{(r)}, \quad \genmat_{k, j, \xs}^{(r)} \defeq \genmat_{w_{k, m/\sblock} u_{k, j,\xs}^{(r)}}
\end{talign}
Hence, we can define
\begin{talign}
    u_{\mathrm{C+},\xs} \!\defeq\! \sum_{r=1}^{\sblock_m} \sum_{k=0}^{\log_4 (m/\sblock_m) \!-\!\ossymb\!-\!1} \sum_{j=1}^{4^{k}} w_{k, m/\sblock_m} u_{k, j,\xs}^{(r)}.
\end{talign}
Analogously,
\begin{talign}
    u_{\mathrm{C+},\ys} \!\defeq\! \sum_{r=1}^{\sblock_n} \sum_{k=0}^{\log_4 (n/\sblock_n) \!-\!\ossymb\!-\!1} \sum_{j=1}^{4^{k}} w_{k, n/\sblock_n} u_{k, j,\ys}^{(r)}.
\end{talign}
Also, note that 
\begin{align}
|\mmd_\kernel\left( \xs, \ys \right) - \mmd_\kernel\left( \hatxs, \hatys \right) | &= |\|(\P-\Q)\kernel\|_{\kernel} - \|(\P_m-\Q_n)\kernel\|_{\kernel}| \leq \|(\P-\Q-(\P_m - \Q_n))\kernel\|_{\kernel} \\ &= \|(\P-\P_m)\kernel - (\Q-\Q_n)\kernel\|_{\kernel}
\end{align}
which implies that
\begin{talign}
    &|\mmd_\kernel\left( \xs, \ys \right) - \mmd_\kernel\left( \hatxs, \hatys \right) | %
    \!\leq\! \twonorm{u_{\mathrm{C+},\xs}/m - u_{\mathrm{C+},\ys}/n} \\ &= \twonorm{\sum_{r=1}^{\sblock_m} (\frac{1}{m}\sum_{k=0}^{\log_4 (m/\sblock_m) \!-\!\ossymb\!-\!1} \sum_{j=1}^{4^{k}} w_{k, m/\sblock_m} u_{k, j,\xs}^{(r)} - \frac{1}{n}\sum_{k=0}^{\log_4 (n/\sblock_n) \!-\!\ossymb\!-\!1} \sum_{j=1}^{4^{k}} w_{k, n/\sblock_n} u_{k, j,\ys}^{(r)})} \\ &= \lambda_{\max}({\sum_{r=1}^{\sblock_m} (\frac{1}{m}\sum_{k=0}^{\log_4 (m/\sblock_m) \!-\!\ossymb\!-\!1} \sum_{j=1}^{4^{k}} \genmat_{k, j,\xs}^{(r)} - \frac{1}{n} \sum_{k=0}^{\log_4 (n/\sblock_n) \!-\!\ossymb\!-\!1} \sum_{j=1}^{4^{k}} \genmat_{k, j,\ys}^{(r)})}).
\end{talign}
Now we apply the sub-Gaussian matrix Freedman inequality \citep[Lem.~4]{shetty2022distribution}. 
The zero-mean condition on the matrices holds following the argument in \citet[Sec.~C.3.1]{shetty2022distribution}, while for the moment bounds we use the approach in their Sec.~C.3.2. Namely, we use that for any $q \in 2 \N$,
\begin{talign}
(\genmat_{k, j, \xs}^{(r)})^q = \genmat_{w_{k, m/4} u_{k, j,\xs}^{(r)}}^q \preceq \|w_{k, m/4} u_{k, j,\xs}^{(r)}\|_2^q \mbf{I}_{n+1} = w_{k, m/4}^q \|u_{k, j,\xs}^{(r)}\|_2^q \mbf{I}_{m+n+1},
\end{talign}
and similarly, $(\genmat_{k, j, \xs}^{(r)})^q \preceq w_{k, m/4}^q \|u_{k, j,\xs}^{(r)}\|_2^q \mbf{I}_{m+n+1}$. \citet[Lem.~5]{shetty2022distribution} prove that for any non-negative random variable $Z$,
\begin{talign}
     \P[ Z \!>\! a \!+\! v \sqrt{ t} ] \!\leq\! e^{-t } 
     \ \stext{for all} t\geq0
     \ \Longrightarrow \
      \E [Z^{q}] 
      \leq  (2a\!+\!2v)^q (\frac{q}{2})!
      \ \stext{for all} q \in 2\natural.
  \end{talign}
In their case, their $\kernel$-sub Gaussian assumption on \halve implies that
\begin{talign}
    \label{eq:psi_halve_bound}
        \P [  \|u_{k,j}\|_2 \geq  \l_{k}' (a_{\l_{k}'} + v_{\l_{k}'}  \sqrt{ t})  \mid (u_{k',j'}: j' \in [4^{k'}], k' >k) ] \leq e^{-t} \stext{for all} t\geq 0,
    \end{talign}
    for suitable scalar sequences $\sbraces{a_{\l}, v_{\l}}$ (also see \cref{eq:subgaussian_params}),
    where 
    $\l_{k}' \defeq %
    \sqrt{n} 2^{\ossymb+1-k}$,
    which yields the moment bound $\E [\norm{u_{k,j} }_2^{q} \mid (u_{k',j'}: j' \in [4^{k'}], k' >k)] 
         \leq  (\frac{q}{2})! (2\l_{k}' (a_{\l_{k}'}+v_{\l_{k}'}))^q$. Under the same assumption on \halve, we obtain analogously that 
    \begin{talign}
    \E [\|u_{k, j,\xs}^{(r)}\|_2^{q} \mid (u_{k', j',\xs}^{(r)}: j' \in [4^{k'}], k' >k)] 
    &\leq  (\frac{q}{2})! (2 \l_{k,m/\sblock_m}' (a_{\l_{k,m/\sblock_m}'}+v_{\l_{k,m/\sblock_m}'}))^q, \\ \E [\|u_{k, j,\ys}^{(r)}\|_2^{q} \mid (u_{k', j',\ys}^{(r)}: j' \in [4^{k'}], k' >k)] 
    &\leq  (\frac{q}{2})! (2 \l_{k,n/\sblock_n}' (a'_{\l_{k,n/\sblock_n}'}+v'_{\l_{k,n/\sblock_n}'}))^q,
    \end{talign} 
    where $\l_{k,m/\sblock_m}' \defeq
    \sqrt{m/\sblock_m} \cdot 2^{\ossymb+1-k}$,
    $\l_{k,n/\sblock_n}' \defeq
    \sqrt{n/\sblock_n} \cdot 2^{\ossymb+1-k}$.

    Now let $\sbraces{a_{\l, n}, v_{\l, n}}$ denote the scalar sequences from \cref{eq:subgaussian_params} so that 
     \compress with $\halve = \kt(\frac{\l^2}{n 4^{\ossymb+1} (\beta_n+1)} \delta)$ for input of size $\l$, every \halve call invoked by \compress is $\kernel$-sub-Gaussian with parameters $a_{\l, n}, v_{\l, n}$ on an event of probability at least $1-\frac{\delta}{2}$.
    We define 
    \begin{talign}
    \sigma^2 \defeq \sum_{r=1}^{\sblock_m} (&\sum_{k=0}^{\log_4 (m/\sblock_m) \!-\!\ossymb\!-\!1} \sum_{j=1}^{4^{k}} (\frac{2}{m} w_{k,m/\sblock_m} %
    \l_{k,m/\sblock_m}' (a_{\l_{k,m/\sblock_m}',m/\sblock_m} + v_{\l_{k,m/\sblock_m}',m/\sblock_m}))^2 \\ &+ \sum_{k=0}^{\log_4 (n/\sblock_n) \!-\!\ossymb\!-\!1} \sum_{j=1}^{4^{k}} (\frac{2}{m} w_{k,n/\sblock_n} \l_{k,n/\sblock_n}' (a'_{\l_{k,n/\sblock_n}',n/\sblock_n}+v'_{\l_{k,n/\sblock_n}',n/\sblock_n}))^2),
    \end{talign}
    which when combined with the expressions for $a_{\l, n}$ and $v_{\l, n}$ from \cref{eq:subgaussian_params}, yields that
\begin{talign} \label{eq:sigma2_def1}
\sigma^2 &= \sum_{k=0}^{\log_4 (m/\sblock_m) \!-\!\ossymb\!-\!1} 
\parenth{\frac{4}{2^{\ossymb}\sqrt{m}}
\parenth{\ckk(\xs) +  \mkk(\xs, \delta, 2^{\ossymb+1-k} \sqrt{\frac{m}{\sblock_m}} ) \sqrt{\log( \frac{6 \cdot 4^{\ossymb} \sqrt{m/\sblock_m} (\beta_{m/\sblock_m}+1)}{2^{\ossymb-k} \delta})} }}^2 \\ 
& \quad+ \sum_{k=0}^{\log_4 (n/\sblock_n) \!-\!\ossymb\!-\!1}
\parenth{\frac{4}{2^{\ossymb}\sqrt{n}} \parenth{\ckk(\ys) + \mkk(\ys, \delta, 2^{\ossymb+1-k} \sqrt{\frac{n}{\sblock_n}})\sqrt{\log( \frac{6 \cdot 4^{\ossymb} \sqrt{n/\sblock_n} (\beta_{n/\sblock_n}+1)}{2^{\ossymb-k} \delta})}}}^2 
\\ &\leq \frac{16(\log_4 (m/\sblock_m) \!-\!\ossymb\!-\!1)}{4^{\ossymb}m} \parenth{\ckk(\xs) + \mkk(\xs, \delta, 2^{\ossymb+1} \sqrt{\frac{m}{\sblock_m}}) \sqrt{\log( \frac{3 m (\log_4 (m/\sblock_m) \!-\!\ossymb\!-\!1)}{\sblock_m \delta})}}^2 
\\ &\quad+ \frac{16(\log_4 (n/\sblock_n) \!-\!\ossymb\!-\!1)}{4^{\ossymb}n} \parenth{\ckk(\ys) + \mkk(\ys, \delta,2^{\ossymb+1} \sqrt{\frac{n}{\sblock_n}}) \sqrt{\log( 
\frac{3 n (\log_4 (n/\sblock_n) \!-\!\ossymb\!-\!1)}{\sblock_n \delta})}}^2,
\label{eq:sigma2_def2}
\end{talign}
where in the last inequality we also use the fact that $\mkk$ is non-decreasing in its last argument.
Now, by the sub-Gaussian matrix Freedman inequality as stated in \citep[Lem.~4]{shetty2022distribution}, we obtain that%
\begin{talign} \label{eq:matrix_friedman_app}
    &\mathrm{Pr}(|\mmd_\kernel\left( \xs, \ys \right) - \mmd_\kernel( \hatxs, \hatys ) | > \sigma \sqrt{8(\log(m+n+1)+t)}) \\ &\leq \mathrm{Pr}(\lambda_{\max}({\sum_{r=1}^{4} (\frac{1}{m}\sum_{k=0}^{\beta_{m/4}} \sum_{j=1}^{4^{k}} \genmat_{k, j,\xs}^{(r)} \!\!\!-\! \frac{1}{n} \sum_{k=0}^{\beta_{n/4}} \sum_{j=1}^{4^{k}} \!\genmat_{k, j,\ys}^{(r)})})\! >\! \sigma \sqrt{8(\log(m\!+\!n\!+\!1)\!+\!t)}) \\ &\leq \frac{\delta}{2} + e^{-t}, \qquad \qtext{for all} t \geq 0.
\end{talign}
The term $\frac{\delta}{2}$ does not come from the sub-Gaussian matrix Freedman inequality but rather from the conditioning on the event for which \cref{eq:subgaussian_params} holds. Equation \cref{eq:matrix_friedman_app} in turn implies that $\mathrm{Pr}(|\mmd_\kernel\left( \xs, \ys \right) - \mmd_\kernel\left( \hatxs, \hatys \right) | > \sigma (\sqrt{8\log(m+n+1)} + \sqrt{8t})) \leq \frac{\delta}{2} + e^{-t}$.
Equivalently,  for any $\delta > 0$, with probability at least $1-\delta$, 
\begin{talign}
|\mmd_\kernel\left( \xs, \ys \right) - \mmd_\kernel\left( \hatxs, \hatys \right) | \leq \sqrt{8} \sigma (\sqrt{\log(m+n+1)} + \sqrt{\log(2/\delta)}).
\label{eq:mmd_bound_sigma}
\end{talign}
Putting together the bound~\cref{eq:mmd_bound_sigma} with the upper bound~\cref{eq:sigma2_def2} and the definition~\cref{eq:errorsplit_def} of $\errorsplit$ immediately yields the claimed bound~\cref{eq:mmd_diff_x_y} and we are done.

\subsection{Proof of claim \cref{eq:mmd_diff_p_q}}
Note that
\begin{talign}
    \mathrm{MMD}(\P,\Q) - \mathrm{MMD}(\hat{\mathbb{X}}_m,\hat{\mathbb{Y}}_n) = \mathrm{MMD}(\P,\Q) - \mathrm{MMD}(\xs,\ys) + \mathrm{MMD}(\xs,\ys) - \mathrm{MMD}(\hat{\mathbb{X}}_m,\hat{\mathbb{Y}}_n).
\end{talign}
The second term in the display above can be bounded via \cref{eq:mmd_diff_x_y} and the first term via the following result from \cite{gretton2012akernel}: 
\begin{lemma}[Adapted from Theorem 7, \cite{gretton2012akernel}] \label{lem:concentration_mmd}
Assume that $\kinfnorm < +\infty$. Then,
\begin{talign}
    \mathrm{Pr}\bigg[|\mathrm{MMD}(\xs,\ys)-\mathrm{MMD}(\P,\Q)| > 2 \bigg( \sqrt{\frac{\kinfnorm}{m}} +\sqrt{\frac{\kinfnorm}{n}} \bigg) +\epsilon \bigg] \leq 2 \exp \bigg( \frac{-\epsilon^2 m n}{2 \kinfnorm (m+n)} \bigg).
\end{talign}
\end{lemma}
Using \cref{lem:concentration_mmd} with $%
\delta = 2 \exp \bigg( \frac{-\epsilon^2 m n}{2 \kinfnorm (m+n)} \bigg)$, which is equivalent to $\epsilon = \sqrt{ \frac{2 \kinfnorm (m+n)}{m n}\log\left(\frac{2}{\delta} \right)}$,
we obtain that 
\begin{talign} \label{eq:MMD_difference_2}
    \mathrm{Pr}\bigg[|\mathrm{MMD}(\xs,\ys)-\mathrm{MMD}(\P,\Q)| > 2 \big( \sqrt{\frac{\kinfnorm}{m}} +\sqrt{\frac{\kinfnorm}{n}} \big) + \sqrt{ \frac{2\kinfnorm (m+n)}{m n}\log\left(%
    \frac{2}{\delta} \right)} \bigg] \leq %
    \delta,
\end{talign}
where the bound~\cref{eq:sigma2_def2} on $\sigma$ depends on $(\xs, \ys)$.
Next, the bound~\cref{eq:mmd_diff_x_y} states that with probability at least %
$1-\delta$, 
\begin{talign} \label{eq:MMD_difference_1}
|\mmd( \xs, \ys ) - \mmd( \hatxs, \hatys ) | \leq %
\frac{\errorsplit(\xs, m,\delta, \ossymb)}{2^{\ossymb} \sqrt{m}} \!+  \! \frac{\errorsplit(\ys, n,\delta\ossymb)}{2^{\ossymb} \sqrt{n}},
\end{talign}
and the definition of $\errorsplit(\P, m, \delta,\ossymb)$ (see the discussion around \cref{eq:M_P_def} in \cref{sub:inflation_factor}) implies that
\begin{talign}
\label{eq:R_bound}
    \P\brackets{\errorsplit(\xs, m,\delta,\ossymb) \leq \errorsplit(\P, m,\delta,\ossymb) \qtext{and} \errorsplit(\ys, n,\delta,\ossymb) \leq  \errorsplit(\Q, n,\delta,\ossymb) } \geq 1-\delta.
\end{talign}
Putting the pieces together yields the bound~\cref{eq:mmd_diff_p_q} with probability at least $1-3\delta$ as claimed.

\subsection{Bounds on $\ckk$ and $\mkk$} \label{rem:frak_M_all}
First, we discuss bounds on the sample-based quantities $\ckk(\xs)$ and $\mkk(\xs, \delta, 2^{\ossymb+1}\sqrt{n})$ defined in \cref{eq:subgaussian_params} followed by bounds on its population analog $\ckk(\P, m, \delta)$ and $\mkk(\P, m, \delta, 2^{\ossymb+1}\sqrt{n})$ defined in 
 \cref{eq:M_P_def}.

\subsubsection{Bounds on sample-level quantities ($\ckk(\xs)$ and $\mkk(\xs, \delta, 2^{\ossymb+1}\sqrt{n})$)}
\label{rem:frak_M}

We discuss the default choices $\kersplit=\ksqrt$ and $\kersplit=\kernel$ and the more general case one-by-one.

\paragraph{Case I: $\kersplit=\ksqrt$} For this case, we follow the discussion from \citet[Sec.~3.1]{dwivedi2021kernel}. Let $\klip[\ksqrt]$ denote the Lipschitz constant of $\ksqrt$ and define $\tail[\ksqrt](R) \defeq (\sup_{x} \int_{\ltwonorm{y}\geq R} \ksqrt^2(\x, \x-\y)d\y)^{\frac{1}{2}}$,
\begin{talign}
    \rmin_{\ksqrt, \l} 
    \!&\defeq\! \min\{r:\!\sup_{\substack{\x, \y: \\ \ltwonorm{\x-\y}\geq r}}\! \abss{\ksqrt(\x,\y)}\!\leq\! \frac{\infnorm{\ksqrt}}{\l}\},  \quad
     \rmin_{\ksqrt,\l}^{\dagger}\!\defeq\! \min\{r:\tail[\ksqrt](r)\! \leq\! \frac{\infnorm{\ksqrt}}{\l}\},
     \label{eq:rmin_k} 
     \\
    \rminpn[\xs]\! &\defeq\! \max_{\x\in \xs}\ltwonorm{\x},
    \qtext{and}
     \rmin_{\xs, \ksqrt, \l} \!\defeq\! \min\big(\rminpn[\xs], \l^{1+\frac1d}\rmin_{\ksqrt, \l}+ \l^{\frac1d} \frac{\sinfnorm{\ksqrt}}{\klip[\ksqrt]} \big),
     \label{eq:rmin_P}
\end{talign}
and the kernel thinning inflation factor
\begin{talign}
 \label{eq:err_simple_defn}  
    \err_{\ksqrt}(\l,\! d,\! \delta,\! R) 
     \!\defeq\!
     37\sqrt{ \log\parenth{\frac{3\l}{\delta}}} \brackets{ \sqrt{\log\parenth{\frac{8}{\delta}}} \!+\! 5 \sqrt{{d\log \left(2+2\frac{\klip[\ksqrt]}{\infnorm{\ksqrt}}\big(\rmin_{\ksqrt, \l} \!+\! R\big) \right) }} }.
\end{talign}
Then using \citet[Thm.~1]{dwivedi2021kernel} in \citet[Example~4]{shetty2022distribution}, we find that
\begin{talign}
\ckk[\kernel,\ksqrt](\xs) = 2\kinfsin[\ksqrt] \qtext{and}
 \mkk[\kernel,\ksqrt](\xs, \delta, \l) \!=\! \kinfsin[\ksqrt] (\max(\rminpn[\xs],\rmin_{\ksqrt, \l/2}^\dagger))^{\frac{d}{2}} \!\cdot\! \err_{\ksqrt}(\l,\! d,\! \delta, \!\rmin_{\xs, \ksqrt, \l}),   
\end{talign}
and $\mkk[\kernel,\kersplit](\ys, \delta, \l)$ defined analogously by replacing $\xs$ by $\ys$ and $m$ by $n$. Here $\ksinfsin = \sup_{x \in \xs[m]} \kersplit(x,x)$.
We note that the bounds in \citet{dwivedi2021kernel,shetty2022distribution} are stated with $\staticinfnorm{\kersplit} \defeq  \sup_{x} \kersplit(x,x)$ instead of $\ksinfsin$ (and note that $\ksinfsin \leq \staticinfnorm{\kersplit}$). However, as noted in \citet[App.~B]{dwivedi2022generalized}, all the results of \citet{dwivedi2021kernel} (which is what \citet{shetty2022distribution} build on) go through with $\staticinfnorm{\kersplit}$ replaced by $\ksinfsin$ thereby yielding the result stated above.

\paragraph{Case II: $\kersplit=\kernel$} For this case, we follow the discussion in \citep[Sec~2.2]{dwivedi2022generalized}. In particular, for a set  $\set{A} \subset \real^d$ and scalar $\vareps>0$, 
define the $\kernel$ \emph{covering number} $\mc N_{\kernel}(\set{A}, \vareps)$ 
with $\mc M_{\kernel}(\set A, \vareps)\defeq\log\mc N_{\kernel}(\set{A}, \vareps)$ as the minimum cardinality of a set $\mc C \subset \mbb B_{\kernel} \defeq \sbraces{f: \knorm{f}\leq 1}$ satisfying
\begin{talign}
\label{eq:cover_ball}
 \mbb B_{\kernel} \subseteq \bigcup_{h \in \mc C} \braces{g \in \mbb B_{\kernel}: \sup_{x\in\set{A}}|h(x)-g(x)|\leq \varepsilon}.
\end{talign}
Then choosing $\varepsilon=\frac{\sqrt{\kinfsin[\kernel]}}{\l/2}$ in the notation of \citet[Thm.~2]{dwivedi2022generalized} and combining that result with \citet[Example~4]{shetty2022distribution}, we conclude that we can use the following bounds
\begin{talign} 
    \ckk[\kernel, \kernel](\xs) = 2\sqrt{\kinfsin[\kernel]} \qtext{and}
	\mathfrak{M}_{\kernel,\kernel}(\xs, \delta, \l) = \sqrt{\frac{8\kinfsin[\kernel]}{3} \log(\frac{12\log \l}{\delta})
	\brackets{ \log(\frac{8}{\delta})+\mc M_{\kernel}(\set{A}_{\xs}, (\l/2)\inv)}},
\label{eq:mmd_guarantee_target_k}
\end{talign}
where $\set{A}_{\xs} = \braces{x: \twonorm{x} \leq \rminpn[\xs]}$. We can define
$\mathfrak{M}_{\kernel,\kersplit}(\ys, \delta, \l)$ analogously by replacing $\xs$ by $\ys$ and $m$ by $n$.

\paragraph{Case III: General $\kersplit$}
When $\kersplit$ is neither of the two default choices ($\kernel$ or $\ksqrt$) like in \actt, then the expressions for $\errorsplit$  can be derived using \citet[Thm.~2-4]{dwivedi2022generalized} and \citet[Thm.~1-2]{dwivedi2021kernel}. For instance, when the RKHS of $\kernel$ is contained in the RKHS of $\kersplit$, we can apply the sub-Gaussian tail bounds for a single $f$ (\citet[Thm.~1]{dwivedi2021kernel}) and then apply a union bound with a covering argument for the ball $\sbraces{\knorm{f}\leq 1}$ (\citet[Thm.~2]{dwivedi2022generalized}), in which case $\errorsplit$ also scales with $\sup_{\knorm{f}\leq 1}\norm{f}_{\kersplit}$.  See \cref{rem:aggk} for an example of this case.

\subsubsection{Bounds on population-level quantities ($\ckk(\P, m, \delta, \ossymb)$ and $\mkk(\P, m,\delta, 2^{\ossymb+1}\sqrt{n})$)}
\label{rem:frak_M_P}
Define
\begin{talign}
    \ckk'  =\begin{cases}
    2\sinfnorm{\ksqrt} \qtext{when} \kersplit=\ksqrt  \\ 
    2\sqrt{\kinfnorm} \qtext{when} \kersplit=\kernel 
    \end{cases}.
\end{talign}
Then for the choices of $\ckk(\cdot)$ in \cref{rem:frak_M}, we have $\max\sbraces{\ckk(\xs), \ckk(\ys)} \leq \ckk'$ almost surely, where $\sinfnorm{\kernel}\defeq \sup_{x\in\mathcal X}\kernel(x, x)$. Thus if we set $\ckk(\P, m, \delta)=\ckk(\Q, n, \delta)= \ckk'$, to satisfy \cref{eq:M_P_def}, it remains to determine $\mkk(\P, m, \delta, 2^{\ossymb+1}\sqrt{n})$ such that $\P\brackets{\mkk(\xs, \delta, 2^{\ossymb+1}\sqrt{n})\leq \mkk(\P, m, \delta, 2^{\ossymb+1}\sqrt{n})} \geq 1-\delta/2$ for the choices of $\mkk$ in \cref{rem:frak_M}.

We now derive a suitable expression for these population-level quantities.
Following \citet{dwivedi2021kernel}, we define $\tail[\P](R) \defeq \P(\ball^c(0,R))$ and $\tail[\Q](R) \defeq \Q(\ball^c(0,R))$ where $\ball^{c}(x, R) = \sbraces{y: \twonorm{x-y}\geq R}$. 
The following result (proven using results on order statistics that we later develop in \cref{sec:arbitrary_permutations}) shows that we can upper-bound $\rminpn[\xs]$ and $\rminpn[\ys]$ with high probability by a quantities that depend on $\P$ and $m$, and $\Q$ and $n$, respectively.
\begin{lemma} \label{prop:rmin_P_xs}
Define $\rminpn[\P,m,\delta] \defeq \tail[\P]^{-1}(\delta/m)$. With probability at least $1-\delta$, we have that $\rminpn[\xs] \leq \rminpn[\P,m,\delta]$. Similarly, with probability at least $1-\delta$, we have that $\rminpn[\ys] \leq \rminpn[\Q,n,\delta] \defeq \tail[\Q]^{-1}(\delta/n)$.
\end{lemma}
\begin{proof}
The random variable $\rminpn[\xs]\! \defeq\! \max_{\x\in \xs}\ltwonorm{\x}$ is the $m$-th order statistic for $m$ samples of $\P$ (\cref{def:kth_order_stat}). 
Since the function $\tail[\P](R) \defeq \P(\ball^c(0,R)) = 1 - \P(\ball(0,R))$ is one minus the cumulative function of the random variable $\ltwonorm{\x}, \x \sim \P$, we obtain that $\tail[\P](\rminpn[\xs])$ is the first order statistic for $m$ samples of the uniform distribution over $[0,1]$. Applying \cref{lem:I_x}(iv) on $1 - \tail[\P](\rminpn[\xs])$, we obtain that
\begin{talign}
\mathrm{Pr}(\tail[\P](\rminpn[\xs]) < x) = \mathrm{Pr}(1-\tail[\P](\rminpn[\xs]) > 1-x) = {m \choose m-1} x^{m+1-m} = m x.
\end{talign}
Hence, with probability at least $1-\delta$, $\rminpn[\xs] < \tail[\P]^{-1}(\delta/m)$, and similarly $\rminpn[\ys] < \tail[\P]^{-1}(\delta/n)$.
\end{proof}
Now we can set %
\begin{talign} \label{eq:frak_M_population}
\mathfrak{M}_{\kernel,\kersplit}(\P,m,\delta, \l ) &= \sinfnorm{\ksqrt} \max(\rminpn[\P,m,%
\delta/2],\rmin_{\ksqrt, \l/2}^\dagger)^{\frac{d}{2}}\, \err_{\ksqrt}(\l,\! d,\! %
\delta,\!\rminpn[\P,m, \delta/2, \ksqrt, \l ])
\stext{when} \kersplit=\ksqrt
\qtext{and}
\\ 
\mathfrak{M}_{\kernel,\kersplit}(\P,m,\delta, \l)  &=
\sqrt{\frac{8\sinfnorm{\kernel}}{3} \log(\frac{12\log m}{\delta})
	\brackets{ \log(\frac{8}{\delta})+\mc M_{\kernel}(\set{A}_{\P, m, \delta}, (\l/2)\inv)}}
\stext{when} \kersplit=\kernel,
\label{eq:frak_M_population_k}
\end{talign}
where $\rminpn[\P,m, \delta/2, \ksqrt, \l ] \defeq \min\big(\rminpn[\P,m,
\delta/2], \l^{1+\frac1d}\rmin_{\ksqrt, \l}+ \l^{\frac1d} \frac{\sinfnorm{\ksqrt}}{\klip[\ksqrt]} \big)$
and $\set{A}_{\P, m, \delta} = \braces{x: \twonorm{x} \leq \rminpn[\P,m,\delta/2]}$.
By \cref{prop:rmin_P_xs}, we have that with probability at least $1-%
\delta/2$, $\rminpn[\xs] \leq \rminpn[\P,m,%
\delta/2]$, and by construction (see \cref{rem:frak_M}), %
$\P(\mathfrak{M}_{\kernel,\kersplit}(\xs, \delta, \l) \leq \mathfrak{M}_{\kernel,\kersplit}(\P,m,\delta,\l))\geq 1-\delta/2$ as needed above. 

\subsection{Proof of \cref{table:error_tails}}
\label{sub:proof_of_tab_1}
We begin by showing the bounds on $\errorsplit(\P, m,\delta, \ossymb)$ for the cases in which $\kersplit = \ksqrt$, and we include the case in which $\kersplit$ and $\P$ are sub-Gaussian for completeness. For $\ossymb \leq \log m$, equation \cref{eq:errorsplit_def} implies that the quantity $\errorsplit(\P, m,\delta,\ossymb)$ is of order $c_d\sqrt{\ksplitinfnorm\log(m\!+\!n) \log(m) \log(\frac{m}{\delta})} \cdot \mathfrak{M}_{\kernel,\kersplit}(\P,m,\delta, 2^{\ossymb+1}\sqrt{m})$. Hence, we seek to upper-bound $\mathfrak{M}_{\kernel,\kersplit}(\P,m,\delta,2^{\ossymb+1}\sqrt{m})$. In the table we replace the factor $\sqrt{\log(m\!+\!n) \log(m)}$ by $\log(\frac{m}{\delta})$ for simplcity.

First, note that upper bounds on $\tail[\P](x) \defeq \P(\ball^c(0,x))$, which are the usual notion of tail bounds on distributions, are equivalent to upper bounds on $\rminpn[\P,m,\delta] \defeq \tail[\P]^{-1}(\delta/m)$. Namely,
\begin{itemize}
\item \textbf{$r$-Compact:} $\tail[\P](x) = 0, \ \forall x > r \iff \tail[\P]^{-1}(\delta/m) \leq r, \, \forall \delta, m,$
\item \textbf{$\sigma$-Sub-Gaussian:} $\tail[\P](x) \leq 2 e^{-\frac{x^2}{2\sigma^2}} \iff \tail[\P]^{-1}(\delta/m) \leq \sqrt{2 \sigma^2 \log(\frac{2m}{\delta})}, \, \forall \delta, m,$
\item \textbf{$\sigma,\lambda$-subexponential:} $\tail[\P](x) \leq 2 \max\{ e^{-\frac{x^2}{2\sigma^2}}, e^{-\frac{x}{2\lambda}} \} 
    \iff \tail[\P]^{-1}(\delta/m) \leq \max\{ \sqrt{2 \sigma^2 \log(\frac{2 m}{\delta})}, 2 \lambda \log(\frac{2 m}{\delta}) \}, \, \forall \delta, m,$
\item \textbf{$\rho$-Heavy-Tailed:} $\tail[\P](x) \leq c_d r^{-\rho} \iff \tail[\P]^{-1}(\delta/m) \leq (\frac{c_d m}{\delta})^{1/\rho}, \forall \delta, m.$
\end{itemize}

Second, define $\wtil{\rmin}_{\ksqrt, m} = \max\{ \rmin_{\ksqrt,m}, \rmin_{\ksqrt,m}^{\dagger}\}$. Following \citet{dwivedi2021kernel}, we formulate bounds on the decay of $\kersplit$ in terms of bounds on $\rmin_{\ksqrt}$.

In \cref{table:error_tails}, we consider four different growth conditions for the input point radii $\rminpn[\xs]$ arising from four forms of the target distribution and kernel tail decay (assuming same decay for both $\P$ and $\kersplit$): (1) {\bf Compact}: $\rminpn[\P,m,\delta] \precsim_{d} r, \wtil{\rmin}_{\ksqrt, m} \precsim_{d} 1$, (2) {\bf Sub-Gaussian}: $\rminpn[\P,m,\delta] \precsim_d \sigma \sqrt{\log(m/\delta)}, \wtil{\rmin}_{\ksqrt, m} \precsim_d \sqrt{\log m}$, 
(3) {\bf subexponential}: $\rminpn[\P,m,\delta] \precsim_d \lambda \log(m/\delta), \wtil{\rmin}_{\ksqrt, m} \precsim_d \log m$, and (4) {\bf Heavy-Tailed}: $\rminpn[\P,m,\delta] \precsim_d (m/\delta)^{1/\rho}, \wtil{\rmin}_{\ksqrt, m} \precsim_d m^{1/\rho})$. Here, the notation $\precsim_d$ means that factors depending on $d$ and $\delta$ are hidden.
The first condition holds when $\P$ is supported on a compact set like the unit cube $[0, 1]^d$. %

To get the bounds in the table, we observe that
\begin{talign}
    \mathfrak{M}_{\kernel,\kersplit}(\P,m,\delta, 2^{\ossymb+1}\sqrt{m}) = O_{d}\big( \max(\rminpn[\P,m,\delta/6],\rmin_{\ksqrt, 2^{\ossymb}\sqrt{m}}^\dagger)^{\frac{d}{2}} \sqrt{\log (\frac{m}{\delta}) \cdot \log(\max(\rminpn[\P,m,\delta/6],\rmin_{\ksqrt, 2^{\ossymb}\sqrt{m}})) }\big),
\end{talign}
where $O_{d}$ hides constants that depend on $d$. We now plug in the bounds on $\rminpn[\P,m,\delta], \wtil{\rmin}_{\ksqrt, m}$ for each of the four cases to obtain the following scaling for $\mathfrak{M}_{\kernel,\kersplit}(\P,m,\delta, 2^{\ossymb+1}\sqrt{m})$ (and simplifying expressions by using $\ossymb\leq\frac{1}{2}\log_2 m$):
\begin{itemize}
    \item \textbf{$r$-Compact:} $O_d(r^{\frac{d}{2}} \sqrt{\log (\frac{m}{\delta}) \cdot \log r})$
    \item \textbf{$\sigma$-Sub-Gaussian:} $O_d(\sigma^{\frac{d}{2}} \log(\frac{m}{\delta})^{\frac{d+2}{4}} \sqrt{\log (\log(\frac{m}{\delta}))})$
    \item \textbf{$\sigma,\lambda$-subexponential:} $O_d(\lambda^{\frac{d}{2}} \log(\frac{m}{\delta})^{\frac{d+2}{2}} \sqrt{\log (\log(\frac{m}{\delta}))})$
    \item \textbf{$\rho$-Heavy-Tailed:} $O_d((\frac{m}{\delta})^{\frac{d}{2\rho}}  \log(\frac{m}{\delta}))$
\end{itemize}

To show the bound for $\kersplit=\kernel$ with an analytic $\kernel$ and $\P$ with subexponential tails, we follow the pointers in \cref{rem:frak_M}. Putting together \cref{eq:frak_M_population_k} and \citet[Thm.~2]{dwivedi2022generalized}, we find that for this case $\errorsplit(\P, m,\delta, \ossymb)$ is of order $c_d\sqrt{\kinfnorm\log(m\!+\!n) \log(m) \log(\frac{m}{\delta})} \cdot \mathfrak{M}_{\kernel, \kernel}(\P, m, \delta, 2^{\ossymb+1}\sqrt{m})$,
and doing algebra with \cref{eq:mmd_guarantee_target_k}, we conclude that $\mathfrak{M}_{\P,m,\kernel}(\delta)$  scales linearly with the square-root of the log-covering number $\mc M_{\kernel}$. 
\citet[Prop.~2(a)]{dwivedi2022generalized} states that the kernel covering number $\mc M_{\kernel}$ admits the scaling $(\log(1/\varepsilon))^{d+1}$ times the Euclidean covering number in $\real^d$ that admits a scaling of $r^{d}$ for a Euclidean ball of radius $r$ (see \citet[Lem.~5.7]{wainwright2019high}).
Consequently, using the \textsc{LogGrowth $\mc M_{\kernel}$} and subexponential $\P$ column with $\omega=d+1$ in \citet[Tab.~2]{dwivedi2022generalized} shows that for this case $\mathfrak{M}_{\kernel, \kernel}(\P, m, \delta, 2^{\ossymb+1}\sqrt{m})=O(\log \frac{m}{\delta})^{\frac{3d+2}{2}}$, which in turn implies the corresponding scaling in \cref{table:error_tails} for $\errorsplit(\P, m,\delta, \ossymb)$, where once again we have used the fact that $\ossymb \leq \half\log_2 m$ to simplify expressions.
\section{\pcref{thm:validity}}\label{proof-validity}
We will prove the result under the weaker assumption that the point sequence $(X_1, \dots, X_m, Y_1, \dots, Y_n)$ is exchangeable.
Under this assumption the statistics $(M_b)_{b=1}^{\numperm+1}$ are also exchangeable.
Since $R$ represents the position of $M_{\numperm+1}$ in a sorted ordering of $(M_b)_{b=1}^{\numperm+1}$ with ties broken uniformly at random and all positions in $\{1, \dots, \numperm+1\}$ are equally likely under exchangeability, 
\begin{talign}
\Pr[R = b_\alpha] &= 1/(\numperm+1), \\
\Pr[R > b_\alpha] &= (\numperm+1-b_\alpha)/(\numperm+1), \qtext{and} \\
\Pr[R < b_\alpha] &= (b_\alpha-1)/(\numperm+1).
\end{talign}
Therefore, the \ctt probability of rejection is
\begin{talign}
\Pr[\Delta(\xs,\ys) = 1]
    &= 
\Pr[R > b_\alpha] 
    +
\Pr[R = b_\alpha] p_\alpha \\
    &= 
(\numperm+1-b_\alpha)/(\numperm+1)
    +
(b_\alpha - (1-\alpha)(\numperm+1))/(\numperm+1)
    =
\alpha.
\end{talign}

\section{\pcref{thm:uniform_separation}}
\label{proof-uniform_separation}

We first state a detailed version of \cref{thm:uniform_separation}.
\subsection{\ncref{thm:uniform_separation_detailed}}%
In this section we will prove the following theorem, which is the detailed statement of the result in \cref{thm:uniform_separation}.
\begin{theorem}[Power of CTT, detailed] \label{thm:uniform_separation_detailed}
Suppose \cttname (\cref{algo:ctt}) is run with level $\alpha$, replication count $\numperm \geq \frac{1}{\alpha}- 1$, and coreset count %
$\sblock_m \geq (32/9) \log(\frac{2e}{\gamma})$  for
    $\gamma \defeq \frac{\alpha}{4e} (\frac{\beta}{4+2\beta})^{\frac{1}{\floor{\alpha (\numperm\!+\!1)}}}$. Let $\tilde{\beta} = \beta/(1+\beta/2)$.
Then \ctt 
has power 
\begin{talign} \label{eq:power_condition_main}
    \Pr[\deltactt(\mathbb{X}_m,\mathbb{Y}_n)=1] &\geq 1\!-\!\beta 
    \end{talign}
whenever
\begin{talign} \label{eq:detailed_mmd_lower_bound}
\mmd_{\kernel}(\P,\Q) \!\geq \! 32 \bigg(&\frac{\errorsplit(\P, m/\sblock_m, \tilde{\beta}
    /6),\ossymb}{2^{\ossymb} \sqrt{m}} +  \frac{\errorsplit(\Q, n/\sblock_n, \tilde{\beta}
    /6,\ossymb)}{2^{\ossymb} \sqrt{n}} +c_{\tilde{\beta}
    /6} \bigparenth{\sqrt{\frac{\kinfnorm}{m}}\!+\!\sqrt{\frac{\kinfnorm}{n}}} + 2\big(\sqrt{\frac{9}{32}} + 1\big)\big(2+c' \sqrt{\log(%
    \gamma)} \big) \times \\ & 
    \big(\frac{\errorsplit(\P, m/\sblock_m, \tilde{\beta}/(20\sblock_m),\ossymb
    )}{2^{\ossymb} \sqrt{m}} \!+ \sqrt{\frac{\sblock_n}{\sblock_m}} \frac{ \errorsplit(\Q, n/\sblock_n, \tilde{\beta}/(20\sblock_n),\ossymb
    )}{2^{\ossymb} \sqrt{n}} \!+ \!c_{\tilde{\beta}/(20\sblock)} \big(\sqrt{\frac{\kinfnorm}{m}} +\sqrt{\frac{\sblock_n \kinfnorm}{\sblock_m  n}} \big) \big)
    \bigg).
\end{talign}
\end{theorem}

\begin{remark}
    \cref{thm:uniform_separation} follows from this result as 
    \begin{talign}\sqrt{\sblock_n/\sblock_m} \cdot \errorsplit(\Q, n/\sblock_n, \tilde{\beta}/(20\sblock_n),\ossymb
    )/(2^{\ossymb} \sqrt{n}) = 
    \errorsplit(\Q, n/\sblock_n, \tilde{\beta}/(20\sblock_n),\ossymb
    )/(2^{\ossymb} \sqrt{m})
    \end{talign} 
    and since $20s > 6$ and $m\leq n$.%
\end{remark}

We introduce some notation that we use throughout the proof. First, we let $(M_{(b)})_{b=1}^{\numperm}$ be the increasing ordering of the permuted MMD values. Recall that $b_{\alpha} = \ceil{(1-\alpha)(\numperm+1)}$, and that $R$ is the position of $M_{\numperm + 1}$ after sorting $(M_b)_{b=1}^{\numperm+1}$ increasingly with ties broken at random. 

We note $R \leq b_{\alpha}$ is a necessary condition to accept the null hypothesis, and we show that it implies that $M_{\numperm+1} \defeq \mathrm{MMD}(\hat{\mathbb{X}}_m,\hat{\mathbb{Y}}_n) \leq M_{(b_{\alpha})}$. To prove this, assume the contrapositive: if $M_{\numperm+1} > M_{(b_{\alpha})}$, then forcibly the position $R$ of $M_{\numperm+1}$ within an increasing ordering of $(M_{b})_{b=1}^{\numperm}$ is greater than $b_{\alpha}$. Hence,
\begin{talign}
    \mathrm{Pr}
    [\deltactt(\mathbb{X}_m,\mathbb{Y}_n) = 0] \leq \mathrm{Pr}[R \leq b_{\alpha}] \leq \mathrm{Pr}[\mathrm{MMD}(\hat{\mathbb{X}}_m,\hat{\mathbb{Y}}_n) \leq M_{(b_{\alpha})}].
\end{talign}
Hence, to prove \cref{thm:uniform_separation} (or \cref{thm:uniform_separation_detailed}) it suffices to show that $\mathrm{Pr}[\mathrm{MMD}(\hat{\mathbb{X}}_m,\hat{\mathbb{Y}}_n) \leq M_{(b_{\alpha})}] \leq \beta$.

\subsection{Recasting the power lower bound into a high-probability threshold upper bound}

We start with the following result that follows a structure similar to  \citet[Lem.~2]{schrab2021mmd}. 

\begin{lemma}[Upper bound on acceptance probability from upper bound on threshold] \label{prop:control_power}
Let $1 \geq \beta > 0$ arbitrary, and define $\tilde{\beta} = \frac{\beta}{1+\frac{\beta}{2}}$. Define the function
\begin{talign} \label{eq:Z_def}
Z(m,n,\beta) %
\defeq 
    \frac{\errorsplit(\P, m/\sblock_m, \tilde{\beta}%
    /6,\ossymb)}{2^{\ossymb} \sqrt{m}} +  \frac{\errorsplit(\Q, n/\sblock_n, \tilde{\beta}%
    /6,\ossymb)}{2^{\ossymb} \sqrt{n}} +c_{\tilde{\beta}%
    /6} \bigparenth{\sqrt{\frac{\kinfnorm}{m}}\!+\!\sqrt{\frac{\kinfnorm}{n}}},
\end{talign}
which is equal to the upper bound in \cref{eq:mmd_diff_p_q} when we make the choice $\delta = \tilde{\beta}/6$.
If $\mathrm{Pr}\left[\mathrm{MMD}(\P,\Q) \geq
Z(m,n,\beta) +
M_{(b_{\alpha})} \right] \geq \frac{1}{1+\frac{\beta}{2}}$
then %
$\mathrm{Pr}[\mathrm{MMD}(\hat{\mathbb{X}}_m,\hat{\mathbb{Y}}_n) \leq %
M_{(b_{\alpha})}] \leq \beta$.
\end{lemma}

\begin{proof}
Define the events 
$\mathfrak{A} := \{ \mathrm{MMD}(\hat{\mathbb{X}}_m,\hat{\mathbb{Y}}_n) \leq %
M_{(b_{\alpha})} \}$,
and 
$\mathfrak{B} := \{ \mathrm{MMD}(\P,\Q) \geq 
Z(m,n,\beta) + %
M_{(b_{\alpha})} \}$.

By assumption, we have $%
\mathrm{Pr}[\mathfrak{B}] \geq 1 - \frac{\tilde{\beta}}{2}$, %
and we want to show $%
\mathrm{Pr}[\mathfrak{A}] \leq \beta$. Note that
\begin{talign} 
\begin{split} \label{eq:prob_bound}
    \mathrm{Pr}[\mathfrak{A}|\mathfrak{B}] &= %
    \mathrm{Pr}\left[\mathrm{MMD}(\hat{\mathbb{X}}_m,\hat{\mathbb{Y}}_n) \leq %
    M_{(b_{\alpha})} \, \big| \, \mathfrak{B} \right] \\ &\leq %
    \mathrm{Pr}\left[ \mathrm{MMD}(\hat{\mathbb{X}}_m,\hat{\mathbb{Y}}_n) \leq \mathrm{MMD}(\P,\Q) - Z(m,n,\beta)
    \, | \, \mathfrak{B} \right] \\
    &\leq \frac{1}{%
    \mathrm{Pr}[\mathfrak{B}]} %
    \mathrm{Pr}\left[ \mathrm{MMD}(\P,\Q) - \mathrm{MMD}(\hat{\mathbb{X}}_m,\hat{\mathbb{Y}}_n) \geq Z(m,n,\beta)
    \right] \\
    &\leq \frac{1}{1-\frac{\tilde{\beta}}{2}} %
    \mathrm{Pr}\left[ \mathrm{MMD}(\P,\Q) - \mathrm{MMD}(\hat{\mathbb{X}}_m,\hat{\mathbb{Y}}_n) \geq Z(m,n,\beta) 
    \right]. 
\end{split}
\end{talign}
Equation \cref{eq:mmd_diff_p_q} in \cref{thm:compression_guarantee} shows that with probability at least $1- \frac{\tilde{\beta}}{2}$,
\begin{talign}
    |\mathrm{MMD}(\P,\Q) - \mathrm{MMD}(\hat{\mathbb{X}}_m,\hat{\mathbb{Y}}_n)| \leq Z(m,n,\beta).
\end{talign}
Thus, the right-hand side of \cref{eq:prob_bound} is upper-bounded by $\frac{\tilde{\beta}}{2}\cdot \frac{1}{1-\frac{\tilde{\beta}}{2}} = \frac{\beta}{2}$,
where we used that
\begin{talign}
    \tilde{\beta} = \frac{\beta}{1+\frac{\beta}{2}} \iff \frac{\tilde{\beta}}{2} = \frac{\frac{\beta}{2}}{1+\frac{\beta}{2}} \iff \frac{\beta}{2} = \frac{\frac{\tilde{\beta}}{2}}{1-\frac{\tilde{\beta}}{2}}.
\end{talign}
We conclude the proof:
\begin{talign}
    \mathrm{Pr}(\mathfrak{A}) &= %
    \mathrm{Pr}(\mathfrak{A}|\mathfrak{B}) %
    \mathrm{Pr}(\mathfrak{B}) \\ &+ %
    \mathrm{Pr}(\mathfrak{A}|\mathfrak{B}^{c}) %
    \mathrm{Pr}(\mathfrak{B}^{c}) \leq \frac{\beta}{2} \cdot 1 + 1 \cdot \frac{\beta}{2} = \beta.
\end{talign}
\end{proof}

\subsection{%
High-probability bound on the threshold} \label{sec:arbitrary_permutations}

Given \cref{prop:control_power}, the remainder of the proof of \cref{thm:uniform_separation} is devoted to checking that $\mathrm{Pr}\left[\mathrm{MMD}(\P,\Q) \geq
Z(m,n,\beta) +
M_{(b_{\alpha})} \right] \geq 1/(1+\frac{\beta}{2})$ holds, which involves getting a high-probability upper-bound on $M_{(b_{\alpha})}$. 
\citet{schrab2021mmd} use an approach based on the Dvoretzky-Kiefer-Wolfowitz theorem, %
which forces them to use a number of permutations $\numperm$ larger than a threshold which is larger than the values used in practice. We employ more precise techniques based on order statistics (pioneered in this setting by \cite{learnedmiller2008aprobabilistic}) which give tight results for any $\numperm$ as long as $\numperm \!\geq\! \alpha\inv\!-\!1$. 
We focus on the case of permutations instead of wild bootstrap, but the arguments could be adapted for the wild boostrap case, which is in fact simpler. %

\begin{definition}[$k$-th order statistic] \label{def:kth_order_stat}
Given $n$ \iid variables $(Y_{k})^n_{k=1}$, and define the variables $(Y_{(k)})^n_{k=1}$ as the result of sorting $(Y_k)^n_{k=1}$ in increasing order.
For any $1 \leq k \leq n$, the variable $Y_{(k)}$ is known as the $k$-th order statistic.
\end{definition}

It is well known that the $k$-th order statistic for $n$ samples of the uniform distribution on $[0,1]$ is distributed according to the beta distribution $\mathrm{Beta}(k,n+1-k)$ \cite[p.63]{gentle2009computational}. The CDF of the distribution $\mathrm{Beta}(k,n+1-k)$ is equal to the regularized incomplete beta function $I_x(k,n+1-k)$, which is defined below.

Given positive $a, b \in \R$ and $x \in [0,1]$, the regularized incomplete beta function is defined as $I_x(a,b) = \frac{B(x;a,b)}{B(a,b)}$, where $B(x;a,b) = \int_{0}^{x} t^{a-1} (1-t)^{b-1} \, dt$ is the incomplete beta function and $B(a,b) = \int_{0}^{1} t^{a-1} (1-t)^{b-1} \, dt$ is the beta function.

\begin{lemma}[Properties of the regularized incomplete beta function] \label{lem:I_x}
The following statements regarding the regularized incomplete beta function and order statistics hold:
\begin{enumerate}[label=(\roman*),leftmargin=*]
\item \label{eq:prop_reg_incomplete_beta} For any integers $m \leq n$ and $x \in [0,1)$, we have that 
    \begin{talign} \label{eq:I_x_int}
        I_{x} (m,n+1-m) = \sum_{j=m}^{n} {n \choose j} x^j (1-x)^{n-j}. %
    \end{talign}
\item For any $m \leq n$ and $x \in [0,1)$, we have that $\frac{\partial^k}{\partial x^k} I_{x} (m,n+1-m) \rvert_{x = 0} = 0$ for any $0 \leq k < m$, and that
    \begin{talign} \label{eq:partial_m_I_x}
        \frac{\partial^m}{\partial x^m} I_{x} (m,n+1-m) \rvert_{x = 0} &= {n \choose m} m! %
        = \frac{n!}{(n-m)!}, \\
        \frac{\partial^{m+1}}{\partial x^{m+1}} I_{x} (m,n+1-m) \rvert_{x = 0} &= - \indicator_{m < n} \frac{n! m}{(n-m-1)!} \label{eq:partial_mp1_I_x}
    \end{talign}
\item \label{eq:I_x_near_0} For any $x \in [0,1)$, there exists $z \in [0,x)$ such that $I_{x} (m,n+1-m) = {n \choose m} x^m - m {n \choose m+1} z^{m+1}$.
\item Let $Y_{(m)}$ be the $m$-th order statistic (\cref{def:kth_order_stat}) for $n$ samples of the uniform distribution on $[0,1]$. For any $x \in [0,1]$, we have that
    \begin{talign}
        {n \choose m-1} x^{n+1-m} - (n+1-m) {n \choose m-2} x^{n+2-m} \indicator_{m > 1} \leq \mathrm{Pr}[Y_{(m)} > 1-x] \leq {n \choose m-1} x^{n+1-m}.
    \end{talign}
\end{enumerate}
\end{lemma}
\begin{proof} We prove each part separately.
\begin{enumerate}[label=(\roman*),leftmargin=0.75cm]
\item This part follows directly from \citet[Eq.~8.17.5]{NIST:DLMF}.
\item The statement $\frac{\partial^k}{\partial x^k} I_{x} (m,n+1-m) \rvert_{x = 0} = 0$ for any $0 \leq k < m$ holds because by \cref{eq:I_x_int}, $I_{x} (m,n+1-m)$ can be expressed as a polynomial in $x$ where all the terms are of power at least $m$. 
    
To obtain $\frac{\partial^m}{\partial x^m} I_{x} (m,n+1-m) \rvert_{x = 0}$, we multiply by $m!$ the coefficient of $I_{x} (m,n+1-m)$ for the term of degree $m$, which is the term of degree $m$ of the polynomial ${n \choose m} x^m (1-x)^{n-m}$.
    
To obtain $\frac{\partial^{m+1}}{\partial x^{m+1}} I_{x} (m,n+1-m) \rvert_{x = 0}$, we multiply by $(m+1)!$ the coefficient of $I_{x} (m,n+1-m)$ for the term of degree $m+1$, which is the term of degree $m+1$ of the polynomial ${n \choose m} x^m (1-x)^{n-m}$ plus the term of degree $m+1$ of the polynomial ${n \choose m+1} x^{m+1} (1-x)^{n-m-1}$ (if the latter term exists). Thus, when $m+1 \leq n$, $(m+1)!$ times the coefficient of $I_{x} (m,n+1-m)$ for the term of degree $m+1$ reads:
    \begin{talign}
        &- {n \choose m} (n-m) (m+1)! + {n \choose m+1} (m + 1)! \\ &= - \frac{n!}{m! (n-m)!} (n-m) (m+1)! + \frac{n!}{(m+1)! (n-m-1)!} (m + 1)! \\ &= - \frac{n! (m+1)}{(n-m-1)!} + \frac{n!}{(n-m-1)!} \\ &= - \frac{n! m}{(n-m-1)!}.
    \end{talign}
When $m = n$, we obtain 0 instead.
\item By the residual form of Taylor's theorem, we have that for any $y \in [0,1)$,
\begin{talign}
    I_{y} (m,n+1-m) = \frac{1}{m!} \frac{\partial^m}{\partial x^m} I_{x} (m,n+1-m) \rvert_{x = 0} y^m + \frac{1}{(m+1)!} \frac{\partial^{m+1}}{\partial x^{m+1}} I_{x} (m,n+1-m) \rvert_{x = 0} z^{m+1},
\end{talign}
where $z \in [0,y]$. Substituting the expressions from \cref{eq:partial_m_I_x} and \cref{eq:partial_mp1_I_x} into this equation, we obtain that
\begin{talign}
    I_{y} (m,n+1-m) = \frac{n!}{(n-m)! m!} y^m - \frac{n! m}{(m+1)! (n-m-1)!} z^{m+1} \indicator_{m < n} = {n \choose m} y^m - m {n \choose m+1} z^{m+1} \indicator_{m < n}.
\end{talign}
\item \label{prop:kth_order_tail} By \citet[Eq.~8.17.4]{NIST:DLMF}, for any $a,b$ non-negative and $x \in [0,1]$, we have that $I_x(a,b) = 1-I_{1-x}(a,b)$. The variable $Y_{(m)}$ is distributed according to $\mathrm{Beta}(m,n+1-m)$, which means that 
\begin{talign}
    \mathrm{Pr}[Y_{(m)} > 1-x] &= 1 - \mathrm{Pr}[Y_{(m)} \leq 1-x] = 1 - I_{1-x}(m, n+1-m) = I_x(n+1-m,m) \\ &= I_x(n+1-m,n+1-(n+1-m)) \\ &= {n \choose n + 1 - m} x^{n+1-m} - (n+1-m) {n \choose n+2-m} z^{n+2-m} \indicator_{n+1-m < n} \\ &= {n \choose m-1} x^{n+1-m} - (n+1-m) {n \choose m-2} z^{n+2-m} \indicator_{m > 1}
\end{talign}
In the second-to-last equality we plugged the result from \cref{lem:I_x}\cref{eq:I_x_near_0}, replacing $m$ by $n+1-m$. Since $z \in [0,x]$, the result follows.
\end{enumerate}
\end{proof}

Let $F$ be the CDF of the random variable $M_{\sigma} \defeq \mmd(\hat{\mathbb{X}}_m^{\sigma},\hat{\mathbb{Y}}_n^{\sigma})$,%
 i.e. $F(x) = P(M_{\sigma} \leq x)$. We define the random map $\mathfrak{F}$ as
\begin{align}
    \mathfrak{F}(x) = 
    \begin{cases}
    F(x) &\text{if } F \text{ continuous at } x \\
    \mathrm{Unif}(\lim_{y \to x^{-}} F(y),F(x)) &\text{otherwise}
    \end{cases}
\end{align}
Note that by definition, for all $x$ we have that $\mathrm{Pr}(M_{\sigma} < x) \leq \mathfrak{F}(x) \leq \mathrm{Pr}(M_{\sigma} \leq x)$. Also, by construction $\mathfrak{F}(M_{\sigma})$ is distributed uniformly over $[0,1]$.

\begin{lemma}[High probability bound on the threshold from quantile of the CDF $F$] \label{prop:mathfrak_q_bound}
For an arbitrary $\alpha' \in (0,1)$, we define the random variable
\begin{talign} \label{eq:q_def}
    q_{1-\alpha'}(\xs, \ys) \defeq \inf \big\{x \in \R : 1-\alpha' \leq F(x) \big\}.
\end{talign}
Given $\xs,\ys$, we have that with probability at least $1-\frac{\delta}{2}$,
\begin{talign}
    M_{(b_{\alpha})} \leq q_{1-\alpha^*}(\xs, \ys),
\end{talign}
where $\alpha^* = \big(\frac{\delta}{2{\numperm \choose \floor{\alpha(\numperm+1)}}} \big)^{1/\floor{\alpha(\numperm+1)}}$.
\end{lemma}
\begin{proof}
Note that $M_{(b_{\alpha})}$ is the $b_{\alpha}$-th order statistic for the $\numperm$ samples $(M_b)_{b=1}^{\numperm}$. Since the random map $\mathfrak{F}$ is increasing, this implies that $\mathfrak{F}(M_{(b_{\alpha})})$ is the $b_{\alpha}$-th order statistic for the $\numperm$ samples $(\mathfrak{F}(M_b))_{b=1}^{\numperm}$. As stated above, $(\mathfrak{F}(M_b))_{b=1}^{\numperm}$ are uniform i.i.d. variables over $[0,1]$, which means that $\mathfrak{F}(M_{(b_{\alpha})})$ is the $b_{\alpha}$-th order statistic for $\numperm$ samples of the uniform distribution over $[0,1]$. Applying \cref{lem:I_x}\cref{prop:kth_order_tail} with $n = \numperm$, $m = b_{\alpha} = \ceil{(1-\alpha)(\numperm + 1)}$, we obtain that for any $x \in [0,1)$,
\begin{talign}
    \mathrm{Pr}[\mathfrak{F}(M_{(b_{\alpha})}) > 1-x] \leq {\numperm \choose b_{\alpha}-1} x^{\numperm + 1 - b_{\alpha}} = {\numperm \choose \numperm + 1 - b_{\alpha}} x^{\numperm + 1 - b_{\alpha}} = {\numperm \choose \floor{\alpha(\numperm+1)}} x^{\floor{\alpha(\numperm+1)}}.
\end{talign}
Since 
\begin{talign}
{\numperm \choose \floor{\alpha(\numperm+1)}} x^{\floor{\alpha(\numperm+1)}} = \delta/2 \quad \iff \quad x = \big(\frac{\delta}{2{\numperm \choose \floor{\alpha(\numperm+1)}}} \big)^{1/\floor{\alpha(\numperm+1)}},
\end{talign}
we obtain that with probability at least $1-\frac{\delta}{2}$,
\begin{talign}
    \mathfrak{F}(M_{(b_{\alpha})}) < 1-\big(\frac{\delta}{2{\numperm \choose \floor{\alpha(\numperm+1)}}} \big)^{1/\floor{\alpha(\numperm+1)}}.
\end{talign}
For any $\epsilon > 0$, we have that given $\xs,\ys$, $F(x-\epsilon) = \mathrm{Pr}(M_{\sigma} \leq x-\epsilon) \leq \mathrm{Pr}(M_{\sigma} < x) \leq \mathfrak{F}(x)$. Hence, with probability at least $1-\frac{\delta}{2}$, $F(M_{(b_{\alpha})} - \epsilon) \leq 1-\big(\frac{\delta}{2{\numperm \choose \floor{\alpha(\numperm+1)}}} \big)^{1/\floor{\alpha(\numperm+1)}}.$

Hence, if we define $\alpha^* = \big(\frac{\delta}{2{\numperm \choose \floor{\alpha(\numperm+1)}}} \big)^{1/\floor{\alpha(\numperm+1)}}$, we obtain that 
\begin{talign}
    M_{(b_{\alpha})} - \epsilon \leq \inf \big\{x \in \R : 1-\alpha^* \leq F(x) \big\} \defeq q_{1-\alpha^*}(\xs, \ys)
\end{talign}
Since $\epsilon > 0$ is arbitrary, we conclude that $M_{(b_{\alpha})} \leq q_{1-\alpha^*}(\xs, \ys)$ with probability at least $1-\frac{\delta}{2}$.
\end{proof}

Recall that $( \hatxs^{(i)} )_{i=1}^{\sblock_m}$, $( \hatys^{(i)} )_{i=1}^{\sblock_n}$ are the outputs of 
\ktcompress
on inputs $( \xs^{(i)} )_{i=1}^{\sblock_m}$, $( \ys^{(i)} )_{i=1}^{\sblock_n}$.
For $i = 1, \dots, \sblock_m$, $j = 1, \dots, \sblock_n$, denote 
\begin{talign}
\hat{\P}_m = \frac{1}{|\hatxs|} \sum_{x \in \hatxs} \delta_{x},& \quad \hat{\Q}_n = \frac{1}{|\hatys|} \sum_{y \in \hatys} \delta_{y}, \\ \hat{\P}_m^{(i)} = \frac{1}{|\hatxs^{(i)}|} \sum_{x \in \hatxs^{(i)}} \delta_{x},& \quad \hat{\Q}_n^{(j)} = \frac{1}{|\hatys^{(j)}|} \sum_{y \in \hatys^{(j)}} \delta_{y}, \\ \hat{\mathbb{S}}^{(i)}_{m+n} = \hat{\P}_m^{(i)},& \quad \hat{\mathbb{S}}^{(\sblock_m+j)}_{m+n} = \hat{\Q}_n^{(j)}
\end{talign}
We can write
\begin{talign} 
\begin{split} \label{eq:MMD_sq_hatx_haty}
    \mmd^2(\hatxs,\hatys) &= \langle (\hat{\P}_m - \hat{\Q}_n)\kernel, (\hat{\P}_m - \hat{\Q}_n)\kernel \rangle_{\kernel} \\ &= %
    \frac{1}{\sblock_m^2 \sblock_n^2} \sum_{i=1}^{\sblock_m} \sum_{i'=1}^{\sblock_m} \sum_{j=1}^{\sblock_n} \sum_{j'=1}^{\sblock_n} \langle (\hat{\P}_m^{i} - \hat{\Q}_n^{j})\kernel, (\hat{\P}_m^{i'} - \hat{\Q}_n^{j'})\kernel \rangle_{\kernel} \\ &= \frac{1}{\sblock_m^2 \sblock_n^2} \sum_{i=1}^{\sblock_m} \sum_{i'=1}^{\sblock_m} \sum_{j=1}^{\sblock_n} \sum_{j'=1}^{\sblock_n}
    \langle (\hat{\mathbb{S}}_{m+n}^{i} - \hat{\mathbb{S}}_{m+n}^{\sblock_m+j})\kernel, (\hat{\mathbb{S}}_{m+n}^{i'} - \hat{\mathbb{S}}_{m+n}^{\sblock_m+j'})\kernel \rangle_{\kernel} \\ &= \frac{1}{\sblock_m^2 \sblock_n^2} \big( \sum_{i\neq i'\in \{1,\dots,\sblock_m\}} \sum_{j\neq j'\in \{1,\dots,\sblock_n\}}
    \langle (\hat{\mathbb{S}}_{m+n}^{i} - \hat{\mathbb{S}}_{m+n}^{\sblock_m+j})\kernel, (\hat{\mathbb{S}}_{m+n}^{i'} - \hat{\mathbb{S}}_{m+n}^{\sblock_m+j'})\kernel \rangle_{\kernel} \\ &+ \sum_{i=1}^{\sblock_m} \sum_{j\neq j'\in \{1,\dots,\sblock_n\}}
    \langle (\hat{\mathbb{S}}_{m+n}^{i} - \hat{\mathbb{S}}_{m+n}^{\sblock_m+j})\kernel, (\hat{\mathbb{S}}_{m+n}^{i} - \hat{\mathbb{S}}_{m+n}^{\sblock_m+j'})\kernel \rangle_{\kernel} \\ &+ \sum_{j=1}^{\sblock_n} \sum_{i\neq i'\in \{1,\dots,\sblock_m\}}
    \langle (\hat{\mathbb{S}}_{m+n}^{i} - \hat{\mathbb{S}}_{m+n}^{\sblock_m+j})\kernel, (\hat{\mathbb{S}}_{m+n}^{i'} - \hat{\mathbb{S}}_{m+n}^{\sblock_m+j})\kernel \rangle_{\kernel} 
    \\ &+ \sum_{i=1}^{\sblock_m} \sum_{j=1}^{\sblock_n} \langle (\hat{\mathbb{S}}_{m+n}^{i} - \hat{\mathbb{S}}_{m+n}^{\sblock_m+j})\kernel, (\hat{\mathbb{S}}_{m+n}^{i} - \hat{\mathbb{S}}_{m+n}^{\sblock_m+j})\kernel \rangle_{\kernel} \big) 
\end{split}
\end{talign}
By assuming $m \leq n$, let $L := \{l_1, . . . , l_{m} \}$ be an $m$-tuple uniformly drawn without replacement from $\{1, \dots, n\}$. Then, we can write \cref{eq:MMD_sq_hatx_haty} as
\begin{talign} \label{eq:MMD_sq_hatx_haty_rewritten}
    \mmd^2(\hatxs,\hatys) &= \frac{\sblock_n-1}{\sblock_n \sblock_m^2}\mathbb{E}_{L} [ \sum_{i\neq i'\in \{1,\dots,\sblock_m\}} \langle (\hat{\mathbb{S}}_{m+n}^{i} - \hat{\mathbb{S}}_{m+n}^{\sblock_m+l_i})\kernel, (\hat{\mathbb{S}}_{m+n}^{i'} - \hat{\mathbb{S}}_{m+n}^{\sblock_m+l_{i'}})\kernel \rangle_{\kernel}] \\ &+ \frac{1}{\sblock_m^2 \sblock_n^2} \big( \sum_{i=1}^{\sblock_m} \sum_{j\neq j'\in \{1,\dots,\sblock_n\}}
    \langle (\hat{\mathbb{S}}_{m+n}^{i} - \hat{\mathbb{S}}_{m+n}^{\sblock_m+j})\kernel, (\hat{\mathbb{S}}_{m+n}^{i} - \hat{\mathbb{S}}_{m+n}^{\sblock_m+j'})\kernel \rangle_{\kernel} \\ &+ \sum_{j=1}^{\sblock_n} \sum_{i\neq i'\in \{1,\dots,\sblock_m\}}
    \langle (\hat{\mathbb{S}}_{m+n}^{i} - \hat{\mathbb{S}}_{m+n}^{\sblock_m+j})\kernel, (\hat{\mathbb{S}}_{m+n}^{i'} - \hat{\mathbb{S}}_{m+n}^{\sblock_m+j})\kernel \rangle_{\kernel} 
    \\ &+ \sum_{i=1}^{\sblock_m} \sum_{j=1}^{\sblock_n} \langle (\hat{\mathbb{S}}_{m+n}^{i} - \hat{\mathbb{S}}_{m+n}^{\sblock_m+j})\kernel, (\hat{\mathbb{S}}_{m+n}^{i} - \hat{\mathbb{S}}_{m+n}^{\sblock_m+j})\kernel \rangle_{\kernel} \big).
\end{talign}
This holds because for any $i \neq i' \in \{1,\dots,\sblock_m\}$,
\begin{talign}
    \mathbb{E}_{L} [\langle (\hat{\P}_m^{(i)} - \hat{\Q}_n^{(\sblock_m+l_i)})\kernel, (\hat{\P}_m^{(i')} - \hat{\Q}_n^{(\sblock_m+l_{i'})})\kernel \rangle_{\kernel}] = \frac{1}{\sblock_n (\sblock_n-1)}
    \sum_{j\neq j' \in \{1,\dots,\sblock_n\}} \langle (\hat{\mathbb{S}}_{m+n}^{(i)} - \hat{\mathbb{S}}_{m+n}^{(\sblock_m+j)})\kernel, (\hat{\mathbb{S}}_{m+n}^{(i')} - \hat{\mathbb{S}}_{m+n}^{(\sblock_m+j')})\kernel \rangle_{\kernel}
\end{talign}

Recall also that $\mathbb{U}_{m+n} = (U_i)_{i=1}^{m+n}$, with $U_i = X_i$ for $i=1,\dots,m$ and $U_{m+j} = Y_j$ for $j=1,\dots,n$. Equivalently, we can write that $\mathbb{U}_{m+n} = (\mathbb{U}_{m+n}^{(i)})_{i=1}^{\sblock}$, with $\mathbb{U}_{m+n}^{(i)} = \xs^{(i)}$ and $\mathbb{U}_{m+n}^{(\sblock_m+j)} = \ys^{(j)}$ for $i=1,\dots,\sblock_m$, $j=1,\dots,\sblock_n$. Analogously, we define $\hat{\mathbb{U}}_{m+n} = (\hat{\mathbb{U}}_{m+n}^{(i)})_{i=1}^{\sblock}$, with $\hat{\mathbb{U}}_{m+n}^{(i)} = \hatxs^{(i)}$ and $\hat{\mathbb{U}}_{m+n}^{(\sblock_m+j)} = \hatys^{(j)}$ for $i=1,\dots,\sblock_m$, $j=1,\dots,\sblock_n$.

Given a permutation $\sigma : \{1,\dots,\sblock\} \to \{1,\dots,\sblock\}$, we write $\mathbb{U}_{m+n}^{\sigma} = (\mathbb{U}_{m+n}^{(\sigma(i))})_{i=1}^{\sblock}$, $\hat{\mathbb{U}}_{m+n}^{\sigma} = (\hat{\mathbb{U}}_{m+n}^{(\sigma(i))})_{i=1}^{\sblock}$, and $\xs^{\sigma} = (\mathbb{U}_{m+n}^{(\sigma(i))})_{i=1}^{\sblock_m}$, $\hatxs^{\sigma} = (\hat{\mathbb{U}}_{m+n}^{(\sigma(i))})_{i=1}^{\sblock_m}$, $\ys^{\sigma} = (\mathbb{U}_{m+n}^{(\sigma(i))})_{i=\sblock_m+1}^{\sblock}$, $\hatys^{\sigma} = (\hat{\mathbb{U}}_{m+n}^{(\sigma(i))})_{i=\sblock_m+1}^{\sblock}$. Analogously to \cref{eq:MMD_sq_hatx_haty_rewritten}, we can write 
\begin{talign} 
\begin{split} \label{eq:mmd_perm_decomposition}
    \mmd^2(\hatxs^{\sigma},\hatys^{\sigma}) &= %
    \frac{1}{\sblock_m^2 \sblock_n^2} \sum_{i=1}^{\sblock_m} \sum_{i'=1}^{\sblock_m} \sum_{j=1}^{\sblock_n} \sum_{j'=1}^{\sblock_n}
    \langle (\hat{\mathbb{S}}_{m+n}^{(\sigma(i))} - \hat{\mathbb{S}}_{m+n}^{(\sigma(\sblock_m+j))})\kernel, (\hat{\mathbb{S}}_{m+n}^{(\sigma(i'))} - \hat{\mathbb{S}}_{m+n}^{(\sigma(\sblock_m+j'))})\kernel \rangle_{\kernel} \\ &= \frac{(\sblock_m-1)(\sblock_n-1)}{\sblock_m \sblock_n} \mathbb{E}_{L} [ \mathcal{M}^{\sigma,L} ] \\ &+ \frac{1}{\sblock_m^2 \sblock_n^2} \big( \sum_{i=1}^{\sblock_m} \sum_{j\neq j'\in \{1,\dots,\sblock_n\}}
    \langle (\hat{\mathbb{S}}_{m+n}^{i} - \hat{\mathbb{S}}_{m+n}^{\sigma(\sblock_m+j)})\kernel, (\hat{\mathbb{S}}_{m+n}^{i} - \hat{\mathbb{S}}_{m+n}^{\sigma(\sblock_m+j')})\kernel \rangle_{\kernel} \\ &+ \sum_{j=1}^{\sblock_n} \sum_{i\neq i'\in \{1,\dots,\sblock_m\}}
    \langle (\hat{\mathbb{S}}_{m+n}^{\sigma(i)} - \hat{\mathbb{S}}_{m+n}^{\sigma(\sblock_m+j)})\kernel, (\hat{\mathbb{S}}_{m+n}^{\sigma(i')} - \hat{\mathbb{S}}_{m+n}^{\sigma(\sblock_m+j)})\kernel \rangle_{\kernel} 
    \\ &+ \sum_{i=1}^{\sblock_m} \sum_{j=1}^{\sblock_n} \langle (\hat{\mathbb{S}}_{m+n}^{\sigma(i)} - \hat{\mathbb{S}}_{m+n}^{\sigma(\sblock_m+j)})\kernel, (\hat{\mathbb{S}}_{m+n}^{\sigma(i)} - \hat{\mathbb{S}}_{m+n}^{\sigma(\sblock_m+j)})\kernel \rangle_{\kernel} \big),
\end{split}
\end{talign}
where we use the short-hand
\begin{talign} \label{eq:M_sigma_L_def}
    \mathcal{M}^{\sigma,L} \defeq \frac{1}{\sblock_m(\sblock_m-1)} \sum_{i\neq i'\in \{1,\dots,\sblock_m\}} \langle (\hat{\mathbb{S}}_{m+n}^{(\sigma(i))} - \hat{\mathbb{S}}_{m+n}^{(\sigma(\sblock_m+l_i))})\kernel, (\hat{\mathbb{S}}_{m+n}^{(\sigma(i'))} - \hat{\mathbb{S}}_{m+n}^{(\sigma(\sblock_m+l_{i'}))})\kernel \rangle_{\kernel}.
\end{talign}

The following proposition, whose proof is deferred to \cref{subsub:proof_random_perm}, provides a tail upper-bound on the random variable $\mmd^2(\hatxs^{\sigma},\hatys^{\sigma})$.
\begin{proposition}[Tail bound on $\mmd^2(\hatxs^{\sigma},\hatys^{\sigma})$ conditioned on $\hatxs$, 
$\hatys$] \label{prop:random_perm_mmd_bound}
Let $\sigma$ be a uniformly random permutation over $\{1,\dots,\sblock\}$. Let $\delta, \delta'' \in (0,1)$, and $\delta \in (0,e^{-1})$. %
There is an event $\mathcal{A}$ of probability at least $1-\delta$ concerning the draw of $\hatxs$, 
$\hatys$, such that conditioned on $\mathcal{A}$, 
with probability at least $1-\delta'-\delta''$ on the draw of $\sigma$,
\begin{talign} 
\begin{split}\label{eq:sigma_prop}
    \mmd^2(\hatxs^{\sigma},\hatys^{\sigma}) &\leq 
    \frac{2}{\sblock_m} (\log(\frac{2}{\delta''}) +1) \mmd^2(\P,\Q) + %
    \frac{c' \log(1/\delta') + 2}{\sblock_m} \big( \mmd(\P,\Q) W(m,n,\delta/(5\sblock)) + W(m,n,\delta/(5\sblock))^2 \big)
\end{split}
\end{talign}
where $c'$ is a universal constant and
\begin{talign} \label{eq:W_def}
    W(m,n,\delta%
    ) = 2\big(\frac{\sqrt{\sblock_m} \errorsplit(\P, m/\sblock_m, \delta%
    ,\ossymb)}{2^{\ossymb} \sqrt{m}} \!+ \!\frac{\sqrt{\sblock_n} \errorsplit(\Q, n/\sblock_n, \delta%
    ,\ossymb)}{2^{\ossymb} \sqrt{n}} \!+ \!c_{\delta%
    } \big(\sqrt{\frac{\sblock_m \kinfnorm}{m}} +\sqrt{\frac{\sblock_n \kinfnorm}{n}} \big) \big).
\end{talign}
Here, the \ktcompress error inflation factors $\errorsplit$ and the factor $c_{\delta}$ are defined as in \cref{thm:compression_guarantee}.
\end{proposition}

Combining \cref{prop:mathfrak_q_bound} with the tail bound from \cref{prop:random_perm_mmd_bound} yields the following high-probability upper bound on the threshold $M_{(b_{\alpha})}$.
\begin{corollary}[High probability bound on the threshold] \label{prop:bound_q_arbitrary_B}
    Assume %
    that $b_{\alpha} \defeq \ceil{(\numperm+1)(1-\alpha)} \leq \numperm$, and that \ktcompress calls are run with value $\delta^{*}/(5\sblock_m)$ and $\delta^{*}/(5\sblock_n)$ respectively.
    Then, with probability at least $1-\frac{\delta}{2}$,
    \begin{talign} 
    \begin{split} \label{eq:hat_Z_def}
    M_{(b_{\alpha})}
    &\leq \hat{Z}(m,n,\alpha,\delta) 
    \\ &\defeq 
    \sqrt{\frac{2}{\sblock_m} (\log(\frac{2}{\delta^*}) + 1)} \mmd(\P,\Q) + \sqrt{\frac{1}{\sblock_m}(2+c' \log(1/\delta^*))} (\sqrt{\mmd(\P,\Q) W(m,n,\delta/(20\sblock))} \\ &+ W(m,n,\delta/(20\sblock))).
    \end{split}
    \end{talign}
    where 
    \begin{talign}
    \delta^{*} \defeq %
    (\frac{\delta}{2})^{1/k_{\alpha}} \frac{\alpha}{4e}, \quad k_{\alpha} \defeq \floor{\alpha(\numperm+1)},
    \end{talign}
\end{corollary}
\begin{proof}
By \cref{prop:mathfrak_q_bound}, with probability at least $1-\frac{\delta}{4}$, we have that
    \begin{talign} \label{eq:mathfrak_F_q_infty}
    M_{(b_{\alpha})} \leq q_{1-\alpha^*}(\xs, \ys),
    \end{talign}
    where $\alpha^* = \big(\frac{\delta}{4{\numperm \choose k_{\alpha}}} \big)^{1/k_{\alpha}}$. Conditioned on the event $\mathcal{A}$ with $\delta \gets \delta/4$ (i.e. $\mathrm{Pr}(\mathcal{A}) \geq 1-\delta/4$), and setting $\delta' = \delta'' = \alpha^*/2$ we have that with probability at least $1-\alpha^{*}$ on the choice of $\sigma$,
    \begin{talign}
        \mmd(\hatxs^{\sigma},\hatys^{\sigma}) &\leq \big(\frac{2}{\sblock_m} (\log(\frac{2 \cdot 2}{\alpha^*}) + 1) \mmd^2(\P,\Q) \\ &\quad + \frac{1}{\sblock_m} \big( \mmd(\P,\Q) W(m,n,\delta/(20\sblock)) + W(m,n,\delta/(20\sblock))^2 \big) (2+c' \log(2/\alpha^*)) \big)^{1/2} \\ &\leq \sqrt{\frac{2}{\sblock_m} (\log(\frac{2 \cdot 2}{\alpha^*}) + 1)} \mmd(\P,\Q) \\ &\quad + \sqrt{\frac{1}{\sblock_m}(2+c' \log(2/\alpha^*))} (\sqrt{\mmd(\P,\Q) W(m,n,\delta/(20\sblock))} + W(m,n,\delta/(20\sblock))).
    \end{talign}
    An application of \cref{prop:random_perm_mmd_bound} yields, conditioned on the event $\mathcal{A}$, 
    \begin{talign} \label{eq:q_alphastar}
        q_{1-\alpha^*}(\xs, \ys) &\leq \sqrt{\frac{2}{\sblock_m} (\log(\frac{2 \cdot 2}{\alpha^*}) + 1)} \mmd(\P,\Q) \\ &+ \sqrt{\frac{1}{\sblock_m}(2+c' \log(2/\alpha^*))} (\sqrt{\mmd(\P,\Q) W(m,n,
        \delta/(20\sblock))} + W(m,n,
        \delta/(20\sblock))).
    \end{talign}
    Since the probability of $\mathcal{A}$ is at least $1-\delta/4$, we obtain that with probability at least $1-\delta/2$,
    \begin{talign}
        M_{(b_{\alpha})} &\leq \sqrt{\frac{2}{\sblock_m} (\log(\frac{2 \cdot 2}{\alpha^*}) + 1)} \mmd(\P,\Q) \\ &+ \sqrt{\frac{1}{\sblock_m}(2+c' \log(2/\alpha^*))} (\sqrt{\mmd(\P,\Q) W(m,n,%
        \delta/(20\sblock))} + W(m,n,%
        \delta/(20\sblock))).
    \end{talign}
    Using the fact $(\frac{n}{k})^{k} \!\leq\! {n \choose k} \!\leq\! (\frac{e n}{k})^{k}$, we obtain that $ \log(\frac{\numperm}{k_{\alpha}}) \!\leq\! \log({\numperm \choose k_{\alpha}}^{\frac{1}{k_{\alpha}}}) \!\leq\! \log(\frac{\numperm}{k_{\alpha}})\!+\!1$. Furthermore, $\log(\frac{\numperm}{k_{\alpha}}) \!\leq\! \log(\frac{2\numperm}{\alpha(\numperm+1)}) \!\leq\! \log(\frac{2}{\alpha})$ since $k_{\alpha} \!\geq\! 1$. Consequently, $\log({\numperm \choose k_{\alpha}}^{\frac{1}{k_{\alpha}}}) \!\leq\! \log\frac{2e}{\alpha}$, so that the separation rate in \cref{thm:uniform_separation} is independent of $\numperm$ (up to the condition $\numperm\!+\!1 \!\geq\! \alpha\inv$). Equivalently, ${\numperm \choose k_{\alpha}}^{\frac{1}{k_{\alpha}}} \leq \frac{2e}{\alpha}$, and $\alpha^{*} \geq (\frac{\delta}{4})^{1/k_{\alpha}} \frac{\alpha}{2e}$, and $\delta^{*} \geq (\frac{\delta}{4})^{1/k_{\alpha}} \frac{\alpha}{4e}$. Plugging this into \cref{eq:q_alphastar} concludes the proof.
\end{proof}

\subsubsection{\pcref{prop:random_perm_mmd_bound}} \label{subsub:proof_random_perm}

We will first show the following lemma, which gives a high-probability upper bound on the expectation $\E_L[\mathcal{M}^{\sigma,L}]$, where $\mathcal{M}^{\sigma,L}$ is defined in equation \cref{eq:M_sigma_L_def}, and on the rest of the terms that appear in the right-hand side of \cref{eq:mmd_perm_decomposition}.
\begin{lemma}[Bounding the right-hand side of \cref{eq:mmd_perm_decomposition}] \label{lem:mathcal_M_bound}
Let $\sigma$ be a uniformly random permutation over $\{1,\dots,\sblock\}$, and $L := \{l_1, . . ., l_{m} \}$ a uniformly random $m$-tuple of elements from $\{1, \dots, n\}$ without replacement. Let $\delta, \delta'' \in (0,1)$, and $\delta' \in (0,e^{-1})$. %
Conditioned on the event $\mathcal{A}$ defined in \cref{lem:ktcompress_simultaneous}, which has probability at least $1-\delta$, we have that with probability at least $1-\delta'-\delta''$ on the choice of $\sigma$,
\begin{talign} 
\begin{split}\label{eq:sigma_prop}
    \E_L[\mathcal{M}^{\sigma,L}] &\leq 
    \frac{2}{\sblock_m - 1} \log(\frac{2}{\delta''}) \mmd^2(\P,\Q) \\ &+ \frac{c' \log(1/\delta')}{\sqrt{\sblock_m(\sblock_m-1)}} \big( \mmd(\P,\Q) W(m,n,\delta/(5\sblock)) + W(m,n,\delta/(5\sblock))^2 \big)
\end{split}
\end{talign}
and 
\begin{talign} 
\begin{split} \label{eq:other_terms_prop}
    &\frac{1}{\sblock_m^2 \sblock_n^2} \big( \sum_{i=1}^{\sblock_m} \sum_{j\neq j'\in \{1,\dots,\sblock_n\}}
    \langle (\hat{\mathbb{S}}_{m+n}^{i} - \hat{\mathbb{S}}_{m+n}^{\sigma(\sblock_m+j)})\kernel, (\hat{\mathbb{S}}_{m+n}^{i} - \hat{\mathbb{S}}_{m+n}^{\sigma(\sblock_m+j')})\kernel \rangle_{\kernel} \\ &+ \sum_{j=1}^{\sblock_n} \sum_{i\neq i'\in \{1,\dots,\sblock_m\}}
    \langle (\hat{\mathbb{S}}_{m+n}^{\sigma(i)} - \hat{\mathbb{S}}_{m+n}^{\sigma(\sblock_m+j)})\kernel, (\hat{\mathbb{S}}_{m+n}^{\sigma(i')} - \hat{\mathbb{S}}_{m+n}^{\sigma(\sblock_m+j)})\kernel \rangle_{\kernel} 
    \\ &+ \sum_{i=1}^{\sblock_m} \sum_{j=1}^{\sblock_n} \langle (\hat{\mathbb{S}}_{m+n}^{\sigma(i)} - \hat{\mathbb{S}}_{m+n}^{\sigma(\sblock_m+j)})\kernel, (\hat{\mathbb{S}}_{m+n}^{\sigma(i)} - \hat{\mathbb{S}}_{m+n}^{\sigma(\sblock_m+j)})\kernel \rangle_{\kernel} \big) \\ &\leq %
    \frac{\sblock - 1}{\sblock_m \sblock_n}
    (\mmd(\P,\Q) W(m,n,\delta/(5\sblock)) + W(m,n,\delta/(5\sblock))^2 + \mmd^2(\P,\Q))
\end{split}
\end{talign}
simultaneously, where $c'$ is a universal constant and $W(m,n,\delta)$ is as defined in \cref{eq:W_def}.
\end{lemma}
\begin{proof}
Given a vector $\epsilon = (\epsilon_i)_{i=1}^{\sblock} \in \{\pm1\}^\sblock$, an $m$-tuple $L := \{l_1, . . . , l_{m} \}$ of elements from $\{1, \dots, n\}$ without replacement, and a permutation $\sigma$ on $\{1, \dots, \sblock\}$, define the permutation $\sigma^{\epsilon}$ on $\{1, \dots, \sblock \}$ as
\begin{talign}
\begin{cases}
    \sigma^{\epsilon,L}(i) = \epsilon_i \sigma(i) + (1-\epsilon_i) \sigma(\sblock_m + l_i) &\text{for } i \in \{1,\dots,\sblock_m\}, \\
    \sigma^{\epsilon,L}(\sblock_m + l_i) = (1-\epsilon_i) \sigma(i) + \epsilon_i \sigma(\sblock_m + l_i) &\text{for } i \in \{1,\dots,\sblock_m\}, \\
    \sigma^{\epsilon,L}(j) = \sigma(j) &\text{otherwise}.
\end{cases}
\end{talign}
Using these objects, we have that
\begin{talign} 
    \mathcal{M}^{\sigma^{\epsilon,L},L} &\defeq \frac{1}{\sblock_m(\sblock_m-1)} \sum_{i\neq i'\in \{1,\dots,\sblock_m\}} \langle (\hat{\mathbb{S}}_{m+n}^{\sigma_{\epsilon,L}(i)} - \hat{\mathbb{S}}_{m+n}^{\sigma_{\epsilon,L}(\sblock_m+l_i)})\kernel, (\hat{\mathbb{S}}_{m+n}^{\sigma_{\epsilon,L}(i')} - \hat{\mathbb{S}}_{m+n}^{\sigma_{\epsilon,L}(\sblock_m+l_{i'})})\kernel \rangle_{\kernel}
    \\ &= \frac{1}{\sblock_m(\sblock_m-1)} \sum_{i\neq i'\in \{1,\dots,\sblock_m\}} \epsilon_i \epsilon_{i'} \langle (\hat{\mathbb{S}}_{m+n}^{\sigma(i)} - \hat{\mathbb{S}}_{m+n}^{\sigma(\sblock_m+l_i)})\kernel, (\hat{\mathbb{S}}_{m+n}^{\sigma(i')} - \hat{\mathbb{S}}_{m+n}^{\sigma(\sblock_m+l_{i'})})\kernel \rangle_{\kernel} 
    \label{eq:mmd_sigma_epsilon}
\end{talign}

Note that given a fixed $m$-tuple $L$, if $\sigma$ is distributed uniformly over the permutations of $\{1, \dots, \sblock \}$, %
and $\epsilon$ contains i.i.d. Rademacher variables, then $\sigma^{\epsilon,L}$ is distributed uniformly over the permutations of $\{1, \dots, \sblock \}$ as well.
Hence, $\mathcal{M}^{\sigma,L}$ %
has the same distribution as %
$\mathcal{M}^{\sigma^{\epsilon,L},L}$, conditioned on $\hatxs$,$\hatys$ and $L$ (and consequently conditioned on $\xs$ and $\ys$). 
Given $\sigma$, $L$, define the function $\rho_{\sigma,L} : \{ 1, \dots, \sblock_m \} \to \{ -1, 0, 1 \}$ as
\begin{talign}
    \rho_{\sigma,L}(i) = 
    \begin{cases}
    1 &\text{if } \sigma(i) \in \{1, \dots, \sblock_m \} \text{ and } \sigma(\sblock_m+l_i) \in \{\sblock_m + 1, \dots, \sblock \} \\ 
    -1 &\text{if } \sigma(i) \in \{\sblock_m + 1, \dots, \sblock \} \text{ and } \sigma(\sblock_m + l_i) \in \{1, \dots, \sblock_m \} \\
    0 &\text{if } \sigma(i), \sigma(\sblock_m + l_i) \in \{1, \dots, \sblock_m \} \text{ or } \sigma(i), \sigma(\sblock_m + l_i) \in \{\sblock_m + 1, \dots, \sblock \}
    \end{cases}
\end{talign}
We can rewrite \cref{eq:mmd_sigma_epsilon} as
\begin{talign}
    \mathcal{M}^{\sigma^{\epsilon,L},L} &= \frac{1}{\sblock_m(\sblock_m-1)} \big( \sum_{i\neq i'\in \rho_{\sigma,L}^{-1}(\{-1,1\})} \tilde{\epsilon}_i \tilde{\epsilon}_{i'} \langle (\hat{\P}_m^{\tilde{\sigma}(i)} - \hat{\Q}_n^{\tilde{\sigma}(\sblock_m+l_i)-\sblock_m})\kernel, (\hat{\P}_m^{\tilde{\sigma}(i')} - \hat{\Q}_n^{\tilde{\sigma}(\sblock_m+l_{i'}) -\sblock_m})\kernel \rangle_{\kernel} \\ &+ %
    \sum_{i\neq i'\in \rho_{\sigma,L}^{-1}(\{0\})} \tilde{\epsilon}_i \tilde{\epsilon}_{i'} \langle (\hat{\mathbb{S}}_{m+n}^{\tilde{\sigma}(i)} - \hat{\mathbb{S}}_{m+n}^{\tilde{\sigma}(\sblock_m+l_i)})\kernel, (\hat{\mathbb{S}}_{m+n}^{\tilde{\sigma}(i')} - \hat{\mathbb{S}}_{m+n}^{\tilde{\sigma}(\sblock_m+l_{i'})})\kernel \rangle_{\kernel} \\ &+ 
    2 \sum_{i \in \rho_{\sigma,L}^{-1}(\{-1,1\}),i'\in \rho_{\sigma,L}^{-1}(\{0\})} \tilde{\epsilon}_i \tilde{\epsilon}_{i'} \langle (\hat{\P}_m^{\tilde{\sigma}(i)} - \hat{\Q}_n^{\tilde{\sigma}(\sblock_m+l_{i})-\sblock_m})\kernel, (\hat{\mathbb{S}}_{m+n}^{\tilde{\sigma}(i')} - \hat{\mathbb{S}}_{m+n}^{\tilde{\sigma}(\sblock_m+l_{i'})})\kernel \rangle_{\kernel} \big)
\end{talign}
where we introduced $\tilde{\epsilon} = (\tilde{\epsilon}_i)_{i=1}^{\sblock}$ and the permutation $\tilde{\sigma}$, defined as:
\begin{align}
    \tilde{\epsilon}_i &= 
    \begin{cases}
    - \epsilon_i &\text{if } \rho_{\sigma}(i) = -1, \\
    \epsilon_i &\text{otherwise}
    \end{cases} \\
    \tilde{\sigma}(i) &= 
    \begin{cases}
    \sigma(\sblock_m+l_i) &\text{if } i \in \{1,\dots,\sblock\} \text{ and } \rho_{\sigma}(i) = -1 \\ 
    \sigma(j) &\text{if } i = \sblock_m + l_j \text{ and } \rho_{\sigma}(j) = -1 \\
    \sigma(i) &\text{otherwise}.
    \end{cases}
\end{align}
Note that conditioned on $\sigma$, $\tilde{\epsilon}$ is still a vector of i.i.d. Rademacher variables.
Now, we will apply \cref{lem:hanson_wright} on $\tilde{\epsilon} = (\tilde{\epsilon}_i)_{i=1}^{\sblock} \in \R^{\sblock}$ and the matrix $A = (A_{i,i'})_{i,i'=1}^{\sblock} \in \R^{\sblock\times\sblock}$ defined as
\small
\begin{talign} \label{eq:hanson_wright_matrix}
    A_{i,i'} = 
    \begin{cases}
    0 &\text{if } i = i' \\
    \frac{1}{\sblock_m(\sblock_m-1)} (\langle (\hat{\P}_{m}^{\tilde{\sigma}(i)} \!-\! \hat{\Q}_n^{\tilde{\sigma}(\sblock_m + l_i)-\sblock_m})\kernel, (\hat{\P}_{m}^{\tilde{\sigma}(i')} \!-\! \hat{\Q}_n^{\tilde{\sigma}(\sblock_m + l_{i'})-\sblock_m})\kernel \rangle_{\kernel} \!-\!
    \mmd^2(\P,\Q)) &\text{if } i \neq i' \in \rho_{\sigma}^{-1}(\{\pm1\}), \\
    \frac{1}{\sblock_m(\sblock_m-1)} \langle (\hat{\P}_{m}^{(\tilde{\sigma}(i))} - \hat{\Q}_n^{(\tilde{\sigma}(\sblock_m+i)-\sblock_m)})\kernel, (\hat{\mathbb{S}}_{m+n}^{(\tilde{\sigma}(i'))} - \hat{\mathbb{S}}_{m+n}^{(\tilde{\sigma}(\sblock_m+l_{i'}))})\kernel \rangle_{\kernel} &\text{if } i \in \rho_{\sigma}^{-1}(\{\pm1\}), i' \in \rho_{\sigma}^{-1}(\{0\}), \\
    \frac{1}{\sblock_m(\sblock_m-1)} \langle (\hat{\mathbb{S}}_{m+n}^{\tilde{\sigma}(i)} - \hat{\mathbb{S}}_{m+n}^{\tilde{\sigma}(\sblock_m+l_i)})\kernel, (\hat{\P}_{m}^{\tilde{\sigma}(i')} - \hat{\Q}_n^{\tilde{\sigma}(\sblock_m+l_{i'})-\sblock_m})\kernel \rangle_{\kernel} &\text{if } i \in \rho_{\sigma}^{-1}(\{0\}), i' \in \rho_{\sigma}^{-1}(\{\pm1\}), \\
    \frac{1}{\sblock_m(\sblock_m-1)} \langle (\hat{\mathbb{S}}_{m+n}^{\tilde{\sigma}(i)} - \hat{\mathbb{S}}_{m+n}^{\tilde{\sigma}(\sblock_m+l_i)})\kernel, (\hat{\mathbb{S}}_{m+n}^{\tilde{\sigma}(l_{i'})} - \hat{\mathbb{S}}_{m+n}^{\tilde{\sigma}(\sblock_m+l_{i'})})\kernel \rangle_{\kernel} &\text{if } i \neq i' \in \rho_{\sigma}^{-1}(\{0\}).
    \end{cases}
\end{talign}
\normalsize
We develop the expressions that appear in \cref{lem:hanson_wright}. First, note that $\E[\tilde{\epsilon}^{\top} A \tilde{\epsilon}] = \mathrm{Tr}[A] = 0$. Also,
\begin{talign} \label{eq:A_quadratic_form}
    \tilde{\epsilon}^{\top} A \tilde{\epsilon} &= \frac{1}{\sblock_m(\sblock_m - 1)} (\sum_{i \neq i' \in \rho_{\sigma}^{-1}(\{-1,1\})} \tilde{\epsilon}_i \tilde{\epsilon}_{i'} (\langle (\hat{\P}_{m}^{\tilde{\sigma}(i)} - \hat{\Q}_n^{\tilde{\sigma}(\sblock_m+l_{i})-\sblock_m})\kernel, (\hat{\P}_{m}^{\tilde{\sigma}(i')} - \hat{\Q}_n^{\tilde{\sigma}(\sblock_m+l_{i'})-\sblock_m})\kernel \rangle_{\kernel} - \mmd^2(\P,\Q)) \\ &+ 2 \sum_{i \in \rho_{\sigma}^{-1}(\{-1,1\}), i' \in \rho_{\sigma}^{-1}(\{0\})} \tilde{\epsilon}_i \tilde{\epsilon}_{i'} \langle (\hat{\P}_{m}^{\tilde{\sigma}(i)} - \hat{\Q}_n^{\tilde{\sigma}(\sblock_m+i)-\sblock_m})\kernel, (\hat{\mathbb{S}}_{m+n}^{\tilde{\sigma}(i')} - \hat{\mathbb{S}}_{m+n}^{\tilde{\sigma}(\sblock_m+l_{i'})})\kernel \rangle_{\kernel} \\ &+ \sum_{i \neq i' \in \rho_{\sigma}^{-1}(\{0\})} \tilde{\epsilon}_i \tilde{\epsilon}_{i'} \langle (\hat{\mathbb{S}}_{m+n}^{\tilde{\sigma}(i)} - \hat{\mathbb{S}}_{m+n}^{\tilde{\sigma}(\sblock_m+l_i)})\kernel, (\hat{\mathbb{S}}_{m+n}^{\tilde{\sigma}(i')} - \hat{\mathbb{S}}_{m+n}^{\tilde{\sigma}(\sblock+i')})\kernel \rangle_{\kernel}) \\ &=  %
    \mathcal{M}^{\sigma^{\epsilon,L},L} - %
    (\frac{1}{\sblock_m(\sblock_m - 1)} \sum_{i \neq i' \in \rho_{\sigma}^{-1}(\{-1,1\})} \tilde{\epsilon}_i \tilde{\epsilon}_{i'}) \mmd^2(\P,\Q) = \mathcal{M}^{\sigma^{\epsilon,L},L} - c_{\tilde{\epsilon}} \mmd^2(\P,\Q),
\end{talign}
where we defined $c_{\tilde{\epsilon}} \defeq \frac{1}{\sblock_m(\sblock_m - 1)} \sum_{i \neq i' \in \rho_{\sigma}^{-1}(\{-1,1\})} \tilde{\epsilon}_i \tilde{\epsilon}_{i'}$, and
\begin{talign}
    \|A\|_{\text{op}}^2 &\leq \|A\|_{\text{F}}^2 \\ &= \frac{1}{\sblock_m^2(\sblock_m - 1)^2} %
    \big( \sum_{i \neq i' \in \rho_{\sigma}^{-1}(\{-1,1\})} (\langle (\hat{\P}_{m}^{\tilde{\sigma}(i)} - \hat{\Q}_n^{\tilde{\sigma}(\sblock_m+l_i)-\sblock})\kernel, (\hat{\P}_{m}^{\tilde{\sigma}(i')} - \hat{\Q}_n^{\tilde{\sigma}(\sblock_m+l_{i'})-\sblock_m})\kernel \rangle_{\kernel} - \mmd^2(\P,\Q))^2 \label{eq:operator_norm_A} \\ &+ 2 \sum_{i \in \rho_{\sigma}^{-1}(\{-1,1\}), i' \in \rho_{\sigma}^{-1}(\{0\})} \langle (\hat{\P}_{m}^{\tilde{\sigma}(i)} - \hat{\Q}_n^{\tilde{\sigma}(\sblock_m+l_i)-\sblock_m})\kernel, (\hat{\mathbb{S}}_{m+n}^{\tilde{\sigma}(i')} - \hat{\mathbb{S}}_{m+n}^{\tilde{\sigma}(\sblock_m+l_{i'})})\kernel \rangle_{\kernel}^2 \label{eq:operator_norm_A2} \\ &+ \sum_{i \neq i' \in \rho_{\sigma}^{-1}(\{0\})} \langle (\hat{\mathbb{S}}_{m+n}^{\tilde{\sigma}(i)} - \hat{\mathbb{S}}_{m+n}^{\tilde{\sigma}(\sblock_m+l_i)})\kernel, (\hat{\mathbb{S}}_{m+n}^{\tilde{\sigma}(i')} - \hat{\mathbb{S}}_{m+n}^{\tilde{\sigma}(\sblock_m+l_{i'})})\kernel \rangle_{\kernel}^2 \big) \label{eq:operator_norm_A3}
\end{talign}
Using that $\mmd^2(\P,\Q) = \langle (\P - \Q)\kernel, (\P - \Q)\kernel \rangle_{\kernel}$, we upper-bound each of the terms in the right-hand side of \cref{eq:operator_norm_A}:
\begin{talign} 
\begin{split} \label{eq:bound_by_mmd_1}
    &|\langle (\hat{\P}_{m}^{\tilde{\sigma}(i)} - \hat{\Q}_n^{\tilde{\sigma}(\sblock_m+l_i)-\sblock_m})\kernel, (\hat{\P}_{m}^{\tilde{\sigma}(i')} - \hat{\Q}_n^{\tilde{\sigma}(\sblock_m+i')-\sblock_m})\kernel \rangle_{\kernel} - \langle (\P - \Q)\kernel, (\P - \Q)\kernel \rangle_{\kernel}| \\ &\leq |\langle (\hat{\P}_{m}^{\tilde{\sigma}(i)} - \hat{\Q}_n^{\tilde{\sigma}(\sblock_m+l_i)-\sblock_m})\kernel, (\hat{\P}_{m}^{\tilde{\sigma}(i')} - \hat{\Q}_n^{\tilde{\sigma}(\sblock_m+l_{i'})-\sblock_m})\kernel - (\P - \Q)\kernel \rangle_{\kernel}| \\ &+ |\langle (\hat{\P}_{m}^{\tilde{\sigma}(i)} - \hat{\Q}_n^{\tilde{\sigma}(\sblock_m+l_i)-\sblock_m})\kernel - (\P - \Q)\kernel, (\P - \Q)\kernel \rangle_{\kernel}| \\ &\leq \|(\hat{\P}_{m}^{\tilde{\sigma}(i)} - \hat{\Q}_n^{\tilde{\sigma}(\sblock_m+l_i)-\sblock_m})\kernel\|_{\kernel} 
    \cdot \|(\hat{\P}_{m}^{\tilde{\sigma}(i')} - \P)\kernel - (\hat{\Q}_n^{\tilde{\sigma}(\sblock_m+l_{i'})-\sblock_m} - \Q)\kernel \|_{\kernel} \\ &+ \|(\hat{\P}_{m}^{\tilde{\sigma}(i)} - \P)\kernel - (\hat{\Q}_n^{\tilde{\sigma}(\sblock_m+i)-\sblock_m} - \Q)\kernel \|_{\kernel} \cdot \|(\P - \Q)\kernel\|_{\kernel} \\ &\leq \|(\hat{\P}_{m}^{\tilde{\sigma}(i)} - \P)\kernel - (\hat{\Q}_n^{\tilde{\sigma}(\sblock_m+i)-\sblock_m} - \Q)\kernel \|_{\kernel} \cdot \|(\P - \Q)\kernel\|_{\kernel} + \|(\hat{\P}_{m}^{\tilde{\sigma}(i')} - \P)\kernel - (\hat{\Q}_n^{\tilde{\sigma}(\sblock_m+i')-\sblock_m} - \Q)\kernel \|_{\kernel} \cdot \|(\P - \Q)\kernel\|_{\kernel} \\ &+ \|(\hat{\P}_{m}^{\tilde{\sigma}(i)} - \P)\kernel - (\hat{\Q}_n^{\tilde{\sigma}(\sblock_m+l_i)-\sblock_m} - \Q)\kernel \|_{\kernel} \cdot \|(\hat{\P}_{m}^{\sigma(i')} - \P)\kernel - (\hat{\Q}_n^{\tilde{\sigma}(\sblock_m+l_{i'})-\sblock_m} - \Q)\kernel \|_{\kernel} \\ &= (\mmd(\hat{\P}_{m}^{\tilde{\sigma}(i)},\P) + \mmd(\hat{\Q}_n^{\tilde{\sigma}(\sblock_m+l_i)-\sblock_m},\Q) + \mmd(\hat{\P}_{m}^{\tilde{\sigma}(i')},\P) + \mmd(\hat{\Q}_n^{\tilde{\sigma}(\sblock_m+l_{i'})-\sblock_m},\Q)) \cdot \mmd(\P,\Q) \\ &+ (\mmd(\hat{\P}_{m}^{\tilde{\sigma}(i)},\P) + \mmd(\hat{\Q}_n^{\tilde{\sigma}(\sblock_m+l_i)-\sblock_m},\Q))(\mmd(\hat{\P}_{m}^{\tilde{\sigma}(i')},\P) + \mmd(\hat{\Q}_n^{\tilde{\sigma}(\sblock_m+l_{i'})-\sblock_m},\Q)).
\end{split}
\end{talign}
An analogous but simpler approach yields bounds for the terms in \cref{eq:operator_norm_A2} and \cref{eq:operator_norm_A3}:
\begin{talign} 
\begin{split} \label{eq:bound_by_mmd_2}
    &|\langle (\hat{\P}_{m}^{\tilde{\sigma}(i)} - \hat{\Q}_n^{\tilde{\sigma}(\sblock_m+l_i)-\sblock_m})\kernel, (\hat{\mathbb{S}}_{m+n}^{\tilde{\sigma}(i')} - \hat{\mathbb{S}}_{m+n}^{\tilde{\sigma}(\sblock_m+l_{i'})})\kernel \rangle_{\kernel}| \leq \| (\hat{\P}_{m}^{\tilde{\sigma}(i)} - \hat{\Q}_n^{\tilde{\sigma}(\sblock_m+l_i)-\sblock_m})\kernel \|_{\kernel} \cdot \| (\hat{\mathbb{S}}_{m+n}^{\tilde{\sigma}(i')} - \hat{\mathbb{S}}_{m+n}^{\tilde{\sigma}(\sblock_m+l_{i'})})\kernel \|_{\kernel} \\ &= \mmd(\hat{\P}_{m}^{\tilde{\sigma}(i)},\hat{\Q}_n^{\tilde{\sigma}(\sblock_m+l_i)-\sblock_m}) \cdot \mmd(\hat{\mathbb{S}}_{m+n}^{\tilde{\sigma}(i')},\hat{\mathbb{S}}_{m+n}^{\tilde{\sigma}(\sblock_m+l_{i'})}) \\ &\leq (\mmd(\hat{\P}_{m}^{\tilde{\sigma}(i)},\P) + \mmd(\P,\Q) + \mmd(\Q,\hat{\Q}_n^{\tilde{\sigma}(\sblock_m+l_i)-\sblock_m}) ) \cdot (\mmd(\hat{\mathbb{S}}_{m+n}^{\tilde{\sigma}(i')},\mathbb{S}) + \mmd(\mathbb{S},\hat{\mathbb{S}}_{m+n}^{\tilde{\sigma}(\sblock_m+l_{i'})}) ), \\
    &|\langle (\hat{\mathbb{S}}_{m+n}^{\tilde{\sigma}(i)} - \hat{\mathbb{S}}_{m+n}^{\tilde{\sigma}(\sblock_m+l_i)})\kernel, (\hat{\mathbb{S}}_{m+n}^{\tilde{\sigma}(i')} - \hat{\mathbb{S}}_{m+n}^{\tilde{\sigma}(\sblock_m+l_{i'})})\kernel \rangle_{\kernel}| \leq \mmd(\hat{\mathbb{S}}_{m+n}^{\tilde{\sigma}(i)},\hat{\mathbb{S}}_{m+n}^{\tilde{\sigma}(\sblock_m+l_i)}) \cdot \mmd(\hat{\mathbb{S}}_{m+n}^{\tilde{\sigma}(i')},\hat{\mathbb{S}}_{m+n}^{\tilde{\sigma}(\sblock_m+l_{i'})}) \\ &\leq (\mmd(\hat{\mathbb{S}}_{m+n}^{\tilde{\sigma}(i)},\mathbb{S}) + \mmd(\mathbb{S},\hat{\mathbb{S}}_{m+n}^{\tilde{\sigma}(\sblock_m+l_i)}) ) (\mmd(\hat{\mathbb{S}}_{m+n}^{\tilde{\sigma}(i')},\mathbb{S}) + \mmd(\mathbb{S},\hat{\mathbb{S}}_{m+n}^{\tilde{\sigma}(\sblock_m+l_{i'})}) )
\end{split}
\end{talign}
where $\mathbb{S}$ stands for $\P$ or $\Q$ as needed.
Applying \cref{lem:ktcompress_simultaneous}, we obtain that if \ktcompress calls are run with value $\delta/(5\sblock_m)$ and $\delta/(5\sblock_n)$ respectively, %
conditioned on the event $\mathcal{A}$ we have that simultaneously for any $i, i' \in \{1,\dots, \sblock_m\}$,
\begin{talign} \begin{split} \label{eq:bound_by_mmd_3}
    &|\langle (\hat{\P}_{m}^{\tilde{\sigma}(i)} - \hat{\Q}_n^{\tilde{\sigma}(\sblock_m+l_i)-\sblock_m})\kernel, (\hat{\P}_{m}^{\tilde{\sigma}(i')} - \hat{\Q}_n^{\tilde{\sigma}(\sblock_m+l_{i'})-\sblock_m})\kernel \rangle_{\kernel} - \langle (\P - \Q)\kernel, (\P - \Q)\kernel \rangle_{\kernel}|,  \\ &\leq \mmd(\P,\Q) W(m,n,\delta/(5\sblock)) + W(m,n,\delta/(5\sblock))^2, \\
    &|\langle (\hat{\P}_{m}^{\tilde{\sigma}(i)} - \hat{\Q}_n^{\tilde{\sigma}(\sblock_m+l_i)-\sblock_m})\kernel, (\hat{\mathbb{S}}_{m+n}^{\tilde{\sigma}(i')} - \hat{\mathbb{S}}_{m+n}^{\tilde{\sigma}(\sblock_m+l_{i'})})\kernel \rangle_{\kernel}| \leq W(m,n,\delta/(5\sblock)) \mmd(\P,\Q) + \frac{1}{2} W(m,n,\delta/(5\sblock))^2\\ 
    &|\mmd^2(\hat{\P}_{m}^{\tilde{\sigma}(i)},\hat{\Q}_n^{\tilde{\sigma}(\sblock_m+l_{i'})-\sblock_m}) - \mmd^2(\P,\Q)| \leq \mmd(\P,\Q) W(m,n,\delta/(5\sblock)) + W(m,n,\delta/(5\sblock))^2, \\
    &|\langle (\hat{\mathbb{S}}_{m+n}^{\tilde{\sigma}(i)} - \hat{\mathbb{S}}_{m+n}^{\tilde{\sigma}(\sblock_m+l_i)})\kernel, (\hat{\mathbb{S}}_{m+n}^{\tilde{\sigma}(i')} - \hat{\mathbb{S}}_{m+n}^{\tilde{\sigma}(\sblock_m+l_{i'})})\kernel \rangle_{\kernel}| \leq W(m,n,\delta/(5\sblock))^2 \\
    &\mmd^2(\hat{\mathbb{S}}_{m+n}^{\tilde{\sigma}(i)},\hat{\mathbb{S}}_{m+n}^{\tilde{\sigma}(\sblock_m+l_i)}) \leq W(m,n,\delta/(5\sblock))^2,
\end{split}
\end{talign}
where $W(m,n,\delta/(5\sblock))$ is defined as in \cref{eq:W_def}, and where we used that $\delta/(5\sblock) \leq \delta/(5\sblock_m)$ and $\delta/(5\sblock) \leq \delta/(5\sblock_n)$ since $\sblock_m,\sblock_n \leq \sblock$.
We conclude that %
conditioned on the event $\mathcal{A}$,
\begin{talign} \label{eq:A_norm_bound}
\|A\|_{\text{op}}^2 &\leq \|A\|_{\text{F}}^2 \leq \frac{1}{\sblock_m^2(\sblock_m - 1)^2} \big( |\rho_{\sigma}^{-1}(\{-1,1\})| (|\rho_{\sigma}^{-1}(\{-1,1\})|-1) (\mmd(\P,\Q) W(m,n,\delta/(5\sblock)) + W(m,n,\delta/(5\sblock))^2)^2 \\ &+ |\rho_{\sigma}^{-1}(\{-1,1\})||\rho_{\sigma}^{-1}(\{0\})| (\mmd(\P,\Q) W(m,n,\delta/(5\sblock)) + \frac{1}{2} W(m,n,\delta/(5\sblock))^2)^2 \\ &+ |\rho_{\sigma}^{-1}(\{0\})| (|\rho_{\sigma}^{-1}(\{0\})| -1) W(m,n,\delta/(5\sblock))^4 \big) \\ &\leq \eta(m,n,\delta)^2 \defeq \frac{1}{\sblock_m (\sblock_m - 1)} (\mmd(\P,\Q) W(m,n,\delta/(5\sblock)) + W(m,n,\delta/(5\sblock))^2)^2
\end{talign}
Applying \cref{lem:chernoff_rademacher}, we get that with probability at least $1-\delta''$,
\begin{talign}
    |\frac{1}{\sblock_m} \sum_{i \in \rho_{\sigma}^{-1}(\{-1,1\})} \tilde{\epsilon}_i| \leq |\frac{1}{\sblock_m} \sum_{i = 1}^{\sblock_m} \epsilon_i| \leq \sqrt{\frac{2}{\sblock_m} \log(\frac{2}{\delta''})},
\end{talign}
and this implies that with probability at least $1-\delta''$,
\begin{talign} \label{eq:c_tilde_bound}
    c_{\tilde{\epsilon}} &\defeq \frac{1}{\sblock_m(\sblock_m - 1)} \sum_{i \neq i' \in \rho_{\sigma}^{-1}(\{-1,1\})} \tilde{\epsilon}_i \tilde{\epsilon}_{i'} = \frac{1}{\sblock_m(\sblock_m - 1)} \big( (\sum_{i \in \rho_{\sigma}^{-1}(\{-1,1\})} \tilde{\epsilon}_i)^2 - \sum_{i \in \rho_{\sigma}^{-1}(\{-1,1\})} \tilde{\epsilon}_i^2 \big) \\ &\leq \frac{1}{\sblock_m(\sblock_m - 1)} (2 \sblock_m \log(\frac{2}{\delta''}) - |\rho_{\sigma}^{-1}(\{-1,1\})|) \leq \frac{2}{\sblock_m-1} \log(\frac{2}{\delta''}).
\end{talign}
Conditioned on the event $\mathcal{A}$ defined in \cref{lem:ktcompress_simultaneous}, we obtain that for any $x \geq 0$,
\begin{talign}
    \mathrm{Pr}_{\sigma}(\mathbb{E}_L[\mathcal{M}^{\sigma^{\epsilon,L},L}] \geq x) &\stackrel{\text{(i)}}{=} \mathrm{Pr}_{\sigma,\tilde{\epsilon}}(\mathbb{E}_L[\tilde{\epsilon}^{\top} A \tilde{\epsilon} - c_{\tilde{\epsilon}} \mmd^2(\P,\Q)] \geq x) \\ &\stackrel{\text{(ii)}}{\leq} \delta'' + e^{-\lambda^{*} x} \E_{\sigma,\tilde{\epsilon}}[\exp(\lambda^{*} \mathbb{E}_L[\tilde{\epsilon}^{\top} A \tilde{\epsilon} - \frac{2}{\sblock_m-1} \log(\frac{2}{\delta''}) \mmd^2(\P,\Q)])] \\ &\stackrel{\text{(iii)}}{\leq} \delta'' + e^{-\lambda^{*} x} \E_{\sigma,\tilde{\epsilon},L}[\exp(\lambda^{*} (\tilde{\epsilon}^{\top} A \tilde{\epsilon} - \frac{2}{\sblock_m-1} \log(\frac{2}{\delta''}) \mmd^2(\P,\Q)))] \\ &= \delta'' + \E_{\sigma,L}[e^{-\lambda^{*} (x + \frac{2}{\sblock_m-1} \log(\frac{2}{\delta''}) \mmd^2(\P,\Q))} \E_{\tilde{\epsilon}}[\exp(\lambda^{*} \tilde{\epsilon}^{\top} A \tilde{\epsilon})]] 
    \\ &\stackrel{\text{(iv)}}{\leq}
    \delta'' + \E_{\sigma,L} \big[ \exp (- \frac{\lambda^{*}}{K^2} (x + \frac{2}{\sblock_m-1} \log(\frac{2}{\delta''}) \mmd^2(\P,\Q))+ c'' (\lambda^{*})^2 \|A\|_{\text{F}}^2 ) \big]
    \\ &\stackrel{\text{(v)}}{\leq}
    \delta'' + \exp (- \frac{\lambda^{*}}{K^2} (x + \frac{2}{\sblock_m-1} \log(\frac{2}{\delta''}) \mmd^2(\P,\Q))+ c'' (\lambda^{*})^2 \eta(m,n,\delta)^2 ) 
    \\ &\stackrel{\text{(vi)}}{\leq}
    \delta'' + \exp \big( - c \min \{ \frac{(x + \frac{2}{\sblock_m-1} \log(\frac{2}{\delta''}) \mmd^2(\P,\Q))^2}{K^4\eta(m,n,\delta)^2}, \frac{x + \frac{2}{\sblock_m-1} \log(\frac{2}{\delta''}) \mmd^2(\P,\Q)}{K^2\eta(m,n,\delta)} \} \big)
\end{talign}
Here, (i) holds by \cref{eq:A_quadratic_form}, (ii) holds by \cref{eq:c_tilde_bound} and the application of a Chernoff bound; the value of $\lambda^{*}$ is to be set at a later point. Inequality (iii) holds by the convexity of the exponential function, and (iv) follows from \cref{lem:hanson_wright}  \cref{eq:hanson_1}. 
(v) holds because conditioned on $\mathcal{A}$, $\|A\|_{\text{F}}^2 \leq \eta(m,n,\delta)^2$ by equation \cref{eq:A_norm_bound}, and (vi) follows from \cref{lem:hanson_wright} \cref{eq:hanson_2}. 
If we set 
\begin{talign}
    \delta' = %
    \exp \big( - c \min \{ \frac{(x + \frac{2}{\sblock_m-1} \log(\frac{2}{\delta''}) \mmd^2(\P,\Q))^2}{K^4\eta(m,n,\delta)^2}, \frac{x + \frac{2}{\sblock_m-1} \log(\frac{2}{\delta''}) \mmd^2(\P,\Q)}{K^2\eta(m,n,\delta)} \} \big) \in (0,e^{-1})
\end{talign}
we have that
\begin{talign}
    x = \frac{c_0 (\mmd(\P,\Q) W(m,n,\delta/(5\sblock)) + W(m,n,\delta/(5\sblock))^2)\log(1/\delta')}{\sqrt{\sblock_m(\sblock_m-1)}} + \frac{2}{\sblock_m-1} \log(\frac{2}{\delta''}) \mmd^2(\P,\Q),
\end{talign}
where we defined $c_0 = K^2/c$ and we used that $\delta \in (0,e^{-1})$.
We conclude that conditioned on $\mathcal{A}$, with probability at least $1-\delta'-\delta''$,
\begin{talign}
    \mathbb{E}_L[\mathcal{M}^{\sigma^{\epsilon,L},L}] \leq \frac{c_0 (\mmd(\P,\Q) W(m,n,\delta/(5\sblock)) + W(m,n,\delta/(5\sblock))^2)\log(1/\delta')}{\sqrt{\sblock_m(\sblock_m-1)}} + \frac{2}{\sblock_m-1} \log(\frac{2}{\delta''}) \mmd^2(\P,\Q).
\end{talign}

To show \cref{eq:other_terms_prop}, we use the same arguments of \cref{eq:bound_by_mmd_1}, \cref{eq:bound_by_mmd_2} and \cref{eq:bound_by_mmd_3}.  %
\end{proof}

Plugging the results of \cref{lem:mathcal_M_bound} into the right-hand side of \cref{eq:mmd_perm_decomposition} shows that conditioned on the event $\mathcal{A}$, with probability at least $1-\delta'-\delta''$,
\begin{talign}
    \mmd^2(\hatxs^{\sigma},\hatys^{\sigma}) &\leq \frac{(\sblock_m-1)(\sblock_n-1)}{\sblock_m \sblock_n} \big( \frac{2}{\sblock_m - 1} \log(\frac{2}{\delta''}) \mmd^2(\P,\Q) \\ &\qquad\qquad\qquad\quad + \frac{c' \log(1/\delta')}{\sqrt{\sblock_m(\sblock_m-1)}} \big( \mmd(\P,\Q) W(m,n,\delta/(5\sblock)) + W(m,n,\delta/(5\sblock))^2 \big) \big) \\ &+ \frac{\sblock - 1}{\sblock_m \sblock_n}
    (\mmd(\P,\Q) W(m,n,\delta/(5\sblock)) + W(m,n,\delta/(5\sblock))^2 + \mmd^2(\P,\Q)),
\end{talign}
which concludes the proof of \cref{prop:random_perm_mmd_bound} upon simplification.

\begin{lemma}[Hanson-Wright inequality, \citet{rudelson2013hanson}, Thm.~1.1, adapted] \label{lem:hanson_wright}
Let $X = (X_1,\dots,X_n) \in \R^n$ be a random vector with sub-Gaussian independent components $X_i$ which satisfy $\E X_i = 0$, and $\E[X^{\top} A X] = 0$, and $\| X_i\|_{\psi_2} = \inf \{ K' | \E[\exp(X_i^2/(K')^2)] < 2 \} \leq K$. Let $A$ be an $n \times n$ matrix. Then, there exists $c, c', c''>0$ and such that for every $t \geq 0$, $\lambda \leq c'/\|A\|_{\text{op}}$,
\begin{talign} \label{eq:hanson_1}
    \mathrm{Pr}(X^{\top} A X > t) \leq e^{-\lambda t} \E[e^{\lambda X^{\top} A X}] \leq \exp (- \lambda t/K^2 + c'' \lambda^2 \|A\|_{\text{F}}^2 ).
\end{talign}
Optimizing this bound over $\lambda \leq c_0/\|A\|_{\text{op}}$, one obtains
\begin{talign} \label{eq:hanson_2}
    \mathrm{Pr}(X^{\top} A X > t) \leq \exp(-\min\{ \frac{t^2}{K^4 \|A\|_{\text{F}}^2}, \frac{t}{K^2 \|A\|_{\text{op}}} \}).
\end{talign}
\end{lemma}
\begin{lemma}[Simultaneous bound on the MMD errors of $\hat{\P}_m^i$ and $\hat{\Q}_n^j$] \label{lem:ktcompress_simultaneous}
Suppose that \ktcompress is run with value $\delta/(5\sblock_m)$ and $\delta/(5\sblock_n)$, i.e. $\hatxs^{(i)} \leftarrow \ktcompress( \xs^{(i)}, \ossymb, \kernel, \kersplit, \delta/(5\sblock_m))$ and $\hatys^{(i)} \leftarrow \ktcompress( \ys^{(i)}, \ossymb, \kernel, \kersplit, \delta/(5\sblock_n))$. Then, with probability at least $1-\delta$, we have that simultaneously for all $i \in \{1,\dots, \sblock_m\}$, $j \in \{1,\dots, \sblock_n\}$,
\begin{talign}
    \mmd(\hat{\P}_m^i,\P) \leq \frac{\sqrt{\sblock_m} \errorsplit(\P, m/\sblock_m, \delta/(5\sblock_m),\ossymb)}{2^{\ossymb} \sqrt{m}} \!+ \!c_{\delta/(5\sblock_m)} \sqrt{\frac{\sblock_m \kinfnorm}{m}}, \\ 
    \mmd(\hat{\Q}_n^j,\Q) \leq \frac{\sqrt{\sblock_n} \errorsplit(\Q, n/\sblock_n, \delta/(5\sblock_n),\ossymb)}{2^{\ossymb} \sqrt{n}} \!+ \!c_{\delta/(5\sblock_n)} \sqrt{\frac{\sblock_n \kinfnorm}{n}}.
\end{talign}
We define $\mathcal{A}$ as the event that these $\sblock$ conditions take place simultaneously, and observe that $\mathrm{Pr}(\mathcal{A}) \geq 1-\delta$.
\end{lemma}
\begin{proof}
Using the notation in \cref{sub:proof_mmd_diff_x_y}, we can write that
\begin{talign}
    \mmd( \hat{\mathbb{P}}_m^i, \mathbb{P}_m^i ) %
    &\!\leq\! \twonorm{u_{\mathrm{C+},\xs^i} \cdot \sblock_m/m %
    } = \twonorm{\frac{\sblock_m}{m}\sum_{k=0}^{\log_4 (m/\sblock) \!-\!\ossymb\!-\!1} \sum_{j=1}^{4^{k}} w_{k, m/\sblock_m} u_{k, j,\xs}^{(i)} %
    } \\ &= \lambda_{\max}({ \frac{\sblock_m}{m}\sum_{k=0}^{\log_4 (m/\sblock_m) \!-\!\ossymb\!-\!1} \sum_{j=1}^{4^{k}} \genmat_{k, j,\xs}^{(i)} %
    }).
\end{talign}
We define 
\begin{talign}
\sigma^2_i &\defeq \sum_{k=0}^{\log_4 (m/\sblock_m) \!-\!\ossymb\!-\!1} \sum_{j=1}^{4^{k}} (\frac{2\sblock_m}{m} w_{k,m/\sblock_m}
\l_{k,m/\sblock_m}' (a_{\l_{k,m/\sblock_m}',m/\sblock_m}(\frac{\delta}{5\sblock_m}) + v_{\l_{k,m/\sblock_m}',m/\sblock_m}(\frac{\delta}{5\sblock_m})))^2 \\ &= \sblock_m \sum_{k=0}^{\log_4 (m/\sblock_m) \!-\!\ossymb\!-\!1} 
(\frac{1}{2^{\ossymb}\sqrt{m}}
(C_{\shiftparam}\sqrt{\staticinfnorm{\kersplit}} + C_{\kgaussparam} \sqrt{\staticinfnorm{\kersplit}\log( \frac{6 \cdot 4^{\ossymb} \sqrt{m/\sblock_m} (\beta_{m/\sblock_m}+1)}{2^{\ossymb-k} \delta/(5\sblock_m)})} \mathfrak{M}_{\xs^i,\kersplit}))^2
\\ &\leq \frac{\sblock_m (\log_4 (m/\sblock_m) \!-\!\ossymb\!-\!1)\staticinfnorm{\kersplit}}{4^{\ossymb}m} (C_{\shiftparam} + C_{\kgaussparam} \sqrt{\log( \frac{15 m (\log_4 (m/\sblock_m) \!-\!\ossymb\!-\!1)}{\delta})} \mathfrak{M}_{\xs^i,\kersplit})^2
\end{talign}
When comparing this equation with \cref{eq:sigma2_def1}, we have replaced $\delta$ by $\delta/(5\sblock_m)$. We obtain that 
\begin{talign} \label{eq:matrix_friedman_app_simultaneous}
    \mathrm{Pr}(\mmd ( \hat{\mathbb{P}}_m^i, \mathbb{P}_m^i ) > \sigma_i \sqrt{8(\log(m+1)+t)}) \leq \frac{\delta}{10\sblock_m} + e^{-t}, \qquad \qtext{for all} t \geq 0.
\end{talign}
Equivalently, with probability at least $1-\delta/(5\sblock_m)$,
\begin{talign}
\mmd ( \hat{\mathbb{P}}_m^i, \mathbb{P}_m^i ) \leq \sqrt{8} \sigma_i (\sqrt{\log(m+1)} + \sqrt{\log(2\sblock_m/\delta)}).
\end{talign}
An application of \cref{lem:concentration_mmd} with $\Q = \P$ and $n=+\infty$ yields 
\begin{talign} \label{eq:MMD_difference_2}
    \mathrm{Pr}\bigg[\mathrm{MMD}(\P_m^i,\P) > (2 + \sqrt{2 \log(\frac{10\sblock}{\delta})}) \sqrt{\frac{\sblock \kinfnorm}{m}} \bigg] \leq 
    \frac{\delta}{5\sblock}.
\end{talign}
Hence, with probability at least $1-2\delta/(5\sblock_m)$,
\begin{talign}
    \mmd(\hat{\P}_m^i,\P) &\leq \mmd ( \hat{\P}_m^i, \P_m^i ) + \mmd(\P_m^i,\P) \\ &\leq \sqrt{8} \sigma_i (\sqrt{\log(m+1)} + \sqrt{\log(2\sblock_m/\delta)}) + (2 + \sqrt{2 \log(\frac{2\sblock_m}{\delta})}) \sqrt{\frac{\sblock_m \kinfnorm}{m}}.
\end{talign}
Defining $\mathfrak{M}_{\P,m,\kersplit}(\delta)$ as in \cref{eq:frak_M_population} and \cref{eq:frak_M_population_k}, we obtain that with probability at least $1-\delta/(10\sblock_m)$, $\mathfrak{M}_{\xs^i,\kersplit} \leq \mathfrak{M}_{\P,m/\sblock_m,\kersplit}(\delta/(5\sblock_m))$. Hence, with probability at least $1-5\delta/(10\sblock_m) = 1-\delta/(2\sblock_m)$,
\begin{talign}
    \mmd(\hat{\P}_m^i,\P) \leq \frac{\sqrt{\sblock_m} \errorsplit(\P, m/\sblock_m, \delta/(5\sblock_m),\ossymb)}{2^{\ossymb} \sqrt{m}} \!+ \!c_{\delta/(5\sblock_m)} \sqrt{\sblock_m \frac{\kinfnorm}{m}},
\end{talign}
Using a union bound, we obtain the result.
\end{proof}

\begin{lemma}[Chernoff bound for sums of Rademacher variables] \label{lem:chernoff_rademacher}
Let $\epsilon = (\epsilon_i)_{i=1}^{k}$ be i.i.d. Rademacher variables. We have that for any $x > 0$,
\begin{talign}
    \mathrm{Pr}\big( |\frac{1}{k} \big(\sum_{i=1}^{k} \epsilon_i \big) |  > x \big) %
    \leq e^{-D((1+x)/2||1/2)k} + e^{-D((1-x)/2||1/2)k} \leq 2\exp(-\frac{x^2 k}{2}),
\end{talign}
where $D(x||y) = x \log(\frac{x}{y}) + (1-x) \log(\frac{1-x}{1-y})$.
\end{lemma}
\begin{proof}
The first inequality holds by the Chernoff-Hoeffding theorem. The second inequality holds because $D((1+x)/2||1/2) = D((1-x)/2||1/2)$, and because for $p \geq 1/2$, we have that $D(p+x||p) \geq \frac{x^2}{2p(1-p)}$.
\end{proof}

\subsection{Concluding the proof of \cref{thm:uniform_separation_detailed}}

The following result is the basis for \cref{thm:uniform_separation_detailed}.

\begin{lemma}[Putting everything together] \label{thm:uniform_separation_details}
    Let $\beta \in (0,1)$ be arbitrary, and define $\tilde{\beta} = \frac{\beta}{1+\frac{\beta}{2}}$. Let $\alpha \in (0,1)$ and suppose that $k_{\alpha} \defeq \floor{\alpha (\numperm+1)} \geq 1$. Assume that \ctt is run with $\numperm \geq \frac{1}{\alpha} - 1$ and $\delta = \min\{ \frac{\tilde{\beta}}{6}, (\frac{\tilde{\beta}}{2})^{1/k_{\alpha}} \frac{\alpha}{30e\sblock} \}$.
    Define the function 
    \begin{talign}
    \tilde{Z}(m,n,\alpha,\beta)
    \defeq Z(m,n,\beta) + \hat{Z}(m,n,\alpha,\tilde{\beta}),
    \end{talign}
    where the functions $Z$ and $\hat{Z}$ are defined in \cref{prop:control_power} (equation \cref{eq:Z_def}) and \cref{prop:bound_q_arbitrary_B} (equation \cref{eq:hat_Z_def}), respectively.
    If 
    $\mathrm{MMD}(\P,\Q) \geq \tilde{Z}(m,n,\alpha,\beta)$,
    then  
    \begin{talign} \label{eq:power_condition}
    \mathrm{Pr}[\mathrm{MMD}(\hat{\mathbb{X}}_m,\hat{\mathbb{Y}}_n) \leq %
    M_{(b_{\alpha})}] \leq \beta.
    \end{talign}
\end{lemma}
\begin{proof}
Using \cref{prop:control_power}, it suffices to see that with probability at least $1-\frac{\tilde{\beta}}{2}$,
\begin{talign} \label{eq:power_condition_2}
\mathrm{MMD}(\P,\Q) \geq
Z(m,n,\beta) + %
M_{(b_{\alpha})}.
\end{talign}
\cref{prop:bound_q_arbitrary_B} implies that with probability at least $1-\frac{\tilde{\beta}}{2}$,
\begin{talign}
    M_{(b_{\alpha})} \leq \hat{Z}(m,n,%
    \alpha,\tilde{\beta}).
\end{talign}
Hence, with probability at least $1-\frac{\tilde{\beta}}{2}$, $\tilde{Z}(m,n,\alpha,\beta) \geq Z(m,n,\beta) + %
M_{(b_{\alpha})}$. Using the assumption that $\mathrm{MMD}(\P,\Q) \geq \tilde{Z}(m,n,\alpha,\beta)$, we obtain that \cref{eq:power_condition} holds.
\end{proof}

\textbf{Proof of \cref{thm:uniform_separation_detailed}} 
To go from the statement of \cref{thm:uniform_separation_details} to the one of \cref{thm:uniform_separation_detailed}, we write the function $\tilde{Z}(m,n,\alpha,\beta)$ %
in terms of its arguments, as follows:
\begin{talign}
    \tilde{Z}(m,n,\alpha,\beta) &= 
    \frac{\errorsplit(\P, m/\sblock_m, \tilde{\beta}%
    /6,\ossymb)}{2^{\ossymb} \sqrt{m}} +  \frac{\errorsplit(\Q, n/\sblock_n, \tilde{\beta}%
    /6,\ossymb)}{2^{\ossymb} \sqrt{n}} +c_{\tilde{\beta}%
    /6} \bigparenth{\sqrt{\frac{\kinfnorm}{m}}\!+\!\sqrt{\frac{\kinfnorm}{n}}} 
    \\ &+ \sqrt{\frac{2}{\sblock_m} \log(\frac{8e^2}{\alpha} (\frac{4}{\tilde{\beta}})^{1/k_{\alpha}})} \mmd(\P,\Q) \\ &+ \sqrt{\frac{1}{\sblock_m}(2+c' \log(\frac{4e}{\alpha} (\frac{4}{\tilde{\beta}})^{1/k_{\alpha}}))} (\sqrt{\mmd(\P,\Q) W(m,n,\tilde{\beta}/(20\sblock))} + W(m,n,\tilde{\beta}/(20\sblock))).
\end{talign}
If we define 
\begin{talign} 
\begin{split} \label{eq:inequality_coefficients}
    a &= %
    1 - \sqrt{\frac{2}{\sblock_m} \log(\frac{8e^2}{\alpha} (\frac{4}{\tilde{\beta}})^{1/k_{\alpha}})}, \\
    b &= \sqrt{\frac{1}{\sblock_m}(2+c' \log(\frac{4e}{\alpha} (\frac{2}{\tilde{\beta}})^{1/k_{\alpha}})) W(m,n,\tilde{\beta}/(20\sblock))}, 
    \\
    c \! &= \! \frac{\errorsplit(\P, m/\sblock_m, \tilde{\beta}%
    /6,\ossymb)}{2^{\ossymb} \sqrt{m}} \! + \!  \frac{\errorsplit(\Q, n/\sblock_n, \tilde{\beta}%
    /6,\ossymb)}{2^{\ossymb} \sqrt{n}} \! + \! c_{\tilde{\beta}%
    /6} \bigparenth{\sqrt{\frac{\kinfnorm}{m}}\!+\!\sqrt{\frac{\kinfnorm}{n}}} \! + \! \sqrt{\frac{2+c' \log(\frac{4e}{\alpha} (\frac{4}{\tilde{\beta}})^{1/k_{\alpha}})}{\sblock_m}} W(m,n,\tilde{\beta}/(20\sblock)), \\
    x &= \sqrt{\mmd(\P,\Q)},
\end{split}
\end{talign}
and we assume that $a > 0$, we can rewrite the condition $\mathrm{MMD}(\P,\Q) \geq \tilde{Z}(m,n,\alpha,\beta)$ as $a x^2 - b x - c \geq 0$, which together with the positivity constraint on $x$ is equivalent to $x \geq \frac{b + \sqrt{b^2 + 4 a c}}{2a}$. A sufficient condition for this is $\sqrt{\mmd(\P,\Q)} \geq \frac{b}{a} + \sqrt{\frac{c}{a}}$, and yet another sufficient condition is $\mmd(\P,\Q) \geq 2(\frac{b^2}{a^2} + \frac{c}{a})$. Since $0 < a < 1$ by assumption, the right-hand side of this equation is upper-bounded by
\begin{talign}
    \frac{2}{a^2}(b^2 + c) &= \frac{2}{(1 - \sqrt{\frac{2}{\sblock_m} \log(\frac{8e^2}{\alpha} (\frac{4}{\tilde{\beta}})^{1/k_{\alpha}})})^2} \bigg(\frac{\errorsplit(\P, m/\sblock_m, \tilde{\beta}%
    /6,\ossymb)}{2^{\ossymb} \sqrt{m}} +  \frac{\errorsplit(\Q, n/\sblock_n, \tilde{\beta}%
    /6,\ossymb)}{2^{\ossymb} \sqrt{n}} +c_{\tilde{\beta}%
    /6} \bigparenth{\sqrt{\frac{\kinfnorm}{m}}\!+\!\sqrt{\frac{\kinfnorm}{n}}} \\ &+ 2\big(\frac{1}{\sblock_m}(2+c' \log(\frac{4e}{\alpha} (\frac{4}{\tilde{\beta}})^{1/k_{\alpha}})) + \sqrt{\frac{1}{\sblock_m}(2+c' \log(\frac{4e}{\alpha} (\frac{4}{\tilde{\beta}})^{1/k_{\alpha}}))} \big) \times \\ &\big(\frac{\errorsplit(\P, m/\sblock_m, \tilde{\beta}/(20\sblock_m),\ossymb
    )}{2^{\ossymb} \sqrt{m/\sblock_m}} \!+ \!\frac{ \errorsplit(\Q, n/\sblock_n, \tilde{\beta}/(20\sblock_n),\ossymb
    )}{2^{\ossymb} \sqrt{n/\sblock_n}} \!+ \!c_{\tilde{\beta}/(20\sblock)} \big(\sqrt{\frac{\kinfnorm}{m/\sblock_m}} +\sqrt{\frac{\kinfnorm}{n/\sblock_n}} \big) \big) %
    \bigg). 
\end{talign}
We have that 
\begin{talign} \label{eq:cond_s_m}
    \sqrt{\frac{2}{\sblock_m} \log(\frac{8e^2}{\alpha} (\frac{4}{\tilde{\beta}})^{1/k_{\alpha}})} \leq \frac{3}{4} \iff \sblock_m \geq \frac{32}{9} \log(\frac{8e^2}{\alpha} (\frac{4}{\tilde{\beta}})^{1/k_{\alpha}})
\end{talign}
Under \cref{eq:cond_s_m}, we have that 
\begin{talign}
    \frac{2}{(1 - \sqrt{\frac{2}{\sblock_m} \log(\frac{8e^2}{\alpha}(\frac{4}{\tilde{\beta}})^{1/k_{\alpha}})})^2} &\leq \frac{2}{(1-\frac{3}{4})^2} = 32, \\
    \frac{1}{\sqrt{\sblock_m}}(2+c' \log(\frac{4e}{\alpha} (\frac{4}{\tilde{\beta}})^{1/k_{\alpha}})) &\leq \frac{2+c' \log(\frac{4e}{\alpha} (\frac{4}{\tilde{\beta}})^{1/k_{\alpha}})}{\sqrt{\frac{9}{32} \log(\frac{8e^2}{\alpha} (\frac{4}{\tilde{\beta}})^{1/k_{\alpha}})}} \leq \sqrt{\frac{32}{9}}(2+c' \sqrt{\log(\frac{4e}{\alpha} (\frac{4}{\tilde{\beta}})^{1/k_{\alpha}})})
\end{talign}
and consequently,
\begin{talign}
    \frac{2}{a^2}(b^2 + c) &\leq 32 \bigg(\frac{\errorsplit(\P, m/\sblock_m, \tilde{\beta}
    /6)}{2^{\ossymb} \sqrt{m}} +  \frac{\errorsplit(\Q, n/\sblock_n, \tilde{\beta}
    /6,\ossymb)}{2^{\ossymb} \sqrt{n}} +c_{\tilde{\beta}
    /6} \bigparenth{\sqrt{\frac{\kinfnorm}{m}}\!+\!\sqrt{\frac{\kinfnorm}{n}}} \\ &\qquad\quad + 2\big(\sqrt{\frac{9}{32}} + 1\big) \big(2+c' \sqrt{\log(\frac{4e}{\alpha} (\frac{4}{\tilde{\beta}})^{1/k_{\alpha}})} \big)
    \bigg(\frac{\errorsplit(\P, m/\sblock_m, \tilde{\beta}/(20\sblock)
    ,\ossymb)}{2^{\ossymb} \sqrt{m}} \\ &\qquad\qquad\qquad\qquad\qquad\qquad\qquad\qquad+ \sqrt{\frac{\sblock_n}{\sblock_m}} \frac{ \errorsplit(\Q, n/\sblock_n, \tilde{\beta}/(20\sblock)
    ,\ossymb)}{2^{\ossymb} \sqrt{n}} \!+ \!c_{\tilde{\beta}/(20\sblock)} \big(\sqrt{\frac{\kinfnorm}{m}} +\sqrt{\frac{\sblock_n \kinfnorm}{\sblock_m n}} \big) \bigg)
    \bigg)
\end{talign}
The final result follows.

\section{\pcref{thm:rates_other_tests}}
\label{proof_of_rates_other_tests}
We prove the three parts of \cref{thm:rates_other_tests} one by one.

\subsection{Proof of \cref{thm:rates_other_tests}\cref{thm:rate_complete}} \label{proof:rate_complete}
According to \citet[Thm.~8]{gretton2007kernel} (see also \citet[Eq.~2]{gretton2009afast} for the exact formulation), the complete unbiased test with statistic $\mmd^2_{up}(\xs,\ys[m])$ has the following asymptotic distribution under the null hypothesis:
\begin{talign}
    m \mmd^2_{up}(\xs,\ys[m]) \rightarrow \sum_{l = 1}^{\infty} \lambda_l (z_l^2 - 2),
\end{talign}
where $z_l \sim N(0,2)$ i.i.d., $\rightarrow$ denotes convergence in distribution and $\lambda_i$ are the solutions to the eigenvalue equation
\begin{talign}
    \int \tilde{k}(x_i,x_j) \psi_l(x_i) dP(x_i) = \lambda_l \psi_l(x_j),
\end{talign}
where $\tilde{k}(x_i,x_j) = k(x_i,x_j) - \E_x k(x_i,x) - \E_x k(x, x_j) + \E_{x,x'} k(x,x')$. Hence, the variance of $\mmd^2_{up}(\xs,\ys[m])$ is (asymptotically)
\begin{talign}
    \mathrm{Var}(\mmd^2_{up}(\xs,\ys[m])) = \frac{1}{m^2}\sum_{l = 1}^{\infty} \lambda_l^2 \mathrm{Var}(z_l^2 - 2) = \frac{4}{m^2} \sum_{l = 1}^{\infty} \lambda_l^2,
\end{talign}
where the last equality holds because $z_l^2$ is distributed like a chi-squared distribution of one degree of freedom scaled by $\sqrt{2}$, which has variance $4$. Since the asymptotic threshold $t_{1-\alpha}$ of level $\alpha$ for $\mmd^2_{up}(\xs,\ys[m])$ is of the order of the standard deviation of $\mmd^2_{up}(\xs,\ys[m])$, we can write 
\begin{talign} \label{eq:t_1_m_alpha}
    t_{1-\alpha} = \frac{K_{1-\alpha}}{m},
\end{talign}
where the constant $K_{1-\alpha}$ is of the order of $2 \sqrt{\sum_{l = 1}^{\infty} \lambda_l^2}$ and depends on $\alpha$.

Under the alternative, $\mmd^2_{up}(\xs,\ys[m])$ converges in distribution to a Gaussian according to
\begin{talign} \label{eq:dist_up}
    m^{1/2} (\mmd^2_{up}(\xs,\ys[m]) - \mmd^2(\P,\Q)) \rightarrow N(0,\sigma_{up}^2),
\end{talign}
where $\sigma_{up}^2 = 4 (\E_{x,y}(\E_{x',y'} \hkernel(x,x',y,y'))^2 - (\E_{x,y,x',y'} \hkernel(x,x',y,y'))^2)$ \citep[Sec.~6]{gretton2007kernel}, \citep[Sec.~5.5]{serfling2009approximation}. 
Let $z = (x,y), z' = (x',y')$ and $\hkernel(z,z') = \hkernel(x,x',y,y')$. If $\angles{\cdot,\cdot}_{\kernel}$ denotes the RKHS inner product, note that
\begin{talign} 
\begin{split} \label{eq:order_MMD}
&|\mathbb{E}_{z'} [\hkernel(z,z')]| \\ &= |\mathbb{E}_{x',y'} [\kernel(x,x') + \kernel(y,y') - \kernel(x,y') - \kernel(x',y)]| \\ &= \big|\int \kernel(x,x') \, d(\P-\Q)(x') - \int \kernel(y,x') \, d(\P-\Q)(x')\big| = \big|\int (\kernel(x,x') - \kernel(y,x')) \, d(\P-\Q)(x')\big| \\ &= \big|\int \kernel(x'',x') \, d(\P-\Q)(x') \, d(\delta_x -\delta_y)(x'') \big| = \big|\left\langle \int \kernel(\cdot,x') \, d(\P-\Q)(x'), \int \kernel(\cdot,x') \, d(\delta_x -\delta_y)(x') \right\rangle_{\kernel} \big| \\ &\leq \big\| \int \kernel(\cdot,x') \, d(\P-\Q)(x') \big\|_{\kernel} \big\| \int \kernel(\cdot,x') \, d(\delta_x -\delta_y)(x') \big\|_{\kernel} \\ &= \mathrm{MMD}(\P,\Q) (\kernel(x,x) + \kernel(y,y) - 2\kernel(x,y)).
\end{split}
\end{talign}
Hence, using equation \cref{eq:order_MMD} we obtain an upper bound on $\sigma_{up}^2$:
\begin{talign} 
\begin{split} \label{eq:sigma_1_2_bound}
    \sigma_{up}^2 &:= 4 (\mathbb{E}_{z}[(\mathbb{E}_{z'}[\hkernel(z, z')])^2] - (\mathbb{E}_{z,z'}[\hkernel(z, z')])^2) \leq 4 \mathbb{E}_{z}[(\mathbb{E}_{z'}[\hkernel(z, z')])^2] \\ &\leq 4 \mathrm{MMD}^2(\P,\Q) \mathbb{E}_{x,y}[(\kernel(x,x) + \kernel(y,y) - 2\kernel(x,y))^2] = 4 \mathrm{MMD}^2(\P,\Q) \mathbb{E}_z[\hkernel(z,z)^2].
\end{split}
\end{talign}
Asymptotically, we obtain that under the alternative distribution
\begin{talign}
    &\mathrm{Pr}(\mmd^2_{up}(\xs,\ys[m]) < t_{1-\alpha}) \\ &= \mathrm{Pr}(\frac{\sqrt{m}}{\sigma_{up}}(\mmd^2_{up}(\xs,\ys[m])  - \mmd^2(\P,\Q)) < \frac{\sqrt{m}}{\sigma_{up}}(t_{1-\alpha} - \mmd^2(\P,\Q))) \\ &= \Phi(\frac{\sqrt{m}}{\sigma_{up}}(t_{1-\alpha} - \mmd^2(\P,\Q))),
\end{talign}
where $\Phi$ denotes the CDF of a standard Gaussian. Hence, the condition that the Type II error be upper-bounded by $\beta \in (0,1/2)$ translates to
\begin{talign} \label{eq:mmd2_cond_0}
    \frac{\sqrt{m}}{\sigma_{up}}(t_{1-\alpha} - \mmd^2(\P,\Q)) &\leq \Phi^{-1}(\beta) \\ \iff \mmd^2(\P,\Q) &\geq t_{1-\alpha} -  \frac{\sigma_{up}}{\sqrt{m}} \Phi^{-1}(\beta) = t_{1-\alpha} + \frac{\sigma_{up}}{\sqrt{m}} \Phi^{-1}(1-\beta)
\end{talign}
Replacing $t_{1-\alpha}$ by its expression in \cref{eq:t_1_m_alpha} and using the upper bound \cref{eq:sigma_1_2_bound}, we get that a sufficient condition for \cref{eq:mmd2_cond_0} is
\begin{talign}
    \mmd^2(\P,\Q) - 2\sqrt{\frac{ \mathbb{E}_z[\hkernel(z,z)^2]}{m}} \Phi^{-1}(1-\beta) \mmd(\P,\Q) - \frac{K_{1-\alpha}}{m} \geq 0.
\end{talign}
Solving the corresponding second-degree equation, this is equivalent to
\begin{talign}
    \mmd(\P,\Q) \geq \frac{\sqrt{ \mathbb{E}_z[\hkernel(z,z)^2]} \Phi^{-1}(1-\beta) + \sqrt{ \mathbb{E}_z[\hkernel(z,z)^2] \Phi^{-1}(1-\beta)^2 + K_{1-\alpha}}}{\sqrt{m}} = O(1/\sqrt{m})
\end{talign}
A necessary condition for \cref{eq:mmd2_cond_0} to hold is $\mmd(\P,\Q) \geq \sqrt{\frac{K_{1-\alpha}}{m}} = \Omega(1/\sqrt{m})$, which concludes the proof of this part.

\subsection{Proof of \cref{thm:rates_other_tests}\cref{thm:rate_block}}
\label{thm:proof_rate_block}

By the definition of $\mmdblock^2(\xs, \ys[m])$ in \cref{eq:tauBdef}, it is the average of $m/\block$ independent instances $\eta_i(\xs, \ys[m])$ of the quadratic estimator \cref{eq:unbiased_equal}, each with sample size $\block$. Hence, in the regime $m/\block \to \infty$, we have
\begin{talign}
\frac{m}{\block}(\mmdblock^2(\xs, \ys[m]) - \mmd^2(\P,\Q)) \rightarrow N(0,\Var(\eta_i(\xs, \ys[m]))).
\end{talign}
Using the argument of \cref{proof:rate_complete}, under the null hypothesis, we obtain that asymptotically $\mathrm{Var}(\eta_i(\xs, \ys[m])) = \frac{4}{\block^2} \sum_{l = 1}^{\infty} \lambda_l^2$. Hence, under the null hypothesis,
$%
\sqrt{\frac{B m}{4\sum_{l = 1}^{\infty} \lambda_l^2}} \mmdblock^2(\xs, \ys[m]) \to N(0,1)$. We derive the expression for the threshold $t_{1-\alpha}$ corresponding to the level $\alpha$:
\begin{talign}
&\mathrm{Pr}(\mmdblock^2(\xs, \ys[m]) < t_{1-\alpha}) = \mathrm{Pr}(\sqrt{\frac{B n}{4\sum_{l = 1}^{\infty} \lambda_l^2}} \mmdblock^2(\xs, \ys[m]) < \sqrt{\frac{B n}{4\sum_{l = 1}^{\infty} \lambda_l^2}} t_{1-\alpha}) \\ &= \Phi(\sqrt{\frac{B m}{4\sum_{l = 1}^{\infty} \lambda_l^2}}t_{1-\alpha}) = 1-\alpha.
\end{talign}
This implies that 
\begin{talign} \label{eq:t_1_alpha_block}
t_{1-\alpha} = \sqrt{\frac{4\sum_{l = 1}^{\infty} \lambda_l^2}{B m}} \Phi^{-1}(1-\alpha).
\end{talign}
Reusing \cref{eq:dist_up}, we have that asymptotically, under the alternative hypothesis, $\Var(\eta_i(\xs, \ys[m])) = \sigma_{up}^2/\block$. Hence, under the alternative hypothesis,
\begin{talign}
\frac{\sqrt{m}}{\sigma_{up}} (\mmdblock^2(\xs, \ys[m]) - \mmd^2(\P,\Q)) \rightarrow N(0,1)).
\end{talign}
We conclude that asymptotically,
\begin{talign}
    &\mathrm{Pr}(\mmd^2_{\block}(\xs,\ys[m]) < t_{1-\alpha}) \\ &= \mathrm{Pr}(\frac{\sqrt{m}}{\sigma_{up}}(\mmd^2_{\block}(\xs,\ys[m])  - \mmd^2(\P,\Q)) < \frac{\sqrt{m}}{\sigma_{up}}(t_{1-\alpha} - \mmd^2(\P,\Q))) \\ &= \Phi(\frac{\sqrt{m}}{\sigma_{up}}(t_{1-\alpha} - \mmd^2(\P,\Q))).
\end{talign}
Hence, the condition that the Type II error be upper-bounded by $\beta \in (0,1/2)$ translates to
\begin{talign} \label{eq:mmd2_cond_1}
    \frac{\sqrt{m}}{\sigma_{up}}(t_{1-\alpha} - \mmd^2(\P,\Q)) \leq \Phi^{-1}(\beta) \iff \mmd^2(\P,\Q) \geq t_{1-\alpha} + \frac{\sigma_{up}}{\sqrt{m}} \Phi^{-1}(1-\beta)
\end{talign}
Replacing $t_{1-\alpha}$ by its expression in \cref{eq:t_1_alpha_block} and using the upper bound \cref{eq:sigma_1_2_bound}, we get that a necessary condition for \cref{eq:mmd2_cond_1} to hold is
\begin{talign}
    \mmd(\P,\Q) \geq \sqrt{t_{1-\alpha}} = (\frac{4\sum_{l = 1}^{\infty} \lambda_l^2}{B m})^{1/4} \sqrt{\Phi^{-1}(1-\alpha)} = \Omega(1/(Bm)^{1/4}).
\end{talign} 
Also, a sufficient condition for \cref{eq:mmd2_cond_1} is
\begin{talign}
    \mmd^2(\P,\Q) - 2\sqrt{\frac{ \mathbb{E}_z[\hkernel(z,z)^2]}{m}} \Phi^{-1}(1-\beta) \mmd(\P,\Q) - 2\sqrt{\frac{\sum_{l = 1}^{\infty} \lambda_l^2}{B m}} \Phi^{-1}(1-\alpha) \geq 0.
\end{talign}
Solving the corresponding second-degree equation, this is equivalent to
\begin{talign}
    \mmd(\P,\Q) &\geq \sqrt{\frac{ \mathbb{E}_z[\hkernel(z,z)^2]}{m}} \Phi^{-1}(1-\beta) + \sqrt{ \frac{ \mathbb{E}_z[\hkernel(z,z)^2]}{m} \Phi^{-1}(1-\beta)^2 + 2\sqrt{\frac{\sum_{l = 1}^{\infty} \lambda_l^2}{B m}} \Phi^{-1}(1-\alpha)} \\ &= O(1/(B m)^{1/4}).
\end{talign}

\subsection{Proof of \cref{thm:rates_other_tests}\cref{thm:rate_incomplete}}
\citet[Cor.~3]{yamada2018post} show that when the pairs in the design $\mathcal{D}$ are chosen i.i.d.\ (with replacement), and $\lim_{m,|\mathcal{D}| \to \infty} m^{-2} |\mathcal{D}| = 0$, $0 < \gamma = \lim_{m,|\mathcal{D}| \to \infty} m^{-1} |\mathcal{D}| < \infty$, the incomplete MMD statistic $\mmdincomplete^2(\xs,\ys[m])$ is asymptotically distributed according to
\begin{talign}
    \begin{cases}
    |\mathcal{D}|^{1/2} \mmdincomplete^2(\xs,\ys[m]) \rightarrow N(0,\sigma^2) &\text{if } \P=\Q\\
    |\mathcal{D}|^{1/2} (\mmdincomplete^2(\xs,\ys[m]) - \mmd^2(\P,\Q)) \rightarrow N(0,\sigma^2+\gamma \sigma_{up}^2), &\text{if } \P\neq \Q.
    \end{cases}
\end{talign}
where $\sigma^2 = \mathbb{E}_{z,z'}(\hkernel(z,z') - \mathbb{E}_{z,z'}\hkernel(z,z'))^2$ and $\sigma_{up}^2$ is defined in \cref{eq:sigma_1_2_bound}.

We derive the expression for the threshold $t_{1-\alpha}$ corresponding to the level $\alpha$ using the asymptotic distribution under the null hypothesis:
\begin{talign} \label{eq:t_1_alpha_incomplete}
&\mathrm{Pr}(\mmdincomplete^2(\xs,\ys[m]) < t_{1-\alpha}) = \mathrm{Pr}(\frac{|\mathcal{D}|^{1/2}}{\sigma} \mmdincomplete^2(\xs,\ys[m]) < \frac{|\mathcal{D}|^{1/2}}{\sigma} t_{1-\alpha}) \\ &= \Phi(\frac{|\mathcal{D}|^{1/2}}{\sigma} t_{1-\alpha}) = 1-\alpha \iff 
t_{1-\alpha} = \frac{\sigma}{|\mathcal{D}|^{1/2}} \Phi^{-1}(1-\alpha).
\end{talign}
And then, under the alternative hypothesis and asymptotically,
\begin{talign}
    &\mathrm{Pr}(\mmdincomplete^2(\xs,\ys[m]) < t_{1-\alpha}) \\ &= \mathrm{Pr}(\frac{|\mathcal{D}|^{1/2}}{\sqrt{\sigma^2+\gamma^2\sigma^2_{up}}}(\xs,\ys[m])  - \mmd^2(\P,\Q)) < \frac{|\mathcal{D}|^{1/2}}{\sqrt{\sigma^2+\gamma^2\sigma^2_{up}}}(t_{1-\alpha} - \mmd^2(\P,\Q))) \\ &= \Phi(\frac{|\mathcal{D}|^{1/2}}{\sqrt{\sigma^2+\gamma^2\sigma^2_{up}}}(t_{1-\alpha} - \mmd^2(\P,\Q))).
\end{talign}
Proceeding as in \cref{thm:proof_rate_block}, we obtain that a necessary condition for the Type II error to be upper-bounded by $\beta \in (0,1/2)$ is that
\begin{talign}
    \mmd(\P,\Q) \geq \sqrt{t_{1-\alpha}} = \sqrt{\frac{\sigma}{|\mathcal{D}|^{1/2}} \Phi^{-1}(1-\alpha)} = \Omega(1/|\mathcal{D}|^{1/4}),
\end{talign}
and that a sufficient condition is
\begin{talign}
    \mmd^2(\P,\Q) \geq t_{1-\alpha} + \frac{\sigma + 2 \gamma \mathrm{MMD}(\P,\Q) \sqrt{\mathbb{E}_z[\hkernel(z,z)^2]}}{|\mathcal{D}|^{1/2}} \Phi^{-1}(1-\beta)
\end{talign}
In order to derive this, we used that $\sqrt{\sigma^2+4\gamma^2 \mathrm{MMD}^2(\P,\Q) \mathbb{E}_z[\hkernel(z,z)^2] } \leq \sigma + 2 \gamma \mathrm{MMD}(\P,\Q) \sqrt{\mathbb{E}_z[\hkernel(z,z)^2]}$. Solving the corresponding second-degree equation, this is equivalent to
\begin{talign}
    \mmd(\P,\Q) &\geq \frac{ \gamma \sqrt{\mathbb{E}_z[\hkernel(z,z)^2]}}{|\mathcal{D}|^{1/2}} \Phi^{-1}(1-\beta) + \sqrt{ \frac{ \gamma^2 \mathbb{E}_z[\hkernel(z,z)^2]}{|\mathcal{D}|} \Phi^{-1}(1-\beta)^2 +  t_{1-\alpha} + \frac{\sigma \Phi^{-1}(1-\beta)}{|\mathcal{D}|^{1/2}}} \\ &= O(1/|\mathcal{D}|^{1/4}).
\end{talign}
\section{\pcref{lctt_guarantees}}\label{proof-lctt_guarantees}

The proof of \cref{lctt_guarantees} follows the structure of the proof of \cref{thm:uniform_separation}. We first introduce a detailed statement of the result of \cref{lctt_guarantees}.

\begin{theorem} \label{thm:lctt_guarantees_full}
    \lcttname (\cref{algo:lctt}) has size exactly equal to the level $\alpha$ for all $\P$.
Suppose \lcttname (\cref{algo:lctt}) is run with level $\alpha$, replication count $\numperm \geq \frac{1}{\alpha}- 1$, and coreset count %
$\sblock_m \geq (32/9) \log(\frac{2e}{\gamma})$  for
    $\gamma \defeq \frac{\alpha}{4e} (\frac{\beta}{4+2\beta})^{\frac{1}{\floor{\alpha (\numperm\!+\!1)}}}$. Let $\tilde{\beta} = \beta/(1+\beta/2)$.
Then \lctt 
has power 
\begin{talign} \label{eq:power_condition_main}
    \Pr[\deltactt(\mathbb{X}_m,\mathbb{Y}_n)=1] &\geq 1\!-\!\beta 
    \end{talign}
whenever %
\begin{talign} \label{eq:detailed_mmd_lower_bound}
\mmd_{\kernel}(\P,\Q) \!\geq \! 32 \bigg(&\frac{\errorsplit(\P, m/\sblock_{m,r}, \tilde{\beta}
    /6, \ossymb)}{2^{\ossymb} \sqrt{m}} +  \frac{\errorsplit(\Q, n/\sblock_{n,r}, \tilde{\beta}
    /6, \ossymb)}{2^{\ossymb} \sqrt{n}} +c_{\tilde{\beta}
    /6} \bigparenth{\sqrt{\frac{\kinfnorm}{m}}\!+\!\sqrt{\frac{\kinfnorm}{n}}} \!+\!2 %
    \featerr(\xs,\ys) \\ &+ 2^{1+1/4}\big(\sqrt{\frac{9}{32}} + 1\big)\big(2+c' \sqrt{\log(\gamma)} \big) \times 
    \bigg(%
    \featerr(\xs,\ys)+ \frac{\errorsplit(\P, m/\sblock_{m,r}, \tilde{\beta}/(20\sblock_{m,r}), \ossymb
    )}{2^{\ossymb} \sqrt{m}} \\ &\qquad\qquad\qquad\qquad\qquad+ \sqrt{\frac{\sblock_n}{\sblock_{m}}} \frac{\errorsplit(\Q, n/\sblock_{n,r}, \tilde{\beta}/(20\sblock_{n,r}), \ossymb
    )}{2^{\ossymb} \sqrt{n}} \!+ \!c_{\tilde{\beta}/(20\sblock_r)} \big(\sqrt{\frac{\kinfnorm}{m}} +\sqrt{\frac{\sblock_n \kinfnorm}{\sblock_m  n}} \big) \bigg)
    \bigg).
\end{talign}
\end{theorem}

It is important to remark that \cref{algo:lctt} involves two separate parameters: the number of compression bins $\sblock_r := \sblock_{m,r} + \sblock_{n,r}$ and the number of permutation bins $\sblock = \sblock_m + \sblock_n$. The former is always larger or equal than the latter and in particular $\sblock$ divides $\sblock_r$, that is, the compressed outputs of $\sblock_r/\sblock$ compression bins are grouped together into a single permutation bin. %

We formulate a statement which is analogous to \cref{prop:control_power}, but for the \lctt test statistic, and with a slightly different lower bound $Z(m,n,\beta)$.
\begin{lemma}[Upper bound on acceptance probability from upper bound on threshold] \label{prop:control_power_lctt}
Let $1 \geq \beta > 0$ arbitrary, and define $\tilde{\beta} = \frac{\beta}{1+\frac{\beta}{2}}$. Define the function
\begin{talign} \label{eq:Z_def_lctt}
Z(m,n,\beta) %
\defeq 
    \frac{\errorsplit(\P, m/\sblock_{m,r}, \tilde{\beta}%
    /6, \ossymb)}{2^{\ossymb} \sqrt{m}} +  \frac{\errorsplit(\Q, n/\sblock_{n,r}, \tilde{\beta}%
    /6, \ossymb)}{2^{\ossymb} \sqrt{n}} +c_{\tilde{\beta}%
    /6} \bigparenth{\sqrt{\frac{\kinfnorm}{m}}\!+\!\sqrt{\frac{\kinfnorm}{n}}} + 2\featerr(\xs,\ys), %
\end{talign}
which is equal to the upper bound in \cref{eq:mmd_diff_p_q} when we make the choice $\delta = \tilde{\beta}/6$.
If $\mathrm{Pr}\left[\mathrm{MMD}(\P,\Q) \geq
Z(m,n,\beta) +
M_{(b_{\alpha})} \right] \geq \frac{1}{1+\frac{\beta}{2}}$
then %
$\mathrm{Pr}[\mathrm{MMD}_{\feat}(\hat{\mathbb{X}}_m,\hat{\mathbb{Y}}_n) \leq %
M_{(b_{\alpha})}] \leq \beta$. Here, $\errorsplit$ are defined in \cref{sub:inflation_factor}, and $\featerr^2(\P_m, \Q_n) = \sup_{x,x' \in \supp{\P_m} \cup \supp{\Q_m}} |\langle \feat(x), \feat(x') \rangle  - \kernel(x,x')|$.
\end{lemma}
\begin{proof}
    The proof structure is the same as for \cref{prop:control_power}, but in this case we must use instead that with probability at least $1- \frac{\tilde{\beta}}{2}$,
\begin{talign}
    |\mathrm{MMD}(\P,\Q) - \mathrm{MMD}_{\feat}(\hat{\mathbb{X}}_m,\hat{\mathbb{Y}}_n)| &\leq |\mathrm{MMD}(\P,\Q) - \mathrm{MMD}(\hat{\mathbb{X}}_m,\hat{\mathbb{Y}}_n)| \\ &+ |\mathrm{MMD}_{\feat}(\hat{\mathbb{X}}_m,\hat{\mathbb{Y}}_n) - \mathrm{MMD}(\hat{\mathbb{X}}_m,\hat{\mathbb{Y}}_n)|
\end{talign}
The first term in the right-hand side is upper-bounded by 
\begin{align}
\frac{\errorsplit(\P, m/\sblock_{m,r}, \tilde{\beta}%
/6, \ossymb)}{2^{\ossymb} \sqrt{m}} +  \frac{\errorsplit(\Q, n/\sblock_{n,r}, \tilde{\beta}%
/6, \ossymb)}{2^{\ossymb} \sqrt{n}} +c_{\tilde{\beta}%
/6} \bigparenth{\sqrt{\frac{\kinfnorm}{m}}\!+\!\sqrt{\frac{\kinfnorm}{n}}},
\end{align}
while the second term can be upper-bounded as follows, since MMD is nonnegative:
\begin{align}
    &|\mathrm{MMD}_{\feat}(\hat{\mathbb{X}}_m,\hat{\mathbb{Y}}_n) - \mathrm{MMD}(\hat{\mathbb{X}}_m,\hat{\mathbb{Y}}_n)| \\ &\leq \sqrt{|\mathrm{MMD}_{\feat}^2(\hat{\mathbb{X}}_m,\hat{\mathbb{Y}}_n) - \mathrm{MMD}^2(\hat{\mathbb{X}}_m,\hat{\mathbb{Y}}_n)|} \\ &\leq \sqrt{|\langle \feat(\hat{X}_{m}) - \feat(\hat{Y}_{n}), \feat(\hat{X}_{m}) - \feat(\hat{Y}_{n}) \rangle - \langle (\hat{\P}_{m} - \hat{\Q}_{n}) \kernel, (\hat{\P}_{m} - \hat{\Q}_{n}) \kernel \rangle_{\kernel}|}
    \\ &\leq \bigg( |\langle \feat(\hat{X}_{m}), \feat(\hat{X}_{m}) \rangle - \langle \hat{\P}_{m} \kernel, \hat{\P}_{m} \kernel \rangle_{\kernel}| + 2|\langle \feat(\hat{X}_{m}), \feat(\hat{Y}_{n}) \rangle - \langle \hat{\P}_{m} \kernel, \hat{\Q}_{n} \kernel \rangle_{\kernel}| \\ &\qquad\qquad + |\langle \feat(\hat{Y}_{n}), \feat(\hat{Y}_{n}) \rangle - \langle \hat{\Q}_{n} \kernel, \hat{\Q}_{n} \kernel \rangle_{\kernel}| \bigg)^{1/2} \leq 2 \sup_{x, x' \in \xs \cup \ys} \sqrt{|\langle \feat(x), \feat(x') \rangle  - \kernel(x,x')|} = 2 \featerr(\P_m, \Q_n).
    \label{eq:mmd_approx_error}
\end{align}
\end{proof}

The following proposition, an analog of \cref{prop:random_perm_mmd_bound}, establishes a high-probability upper-bound on the MMD obtained by permuting the data samples. 

\begin{lemma}[Tail bound on $\mmd_{\feat}^2(\hatxs^{\sigma},\hatys^{\sigma})$ conditioned on $\hatxs$, 
$\hatys$, and $\Phi_r$] \label{prop:random_perm_mmd_feat_bound}
Let $\sigma$ be a uniformly random permutation over $\{1,\dots,\sblock\}$. Let $\delta, \delta'' \in (0,1)$, and $\delta \in (0,e^{-1})$.  
There is an event $\mathcal{A}$ of probability at least $1-\delta$ concerning the draw of $\hatxs$, 
$\hatys$, such that conditioned on $\mathcal{A}$, 
with probability at least $1-\delta'-\delta''$ on the draw of $\sigma$,
\begin{talign} 
\begin{split}\label{eq:sigma_prop_ctt_rff}
    \mmd^2_{\feat}(\hatxs^{\sigma},\hatys^{\sigma}) &\leq 
    \frac{2}{\sblock_{m}} (\log(\frac{2}{\delta''}) +1) \mmd^2(\P,\Q) \\ &+ 
    \frac{c' \log(1/\delta') + 2}{\sblock_{m}} \big(2 \featerr^4(\P_m, \Q_n) + 2(\mmd(\P,\Q) W(m,n,\delta/(5\sblock)) + W(m,n,\delta/(5\sblock))^2)^2 \big)^{1/2}
\end{split}
\end{talign}
where $\sblock_r = \sblock_{m} + \sblock_{n}$, $c'$ is a universal constant and
\begin{talign} \label{eq:W_def_ctt_rff}
    W(m,n,\delta
    ) = 2\big(\frac{\sqrt{\sblock_{m}} \errorsplit(\P, m/\sblock_{m,r}, \delta, \ossymb
    )}{2^{\ossymb} \sqrt{m}} \!+ \!\frac{\sqrt{\sblock_{n}} \errorsplit(\Q, n/\sblock_{n,r}, \delta, \ossymb
    )}{2^{\ossymb} \sqrt{n}} \!+ \!c_{\delta
    } \big(\sqrt{\frac{\sblock_{m} \kinfnorm}{m}} +\sqrt{\frac{\sblock_{n} \kinfnorm}{n}} \big) \big).
\end{talign}
Here, the \ktcompress error inflation factors $\errorsplit$ and the factor $c_{\delta}$ are defined as in \cref{thm:compression_guarantee}, and $\xi(\P_m, \Q_n) = \sup_{x,x' \in \supp{\P_m} \cup \supp{\Q_m}} |\langle \feat(x), \feat(x') \rangle  - \kernel(x,x')|$.
\end{lemma}
\begin{proof}
    For an arbitrary distribution $\P$, we denote $\feat(\P) = \E_{x \sim \P} \feat(x) \in \R^r$. %
    We can write
\begin{talign} 
\begin{split} \label{eq:MMD_feat_sq_hatx_haty}
    \mmd_{\feat}^2(\hatxs^{\sigma},\hatys^{\sigma}) &= \langle \feat(\hat{\P}_m^{\sigma}) - \feat(\hat{\Q}_n^{\sigma}), \feat(\hat{\P}_m^{\sigma}) - \feat(\hat{\Q}_n^{\sigma}) \rangle \\ &= 
    \frac{1}{\sblock_m^2 \sblock_n^2} \sum_{i=1}^{\sblock_m} \sum_{i'=1}^{\sblock_m} \sum_{j=1}^{\sblock_n} \sum_{j'=1}^{\sblock_n}
    \langle \feat(\hat{\mathbb{S}}_{m+n}^{(\sigma(i))}) - \feat(\hat{\mathbb{S}}_{m+n}^{(\sigma(\sblock_m+j))}), \feat(\hat{\mathbb{S}}_{m+n}^{(\sigma(i'))}) - \feat(\hat{\mathbb{S}}_{m+n}^{(\sigma(\sblock_m+j'))}) \rangle
\end{split}
\end{talign}
Remark that this is the analog of \cref{eq:mmd_perm_decomposition} when one replaces $\langle (\hat{\mathbb{S}}_{m+n}^{(\sigma(i))} - \hat{\mathbb{S}}_{m+n}^{(\sigma(\sblock_m+j))})\kernel, (\hat{\mathbb{S}}_{m+n}^{(\sigma(i'))} - \hat{\mathbb{S}}_{m+n}^{(\sigma(\sblock_m+j'))})\kernel \rangle_{\kernel}$ by $\langle \feat(\hat{\mathbb{S}}_{m+n}^{(\sigma(i))}) - \feat(\hat{\mathbb{S}}_{m+n}^{(\sigma(\sblock_m+j))}), \feat(\hat{\mathbb{S}}_{m+n}^{(\sigma(i'))}) - \feat(\hat{\mathbb{S}}_{m+n}^{(\sigma(\sblock_m+j'))}) \rangle$. Following the analogy and using the construction from \cref{prop:random_perm_mmd_bound}, we will apply the Hanson-Wright inequality (\cref{lem:hanson_wright}) to the matrix $A$ defined as
\begin{talign} \label{eq:hanson_wright_matrix_2}
    A_{i,i'} = 
    \begin{cases}
    0 &\text{if } i = i' \\
    \frac{\langle \feat(\hat{\P}_{m}^{\tilde{\sigma}(i)}) \!-\! \feat(\hat{\Q}_n^{\tilde{\sigma}(\sblock_m + l_i)-\sblock_m}), \feat(\hat{\P}_{m}^{\tilde{\sigma}(i')}) \!-\! \feat(\hat{\Q}_n^{\tilde{\sigma}(\sblock_m + l_{i'})-\sblock_m}) \rangle \!-\!
    \mmd_{\kernel}^2(\P,\Q)}{\sblock_m(\sblock_m-1)}  &\text{if } i \neq i' \in \rho_{\sigma}^{-1}(\{\pm1\}), \\
    \frac{1}{\sblock_m(\sblock_m-1)} \langle \feat(\hat{\P}_{m}^{\tilde{\sigma}(i)}) - \feat(\hat{\Q}_n^{\tilde{\sigma}(\sblock_m+i)-\sblock_m}), \feat(\hat{\mathbb{S}}_{m+n}^{\tilde{\sigma}(i')}) - \feat(\hat{\mathbb{S}}_{m+n}^{\tilde{\sigma}(\sblock_m+l_{i'})}) \rangle &\text{if } i \in \rho_{\sigma}^{-1}(\{\pm1\}), i' \in \rho_{\sigma}^{-1}(\{0\}), \\
    \frac{1}{\sblock_m(\sblock_m-1)} \langle \feat(\hat{\mathbb{S}}_{m+n}^{\tilde{\sigma}(i)}) - \feat(\hat{\mathbb{S}}_{m+n}^{\tilde{\sigma}(\sblock_m+l_i)}), \feat(\hat{\P}_{m}^{\tilde{\sigma}(i')}) - \feat(\hat{\Q}_n^{\tilde{\sigma}(\sblock_m+l_{i'})-\sblock_m}) \rangle &\text{if } i \in \rho_{\sigma}^{-1}(\{0\}), i' \in \rho_{\sigma}^{-1}(\{\pm1\}), \\
    \frac{1}{\sblock_m(\sblock_m-1)} \langle \feat(\hat{\mathbb{S}}_{m+n}^{\tilde{\sigma}(i)}) - \feat(\hat{\mathbb{S}}_{m+n}^{\tilde{\sigma}(\sblock_m+l_i)}), \feat(\hat{\mathbb{S}}_{m+n}^{\tilde{\sigma}(l_{i'})}) - \feat(\hat{\mathbb{S}}_{m+n}^{\tilde{\sigma}(\sblock_m+l_{i'})}) \rangle &\text{if } i \neq i' \in \rho_{\sigma}^{-1}(\{0\}).
    \end{cases}
\end{talign}
which is the analog of the matrix defined in \cref{eq:hanson_wright_matrix}. Note that
\begin{talign} \label{eq:error_split_lctt}
    &|\langle \feat(\hat{\P}_{m}^{\tilde{\sigma}(i)}) \!-\! \feat(\hat{\Q}_n^{\tilde{\sigma}(\sblock_m + l_i)-\sblock_m}), \feat(\hat{\P}_{m}^{\tilde{\sigma}(i')}) \!-\! \feat(\hat{\Q}_n^{\tilde{\sigma}(\sblock_m + l_{i'})-\sblock_m}) \rangle \!-\!
    \mmd_{\kernel}^2(\P,\Q)| \\ &\leq |\langle \feat(\hat{\P}_{m}^{\tilde{\sigma}(i)}) \!-\! \feat(\hat{\Q}_n^{\tilde{\sigma}(\sblock_m + l_i)-\sblock_m}), \feat(\hat{\P}_{m}^{\tilde{\sigma}(i')}) \!-\! \feat(\hat{\Q}_n^{\tilde{\sigma}(\sblock_m + l_{i'})-\sblock_m}) \rangle \\ &\qquad\qquad - \langle (\hat{\P}_{m}^{\tilde{\sigma}(i)} - \hat{\Q}_n^{\tilde{\sigma}(\sblock_m+l_i)-\sblock_m})\kernel, (\hat{\P}_{m}^{\tilde{\sigma}(i')} - \hat{\Q}_n^{\tilde{\sigma}(\sblock_m+i')-\sblock_m})\kernel \rangle_{\kernel}| \\ &\qquad + |\langle (\hat{\P}_{m}^{\tilde{\sigma}(i)} - \hat{\Q}_n^{\tilde{\sigma}(\sblock_m+l_i)-\sblock_m})\kernel, (\hat{\P}_{m}^{\tilde{\sigma}(i')} - \hat{\Q}_n^{\tilde{\sigma}(\sblock_m+i')-\sblock_m})\kernel \rangle_{\kernel} - \langle (\P - \Q)\kernel, (\P - \Q)\kernel \rangle_{\kernel}|.
\end{talign}
The last term in the right-hand side is formally the same as the one upper-bounded in \cref{eq:bound_by_mmd_1}. However, there is a difference: in this case $\hat{\P}_{m}^{\tilde{\sigma}(i)}$ is not the empirical distribution of the output of \ktcompress on the points of $\xs^{(\tilde{\sigma}(i))}$, but rather the distribution corresponding to the concatenation of the outputs of \ktcompress on the $\sblock_r/\sblock$ compression bins contained in $\xs^{(\tilde{\sigma}(i))}$. %
As a result, an adaptation of the argument yields
\begin{align}
    &|\langle \feat(\hat{\P}_{m}^{\tilde{\sigma}(i)}) \!-\! \feat(\hat{\Q}_n^{\tilde{\sigma}(\sblock_m + l_i)-\sblock_m}), \feat(\hat{\P}_{m}^{\tilde{\sigma}(i')}) \!-\! \feat(\hat{\Q}_n^{\tilde{\sigma}(\sblock_m + l_{i'})-\sblock_m}) \rangle \!-\!
    \mmd_{\kernel}^2(\P,\Q)| \\ &\leq \mmd(\P,\Q) W(m,n,\delta/(5\sblock_r)) + W(m,n,\delta/(5\sblock_r))^2,
\end{align}
where the function $W$ defined in \cref{eq:W_def_ctt_rff} is slightly different from the one in \cref{prop:random_perm_mmd_bound} in that the arguments of the error inflation factors are $m/\sblock_{m,r}$ and $n/\sblock_{n,r}$ instead of $m/\sblock_{m}$ and $n/\sblock_{n}$.

We can bound the first term of \cref{eq:error_split_lctt} by
\begin{talign}
    &|\langle \feat(\hat{\P}_{m}^{\tilde{\sigma}(i)}), \feat(\hat{\P}_{m}^{\tilde{\sigma}(i')}) \rangle - \langle \hat{\P}_{m}^{\tilde{\sigma}(i)} \kernel, \hat{\P}_{m}^{\tilde{\sigma}(i')} \kernel \rangle_{\kernel}| + |\langle \feat(\hat{\P}_{m}^{\tilde{\sigma}(i)}), \feat(\hat{\Q}_n^{\tilde{\sigma}(\sblock_m + l_{i'})-\sblock_m}) \rangle - \langle \hat{\P}_{m}^{\tilde{\sigma}(i)} \kernel, \hat{\Q}_n^{\tilde{\sigma}(\sblock_m + l_{i'})-\sblock_m} \kernel \rangle_{\kernel}| \\ &+ |\langle \feat(\hat{\P}_{m}^{\tilde{\sigma}(i')}), \feat(\hat{\Q}_n^{\tilde{\sigma}(\sblock_m + l_{i})-\sblock_m}) \rangle - \langle \hat{\P}_{m}^{\tilde{\sigma}(i')} \kernel, \hat{\Q}_n^{\tilde{\sigma}(\sblock_m + l_{i})-\sblock_m} \kernel \rangle_{\kernel}| \\ &+ |\langle \feat(\hat{\Q}_n^{\tilde{\sigma}(\sblock_m + l_i)-\sblock_m}), \feat(\hat{\Q}_n^{\tilde{\sigma}(\sblock_m + l_{i'})-\sblock_m}) \rangle - \langle \hat{\Q}_n^{\tilde{\sigma}(\sblock_m + l_i)-\sblock_m} \kernel, \hat{\Q}_n^{\tilde{\sigma}(\sblock_m + l_{i'})-\sblock_m} \kernel \rangle_{\kernel}| \\ &\leq 4 \sup_{x,x' \in \supp{\P_m} \cup \supp{\Q_n}} |\langle \feat(x), \feat(x') \rangle  - \kernel(x,x')|.
\end{talign}
The bound in the right-hand side follows from applying \cref{lem:diff_feat}. The other cases in \cref{eq:hanson_wright_matrix_2} admit similar upper-bounds which in this case rely on \cref{eq:bound_by_mmd_2}. In analogy with \cref{eq:A_norm_bound}, we obtain that
\begin{talign}
\|A\|_{\text{op}}^2 &\leq \|A\|_{\text{F}}^2 \leq \frac{1}{\sblock_m^2(\sblock_m - 1)^2} \big( |\rho_{\sigma}^{-1}(\{-1,1\})| (|\rho_{\sigma}^{-1}(\{-1,1\})|-1) \big(2 \xi^2(\P_m, \Q_n) \\ &\qquad\qquad\qquad\qquad\qquad\qquad + 2 (\mmd(\P,\Q) W(m,n,\delta/(5\sblock_r)) + W(m,n,\delta/(5\sblock_r))^2)^2 \big) \\ &\qquad\qquad\qquad\qquad\quad + |\rho_{\sigma}^{-1}(\{-1,1\})||\rho_{\sigma}^{-1}(\{0\})| (2\xi^2(\P_m, \Q_n) \\ &\qquad\qquad\qquad\qquad\qquad\qquad + 2(\mmd(\P,\Q) W(m,n,\delta/(5\sblock_r)) + \frac{1}{2} W(m,n,\delta/(5\sblock_r))^2)^2) \\ &\qquad\qquad\qquad\qquad\quad + |\rho_{\sigma}^{-1}(\{0\})| (|\rho_{\sigma}^{-1}(\{0\})| -1) W(m,n,\delta/(5\sblock_r))^4 \big) \\ &\leq \eta(m,n,\delta)^2 \defeq \frac{1}{\sblock_m (\sblock_m - 1)} \big(2 \featerr^4(\P_m, \Q_n) + 2(\mmd(\P,\Q) W(m,n,\delta/(5\sblock_r)) + W(m,n,\delta/(5\sblock_r))^2)^2 \big) \label{eq:A_norm_bound_2}
\end{talign}
To prove this inequality, we used that $(a + b)^2 \leq 2 a^2 + 2 b^2$ for any $a,b \geq 0$.

Mirroring the proof of \cref{lem:mathcal_M_bound}, we establish that conditioned on the event $\mathcal{A}$ defined in \cref{lem:ktcompress_simultaneous}, which has probability at least $1-\delta$, we have that with probability at least $1-\delta'-\delta''$ on the choice of $\sigma$,
\begin{talign} 
\begin{split}\label{eq:sigma_prop_2}
    \E_L[\mathcal{M}^{\sigma,L}] &\leq 
    \frac{2}{\sblock_m - 1} \log(\frac{2}{\delta''}) \mmd^2(\P,\Q) \\ &+ \frac{c' \log(1/\delta')}{\sqrt{\sblock_m(\sblock_m-1)}} (2 \featerr^4(\P_m, \Q_n) + 2(\mmd(\P,\Q) W(m,n,\delta/(5\sblock_r)) + W(m,n,\delta/(5\sblock_r))^2)^2 )^{1/2}
\end{split}
\end{talign}
and 
\begin{talign} 
\begin{split} \label{eq:other_terms_prop_2}
    &\frac{1}{\sblock_m^2 \sblock_n^2} \big( \sum_{i=1}^{\sblock_m} \sum_{j\neq j'\in \{1,\dots,\sblock_n\}}
    \langle (\hat{\mathbb{S}}_{m+n}^{i} - \hat{\mathbb{S}}_{m+n}^{\sigma(\sblock_m+j)})\kernel, (\hat{\mathbb{S}}_{m+n}^{i} - \hat{\mathbb{S}}_{m+n}^{\sigma(\sblock_m+j')})\kernel \rangle_{\kernel} \\ &+ \sum_{j=1}^{\sblock_n} \sum_{i\neq i'\in \{1,\dots,\sblock_m\}}
    \langle (\hat{\mathbb{S}}_{m+n}^{\sigma(i)} - \hat{\mathbb{S}}_{m+n}^{\sigma(\sblock_m+j)})\kernel, (\hat{\mathbb{S}}_{m+n}^{\sigma(i')} - \hat{\mathbb{S}}_{m+n}^{\sigma(\sblock_m+j)})\kernel \rangle_{\kernel} 
    \\ &+ \sum_{i=1}^{\sblock_m} \sum_{j=1}^{\sblock_n} \langle (\hat{\mathbb{S}}_{m+n}^{\sigma(i)} - \hat{\mathbb{S}}_{m+n}^{\sigma(\sblock_m+j)})\kernel, (\hat{\mathbb{S}}_{m+n}^{\sigma(i)} - \hat{\mathbb{S}}_{m+n}^{\sigma(\sblock_m+j)})\kernel \rangle_{\kernel} \big) \\ &\leq %
    \frac{\sblock - 1}{\sblock_m \sblock_n}
    ((2 \featerr^4(\P_m, \Q_n) + 2(\mmd(\P,\Q) W(m,n,\delta/(5\sblock_r)) + W(m,n,\delta/(5\sblock_r))^2)^2 )^{1/2} + \mmd^2(\P,\Q))
\end{split}
\end{talign}
simultaneously, where $c'$ is a universal constant and $W(m,n,\delta)$ is as defined in \cref{eq:W_def}.

Mirroring the final step of the proof of \cref{prop:random_perm_mmd_bound}, we rely on these two equations to show that conditioned on the event $\mathcal{A}$, with probability at least $1-\delta'-\delta''$,
\begin{talign}
    &\mmd_{\feat}^2(\hatxs^{\sigma},\hatys^{\sigma}) \leq \frac{(\sblock_m-1)(\sblock_n-1)}{\sblock_m \sblock_n} \bigg( \frac{2}{\sblock_m - 1} \log(\frac{2}{\delta''}) \mmd^2(\P,\Q) \\ &\qquad\qquad\qquad\qquad\qquad+ \frac{c' \log(1/\delta')}{\sqrt{\sblock_m(\sblock_m-1)}} (2 \featerr^4(\P_m, \Q_n) + 2(\mmd(\P,\Q) W(m,n,\delta/(5\sblock_r)) + W(m,n,\delta/(5\sblock_r))^2)^2 )^{1/2} \bigg) \\ &\qquad\qquad\quad+ \frac{\sblock - 1}{\sblock_m \sblock_n}
    ((2 \featerr^4(\P_m, \Q_n) + 2(\mmd(\P,\Q) W(m,n,\delta/(5\sblock_r)) + W(m,n,\delta/(5\sblock_r))^2)^2 )^{1/2} + \mmd^2(\P,\Q)),
\end{talign}
which concludes the proof upon simplification.

\end{proof}

\begin{lemma} \label{lem:diff_feat}
Let $\mathbb{S}$ and $\mathbb{S}'$ be arbitrary distributions. We have that
\begin{align} \label{eq:diff_feat}
    |\langle \feat(\mathbb{S}), \feat(\mathbb{S}') \rangle - \langle \mathbb{S} \kernel, \mathbb{S}' \kernel \rangle_{\kernel}| \leq \sup_{x \in \supp{\mathbb{S}}, x' \in \supp{\mathbb{S}'}} |\langle \feat(x), \feat(x') \rangle  - \kernel(x,x')|.
\end{align}
\end{lemma}
\begin{proof}
    The right-hand side of \cref{eq:diff_feat} is equal to 
\begin{align}
    &|\mathbb{E}_{x \sim \mathbb{S}, x' \sim \mathbb{S}'} \langle \feat(x), \feat(x') \rangle - \mathbb{E}_{x \sim \mathbb{S}, x' \sim \mathbb{S}'} \kernel(x,x')| \leq \mathbb{E}_{x \sim \mathbb{S}, x' \sim \mathbb{S}'} |\langle \feat(x), \feat(x') \rangle  - \kernel(x,x')| \\ &\leq \sup_{x \in \supp{\mathbb{S}}, x' \in \supp{\mathbb{S}'}} |\langle \feat(x), \feat(x') \rangle  - \kernel(x,x')|.
\end{align}
\end{proof}

We proceed to prove \cref{thm:lctt_guarantees_full}. Reproducing the argument of \cref{prop:bound_q_arbitrary_B}, if $b_{\alpha} \defeq \ceil{(\numperm+1)(1-\alpha)} \leq \numperm$, and \ktcompress calls are run with value $\delta^{*}/(5\sblock_{m})$ and $\delta^{*}/(5\sblock_{n})$ respectively, then with probability at least $1-\frac{\delta}{2}$,
    \begin{talign} 
    \begin{split} \label{eq:hat_Z_def_lctt}
    M_{(b_{\alpha})}
    &\leq \hat{Z}(m,n,\alpha,\delta) 
    \\ &\defeq 
    \sqrt{\frac{2}{\sblock_{m}} (\log(\frac{2}{\delta^*}) + 1)} \mmd(\P,\Q) + \sqrt{\frac{1}{\sblock_{m}}(2+c' \log(1/\delta^*))} \cdot 2^{1/4} (\featerr(\P_m, \Q_n) \\ &\qquad\qquad\qquad + \sqrt{\mmd(\P,\Q) W(m,n,\delta/(20\sblock_r))} + W(m,n,\delta/(20\sblock_r))).
    \end{split}
    \end{talign}
    where $M_{(b_{\alpha})}$ is the threshold of \cref{algo:lctt}, and 
    \begin{talign}
    \delta^{*} \defeq %
    (\frac{\delta}{2})^{1/k_{\alpha}} \frac{\alpha}{4e}, \quad k_{\alpha} \defeq \floor{\alpha(\numperm+1)}.
    \end{talign}

Then, we have that the statement in \cref{thm:uniform_separation_details} holds with the redefined $\hat{Z}(m,n,\alpha,\tilde{\beta})$ given by \cref{eq:hat_Z_def_lctt}. Consequently, the function $\tilde{Z}(m,n,\alpha,\beta)$ now reads
\begin{talign}
    \tilde{Z}(m,n,\alpha,\beta) &= 
    \frac{\errorsplit(\P, m/\sblock_{m,r}, \tilde{\beta}%
    /6, \ossymb)}{2^{\ossymb} \sqrt{m}} +  \frac{\errorsplit(\Q, n/\sblock_{n,r}, \tilde{\beta}%
    /6, \ossymb)}{2^{\ossymb} \sqrt{n}} +c_{\tilde{\beta}%
    /6} \bigparenth{\sqrt{\frac{\kinfnorm}{m}}\!+\!\sqrt{\frac{\kinfnorm}{n}}} 
    \\ &\qquad + \sqrt{\frac{2}{\sblock_m} \log(\frac{8e^2}{\alpha} (\frac{4}{\tilde{\beta}})^{1/k_{\alpha}})} \mmd(\P,\Q) \\ &\qquad + \sqrt{\frac{\sqrt{2}}{\sblock_m}(2+c' \log(\frac{4e}{\alpha} (\frac{4}{\tilde{\beta}})^{1/k_{\alpha}}))} (\featerr(\P_m, \Q_n) + \sqrt{\mmd(\P,\Q) W(m,n,\tilde{\beta}/(20\sblock_r))} \\ &\qquad\qquad\qquad\qquad\qquad\qquad\qquad + W(m,n,\tilde{\beta}/(20\sblock_r))).
\end{talign}
We also reduce the remainder of the proof to a second-degree inequality analogous to \cref{eq:inequality_coefficients}, but in this case the coefficients read
\begin{talign} \label{eq:inequality_coefficients_lctt}
    a &= %
    1 - \sqrt{\frac{2}{\sblock_m} \log(\frac{8e^2}{\alpha} (\frac{4}{\tilde{\beta}})^{1/k_{\alpha}})}, \\
    b &= \sqrt{\frac{\sqrt{2}}{\sblock_m}(2+c' \log(\frac{4e}{\alpha} (\frac{2}{\tilde{\beta}})^{1/k_{\alpha}})) W(m,n,\tilde{\beta}/(20\sblock_r))}, 
    \\
    c &= \frac{\errorsplit(\P, m/\sblock_{m,r}, \tilde{\beta}%
    /6, \ossymb)}{2^{\ossymb} \sqrt{m}} +  \frac{\errorsplit(\Q, n/\sblock_{n,r}, \tilde{\beta}%
    /6, \ossymb)}{2^{\ossymb} \sqrt{n}} +c_{\tilde{\beta}%
    /6} \bigparenth{\sqrt{\frac{\kinfnorm}{m}}\!+\!\sqrt{\frac{\kinfnorm}{n}}} \\ &\qquad + \sqrt{\frac{\sqrt{2}}{\sblock_m}(2+c' \log(\frac{4e}{\alpha} (\frac{4}{\tilde{\beta}})^{1/k_{\alpha}}))} (\featerr(\P_m, \Q_n) + W(m,n,\tilde{\beta}/(20\sblock_r))), \\
    x &= \sqrt{\mmd(\P,\Q)},
\end{talign}
Proceeding just like in the case of CTT, the proof is concluded (recall that the function $W$ is slightly different in this case).
\section{\pcref{thm:agg_validity_power}}
\label{sec:proof_agg_validity_power}
\acttalg

The validity statement in \cref{eq:agg_validity} follows from exactly the same argument as \citet[Prop.~8]{schrab2021mmd}, replacing the estimate $\mmd_{up}(\xs,\ys)$ with parameter $\lambda$ by $\tmmd(\xs,\ys)$ with parameter $\lambda$.  

Let $\Delta_{\ctt,\lambda}$ denote
the output of a modified $\ctt$ (\cref{algo:ctt}) with level $\alpha w_{\lambda}/2$,  $\numperm = \numperm_1$, $\kernel\!=\! \sum_{\lam\in\Lam}\kernel_{\lam}$ and $\kersplit \!=\! \sum_{\lam\in\Lam} \kersplit[\lam]$ that uses $\klam$ (in place of $\kernel$) to compute \tmmd. Then using arguments from \citet[Proof of Thm.~9, up to their equation 25]{schrab2021mmd}, we find that
\begin{talign}
    \Pr[\deltaactt(\xs, \ys) = 1] \geq
     \max_{\lambda \in \Lambda} \Pr[\Delta_{\ctt,\lambda}(\xs,\!\ys)\!=\!1]
     - \textfrac{\beta}{2}.
      \label{eq:agg_power_int}
\end{talign}
We claim that for $\lam$ such that \cref{eq:mmd_agg_power} holds, we have $ \Pr[\Delta_{\ctt,\lam}(\xs, \ys) = 1]\geq1-\frac\beta2$, which when put together with \cref{eq:agg_power_int} immediately implies the claimed power in \cref{eq:agg_power}. We now establish our power claim for this modified \ctt test.

To do so, we claim that \ktcompress with $\kernel$ and $\kersplit=\sum_{\lam\in\Lam}\kslam$---referred to as \ktagg---is $\klam$-sub-Gaussian \citep[Def.~3]{shetty2022distribution} with parameters $\shiftparam_{\l, n}''$ and $\kgaussparam_{\l, n}''$ (the analog of $(\shiftparam_{\l,n}, \kgaussparam_{\l,n})$ from \cref{eq:subgaussian_params} in our notation) simultaneously for all $\lam \in \Lam$, on an event of probability $1-\delta/2$, where 
\begin{talign}
\label{eq:kt_agg_params}
  \shiftparam_{\l, n}'' &= \frac{4}{\l}\sqrt{\sum_{\lam\in\Lam}\parenth{\ackk^2(\inputcoreset)+ \amkk^2(\inputcoreset,\delta,\l)\cdot \log|\Lam|}},
 \\
    \kgaussparam_{\l, n}'' &= \frac{2}{\l}\sqrt{\log( 
    \frac{12n4^{\ossymb} (\beta_n+1)}{\l\delta})  \cdot \sum_{\lam\in\Lam}\amkk^2(\inputcoreset,\delta,\l)},
\end{talign}
and $\ckk$ and $\mkk$ were defined in \cref{eq:subgaussian_params}.
Deferring the proof of this claim to the end of this section, we proceed with the proof. 

Using \cref{eq:kt_agg_params} and repeating the arguments from the proof of \cref{thm:compression_guarantee} (after \cref{eq:subgaussian_params}), we conclude the following analog of \cref{eq:mmd_diff_x_y} for the \tmmd estimate with the output $\hatxs, \hatys$ of \ktagg: With probability at least $1-\delta$, we have
\begin{talign} \label{eq:tilde_sigma_def_agg}
&|\mmd_{\klam}( \xs, \ys ) - \mmd_{\klam}( \hatxs, \hatys )|^2 \\ &\leq 
    1024 (\sqrt{\log(m+n+1)} + \sqrt{\log(2/\delta)})^2 \\ 
      &\qquad \cdot |\Lam| \Bigg[\frac{(\log_4 (m/\sblock_m) \!-\!\ossymb\!-\!1) }{4^{\ossymb}m} \parenth{C_{\Lam,\kersplit}(\xs) + \parenth{\sqrt{\log |\Lam|}+ \sqrt{\log( \frac{3 m (\log_4 (m/\sblock_m) \!-\!\ossymb\!-\!1)}{\sblock_m \delta})}} \mathfrak{M}'_{\Lam,\kersplit}(\xs, \delta, 2^{\ossymb+1}\sqrt{m/\sblock_m})}^2 \\ 
     &\qquad\qquad+ \frac{(\log_4 (n/\sblock_n) \!-\!\ossymb\!-\!1) }{4^{\ossymb}n} \parenth{C_{\Lam,\kersplit}(\ys) + \parenth{\sqrt{\log |\Lam|}+ \sqrt{\log( \frac{3 n (\log_4 (n/\sblock_n) \!-\!\ossymb\!-\!1)}{\sblock_n \delta})}} \mathfrak{M}'_{\Lam,\kersplit}(\ys, \delta, 2^{\ossymb+1}\sqrt{n/\sblock_n})}^2 \Bigg],
\end{talign}
where $C_{\Lam,\kersplit}(\inputcoreset)\defeq \max_{\lam\in\Lam} \ackk(\inputcoreset)$ and
    $\mathfrak{M}'_{\Lam,\kernel'}(\inputcoreset, \delta, \l) \defeq 
    \max_{\lam\in\Lam} \amkk(\inputcoreset,\delta, \l)$.
Putting \cref{eq:tilde_sigma_def_agg} together with the definitions~\cref{eq:errorsplit_def}, we find that
\begin{talign}
    |\mmd_{\klam}( \xs, \ys ) - \mmd_{\klam}( \hatxs, \hatys )| \leq
    \underbrace{2\sqrt{|\Lam|(1+\log(|\Lam|))}}_{=c_{\Lam}}
    \cdot  \max_{\lam\in\Lam} 
    \parenth{\frac{\error[{\klam,\kersplit}](\xs, \frac{m}{\sblock_m}, \delta, \ossymb)}{4^{\ossymb} m}+  \frac{\error[{\klam,\kersplit}](\ys, \frac{n}{\sblock_n}, \delta, \ossymb)}{2^{\ossymb} \sqrt{n}} }.
\end{talign}
with probability at least $1-\delta$.
Propagating this result further in the proof of \cref{thm:compression_guarantee} 
implies the following analog of \cref{eq:mmd_diff_p_q}:
\begin{talign} \label{eq:eps_Lambda}
    |\mmd_{\klam}(\P, \Q)  -  \mmd_{\klam}( \hatxs, \hatys ) | 
    &\leq
    c_{\delta} \bigparenth{\sqrt{\frac{\sinfnorm{\klam}}{m}}\!+\!\sqrt{\frac{\sinfnorm{\klam}}{n}}} 
    \\ 
    &\qquad+ c_{\Lam}
    \max_{\lam\in \Lam } 
    \parenth{\frac{\error[{\kersplit[\lam]}](\P, \frac{m}{\sblock_m}, \delta, \ossymb)}{2^{\ossymb} \sqrt{m}} \!+\!  \frac{\error[{\kersplit[\lam]}](\Q, \frac{n}{\sblock_n}, \delta, \ossymb)}{2^{\ossymb} \sqrt{n}} }, 
    \label{eq:mmd_diff_p_q_agg}
\end{talign}
with probability at least $1\!-\!3\delta$ with  $c_{\delta}\!\defeq\! 2\!+\!\sqrt{2\log(\frac{2}{\delta}})$ as in \cref{thm:compression_guarantee}.

We now apply \cref{thm:uniform_separation} to characterize the power of the modified \ctt (corresponding to $\Delta_{\ctt,\lambda}$) described above. In particular, substituting $\alpha \gets \frac{\alpha w_{\lam}}{2}$, $\beta\gets \frac{\beta}{2}$, 
$\gamma \gets \gamma_{\lam}$ 
in \cref{thm:uniform_separation}, noting $\sblock_m \geq \frac{32}{9}\log(\frac{2e}{\gamma_{\lam}})$, and using the definition of $\epsagg$ along with \cref{eq:mmd_diff_p_q_agg}  in the proof of \cref{thm:uniform_separation}, we conclude that
\begin{talign}
    \Pr[\Delta_{\ctt,\lambda}(\xs, \ys)] \geq 1-\frac{\beta}{2}
    \qtext{whenever}
    \mmd_{\klam}(\P, \Q) \geq  c'\sqrt{\log(\frac{1}{\gamma_{\lam}})}\,
    \epsagg(\frac{\beta/(10\sblock)}{4+\beta})
\end{talign}
for some universal constant $c'$, and yielding the desired claim when $m\leq n$. It remains to prove our earlier claim~\cref{eq:kt_agg_params}.

\paragraph{Proof of \cref{eq:kt_agg_params}}
Note that, for any probability measures $(\P', \Q')$,
\begin{align}
\label{eq:klam_k_mmd}
    \mmd_{\klam}(\P', \Q') \leq  \mmd_{\kernel}(\P', \Q')
\end{align}
whenever the right-hand side is  well-defined, since for any two kernels $\kernel_1, \kernel_2$ with well-defined $\mmd_{\kernel_1+\kernel_2}(\P', \Q')$, we have
\begin{align}
    \mmd_{\kernel_1+\kernel_2}^2(\P', \Q') = (\P'-\Q')(\kernel_1+\kernel_2)(\P'-\Q') 
    = \mmd_{\kernel_1}^2(\P',\Q') + \mmd_{\kernel_2}^2(\P',\Q').
    \label{eq:mmd_k_sum}
\end{align}
In the terminology of \citet[Def.~3]{shetty2022distribution}, we next establish that the halving algorithm $\kt(\delta)$ \citep[Ex.~2]{shetty2022distribution}  underlying \ktcompress is $\klam$-sub-Gaussian when run with $\kernel$ and split kernel $\kersplit$.
To proceed, we can suitably adapt the proof of Thm.~4 of  \citet{dwivedi2022generalized} (which in turn is an adaptation of \citet[Thm.~2-4]{dwivedi2021kernel}). %

We begin by instantiating the notation of \citet{dwivedi2022generalized}.
Given an input coreset $\inputcoreset$, let $\cset_{\trm{split}, 1}$ denote the first coreset output by the \ktsplit step and $\cset_{\trm{out}}$ denote the output of size $\nout$ after the \ktswap step. Then using \cref{eq:klam_k_mmd} and the definition of \ktswap, we have
\begin{talign}
    \mmd_{\klam}^2(\inputcoreset, \cset_{\trm{out}}) \leq
    \mmd_{\kernel}^2(\inputcoreset, \cset_{\trm{out}}) 
    \sless{(i)} \mmd_{\kernel}^2(\inputcoreset, \cset_{\trm{split}, 1})
    \seq{\cref{eq:mmd_k_sum}} \sum_{\lam\in\Lam} \mmd_{\klam}^2(\inputcoreset, \cset_{\trm{split}, 1}),
    \label{eq:mmd_klam_sequence}
\end{talign}
where the inequality~(i) follows directly from the definition of \ktswap \citep[Eqn.~27]{dwivedi2022generalized}.
Hence it remains to show that $\ktsplit(\delta)$ is $\klam$-sub-Gaussian for each $\lam$.

To proceed, we modify the \citet[Proof of Thm.~4]{dwivedi2022generalized}. In particular, replacing $\kernel^{\dagger}$ (in their notation) with $\kersplit$, and $\sinfnorm{\kernel^{\dagger}}$ with $\kagginf$ throughout their proof\footnote{The remark in \citet[Footnote 5]{dwivedi2022generalized} implies that the arguments work both with $\sinfnorm{\kersplit}$ and $\kinfsin[\kersplit]$.} 
we conclude, with analogy to \citet[Ex.~2]{shetty2022distribution},  that $\ktsplit(\delta)$ is $\klam$-sub-Gaussian with parameters $v_{\lam,\l}$ and $a_{\lam,\l}$ satisfying
\begin{talign}
    \label{eq:kt_params}
    a_{\lam,\l} = \frac{\ackk(\inputcoreset)}{\nout}
    \qtext{and}
    v_{\lam, \l} = \frac{\amkk(\inputcoreset, \delta, \l)}{\nout}\sqrt{\log(\frac{6\nout\log_2(\l/\nout)}{\delta})},
\end{talign}
for input (a subset of $\inputcoreset$) of size $\l$ and output of size $\nout$,
for $\ackk$ and $\amkk$ defined in \cref{eq:subgaussian_params} (also see \cref{rem:aggk}).

Next, we use an auxiliary result proven at the end of this section.
\begin{lemma}[Tail bounds for sum of non-centered sub-Gaussian random variables]
\label{lem:sum_tail_bound}
    Consider non-negative random variables $Z_1, \ldots, Z_\l$ such that for $i \in [\l]$, we have $\P\brackets{Z_i \geq a_i + v_i\sqrt{t}} \leq e^{-t}$ for all $t\geq 0$, with
    some suitable scalars $\sbraces{a_i, v_i}_{i=1}^{\l}$.
    Then $\P\brackets{\sqrt{\sum_{i=1}^{\l}Z_i^2} 
        \geq \wtil{\alpha} +\wtil{\beta} \sqrt{t}} \leq e^{-t}$ for $t\geq 0$,
    where 
    \begin{talign}
        \wtil{\alpha}^2 \defeq 2\sum_{i=1}^{\l} (a_i^2 + v_i^2 \log \l )
        \leq 4\l \log \l \cdot \max_{i}\max\sbraces{a_i^2, v_i^2}
        \qtext{and}
        \wtil{\beta}^2 \defeq \sum_{i=1}^{\l} v_i^2 \leq \l \max_{i} v_i^2.
    \end{talign}
\end{lemma}
Putting \cref{lem:sum_tail_bound} together with \cref{eq:mmd_klam_sequence,eq:kt_params}, we conclude that on an event of probability at least $1-\delta/2$ and simultaneously for all $\lam \in \Lam$, 
\ktagg ($\kt(\delta)$ with aggregated kernels as above) is $\klam$-sub-Gaussian with parameters $(a'_{\l}, v'_{\l})$ given by
\begin{talign}
    a_{\l}' &= \frac{1}{\nout} \sqrt{2\sum_{\lam\in\Lam}\parenth{\ackk^2(\inputcoreset)+ \amkk^2(\inputcoreset,\delta, \l)\cdot \log|\Lam|}} 
    \qtext{and} \\
    v_{\l}' &= \frac{1}{\nout}\sqrt{\log(\frac{6\nout\log_2(\l/\nout)}{\delta})} \cdot \sqrt{\sum_{\lam\in\Lam}\amkk^2(\inputcoreset,\delta, \l)}.
\end{talign}
for input (a subset of $\inputcoreset$) of size $\l$ and output of size $\nout$.
Now the arguments of \citet[Ex.~4]{shetty2022distribution} imply that on an event of probability at least $1-\delta/2$, every \halve call invoked by \compress (for \ktagg) is $\klam$-sub-Gaussian with parameters $\shiftparam_{\l, n}''$ and $\kgaussparam_{\l, n}''$ (the analog of $(\shiftparam_{\l,n}, \kgaussparam_{\l,n})$ in our notation~\cref{eq:subgaussian_params}), where
\begin{talign}
  \shiftparam_{\l, n}'' &= \frac{2\sqrt{2}}{\l}\sqrt{2\sum_{\lam\in\Lam}\parenth{\ackk^2(\inputcoreset)+ \amkk^2(\inputcoreset,\delta, \l)\cdot \log|\Lam|}}
\qtext{and} \\
    \kgaussparam_{\l, n}'' &= \frac{2}{\l}\sqrt{\log( 
    \frac{12n4^{\ossymb} (\beta_n+1)}{\l\delta})  \cdot \sum_{\lam\in\Lam}\amkk^2(\inputcoreset,\delta, \l)},
\end{talign}
as claimed in \cref{eq:kt_agg_params}.

\paragraph{Proof of \cref{lem:sum_tail_bound}}
Collect the scalars $\sbraces{a_i}$ (resp. $\sbraces{v_i}$) into vector $a \in \real^{\l}$ (resp. $v\in \real^\l$) such that the i-th coordinate of $a$ (resp. $v$) is equal to $a_i$ (resp. $v_i$).
    A direct union bound yields that with probability at least $1-\l e^{-t}$, we have 
    \begin{talign}
        \sum_{i=1}^{\l}Z_i^2 \leq \sum_{i=1}^{\l} (a_i+v_i\sqrt{t})^2 
        = \sum_{i=1}^{\l} a_i^2+v_i^2 t + 2 a_i v_i \sqrt{t}
        &= \twonorm{a}^2 + \twonorm{v \sqrt{t}}^2 + 2\angles{a, v\sqrt t} \\
        &\sless{(i)} \twonorm{a}^2 + \twonorm{v \sqrt{t}}^2 + 2 \twonorm{a} \twonorm{v\sqrt{t}} \\ 
        &=(\twonorm{a}+\twonorm{v\sqrt{t}})^2,
    \end{talign}
    where step~(i) follows from Cauchy-Schwarz's inequality.
    Substituting $t \gets t + \log \l$, we conclude that
    \begin{talign}
        \P\brackets{\sqrt{\sum_{i=1}^{\l}Z_i^2} \geq \twonorm{a}+\twonorm{v}\sqrt{\log \l} + \sqrt{t}\twonorm{v} }  \leq e^{-t}.
    \end{talign}
    The lemma now follows once we note that $\wtil{\beta} = \twonorm{v}$ and
    \begin{talign}
       \wtil{\alpha}^2 = 2(\twonorm{a}^2 +  \twonorm{v}^2 \log\l) \geq  (\twonorm{a}+\twonorm{v}\sqrt{\log \l} )^2.
    \end{talign}
\begin{remark}
    \label{rem:aggk}
    \normalfont
    If the aggregated kernel satisfies $\kersplit=\sum_{\lam \in \Lam } \klam$ with each $\klam$ normalized, i.e., $\sinfnorm{\klam}=1$. In this case, \citet[Eq.~(23)]{dwivedi2022generalized} shows that for any $\lam \in \Lam$ and 
    any $f$ in the RKHS of $\klam$, we have $\norm{f}_{\klam} \leq \norm{f}_{\kersplit}$. Then, repeating arguments as in \citet[App.~F, Proof of Thm.~4]{dwivedi2022generalized}, we find that
\begin{talign}
    \ckk[\klam,\kersplit](\inputcoreset) = 2 \sqrt{\sinfnorm{\kersplit}}= 2\sqrt{|\Lam|}
    \qtext{and}
    \mkk[\klam,\kersplit] = \sqrt{|\Lam|} 
  \cdot \mkk[\klam,\klam]
\end{talign}
where $\mkk[\kernel,\kernel]$ is defined in \cref{eq:mmd_guarantee_target_k}.
\end{remark}
\section{Experiment details and supplementary results} \label{sec:app_experiments}

Here we provide the details deferred from \cref{sec:experiments} along with supplementary results.

\paragraph{Optimal four-point halving}
As discussed in \cref{sec:ktcompress}, we modify the \ktcompress algorithm of 
\citet[Ex.~4]{shetty2022distribution} slightly so that whenever an input of size $4$ is being compressed into an coreset of size $2$, we return an optimal coreset of size $2$ that minimizes $\mmd_{\kernel}$ between the input point set and the output.  This optimal coreset is also symmetrized so the either the coreset or its complement is returned with equal probability.  See \cref{algo:opthalvefour}.

\paragraph{Details on the code} 
All computations related to kernel and MMD evaluations are written using identical Cython commands to ensure both consistent runtime comparisons across methods and faster runtimes overall. 
Our code can be easily extended to cover other MMD tests and can be used as a benchmark to assess power-time trade-off curves.

\paragraph{Additional details for CTT experiments on \textsc{Gaussian} and \textsc{EMNIST}} 
\begin{itemize}[leftmargin=*]
\item The bandwidth of the Gaussian kernel is selected according the median heuristic, which is a popular heuristic in kernel methods \citep{chaudhuri2017themean} that prescribes the usage of kernels of the form $k(x,y) = \exp(-\|x-y\|^2/(2\hat{\sigma}^2))$, where $\hat{\sigma}$ is the median of the pairwise distances between different points in the sequence $\xs \cup \ys$. Unless otherwise specified, we used the median heuristic to select all bandwidths in our experiments. Since computing the median among all pairs is  expensive, we selected 512 points from $\xs$ and 512 points from $\ys$ uniformly at random and computed the median of all $1024\choose 2$ pairwise distances among them. 
\item For wild bootstrap block and incomplete tests, we use the fast computation procedure proposed by \citet{schrab2021mmd, schrab2022efficient}, which computes the terms $\hkernel(X_i,X_j,Y_i,Y_j)$ only once for each pair $i \neq j$. This is the main advantage of the wild bootstrap approach over the permutation approach. The wild bootstrap incomplete test is the same test studied by \citet{schrab2022efficient}.
\item Both in Asymp. Block I and II, the threshold is computed via the CLT using an estimate of the variance of the estimator. In Asymp. Block II, the estimate of the variance is obtained from the variance of the $n/B$ block MMD estimates. Asymp. Block II was considered as a baseline by \cite{yamada2018post}. In Asymp. Block I, the estimate of the variance is obtained by sampling a Rademacher vector length $n$ and flipping the corresponding elements of $\xs$, $\ys$ to obtain a new pair of sets of $n/B$ blocks of size $B$, and computing the empirical variance of these $n/B$ block MMD estimates. Since computations of $\hkernel(X_i,X_j,Y_i,Y_j)$ are reused, Asymp. Block I is almost as fast as Asymp. Block II. Asymp. Block I was proposed first chronologically by \cite{zaremba2013btest} in the paper that introduced block tests, although they used a permutation instead of a Rademacher variable, which made the method twice as slow.
\end{itemize}

\paragraph{Additional details for LR-CTT experiments on \textsc{Gaussian} and \textsc{EMNIST}}
\begin{itemize}[leftmargin=*]
    \item The bandwidth selection is as described above.
\end{itemize}

\paragraph{Additional details for ACTT experiments on \textsc{Blobs} and \textsc{Higgs}} %
\begin{itemize}[leftmargin=*]
    \item We use the permutation approach and take $\numperm_1 = 299$ permutations, $\numperm_2 = 200$ permutations, and $\numperm_3 = 20$ iterations.
    \item As suggested by \citet{schrab2021mmd}, the ACTT experiments set $\Lambda$ as multiples of the bandwidth given by the median heuristic. We computed the median heuristic bandwidth $\lambda_0$ as in the \ctt experiments, and we set $\Lambda = \{ 2^{-i} \lambda_0 | i \in \{0,\dots,4\} \}$. The CTT (median $\lambda$) curves in \cref{fig:blobs} are obtained by compressing using the kernel sums $(\sum_{\lambda\in\Lambda}\kernel_\lambda, \sum_{\lambda\in\Lambda}\kernel'_\lambda)$, exactly as in ACTT, but then testing using the  median heuristic \tmmd, $\mmd_{\lambda_0}(\hatxs,\hatys)$, exactly as in CTT. %
    \item The aggregated wild bootstrap incomplete test is the same test studied by \citet{schrab2022efficient}.
\end{itemize}

\begin{figure}[h]
    \centering
    \includegraphics[width=0.48\textwidth]{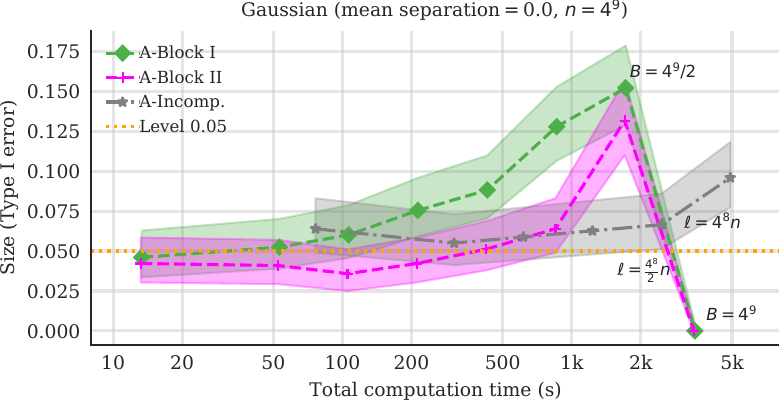}
    \includegraphics[width=0.48\textwidth]{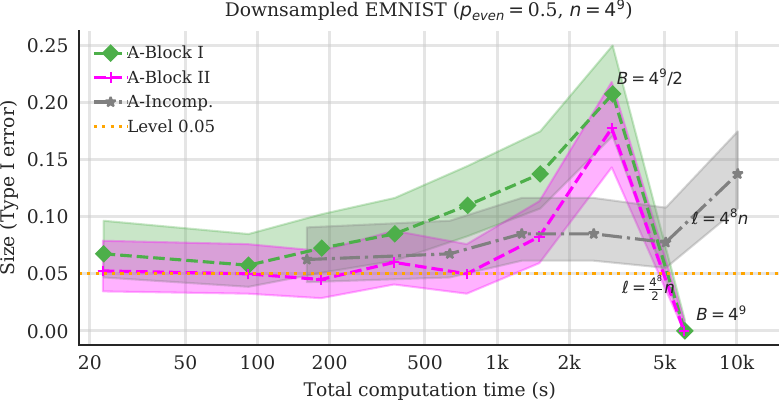}
    \\
    \includegraphics[width=0.495\textwidth]{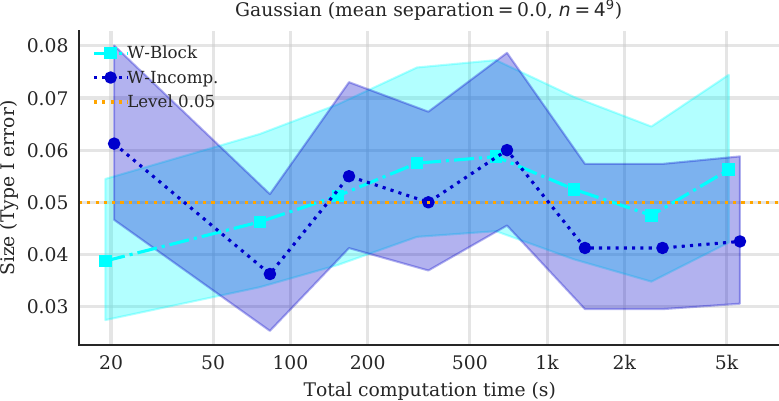}
    \includegraphics[width=0.495\textwidth]{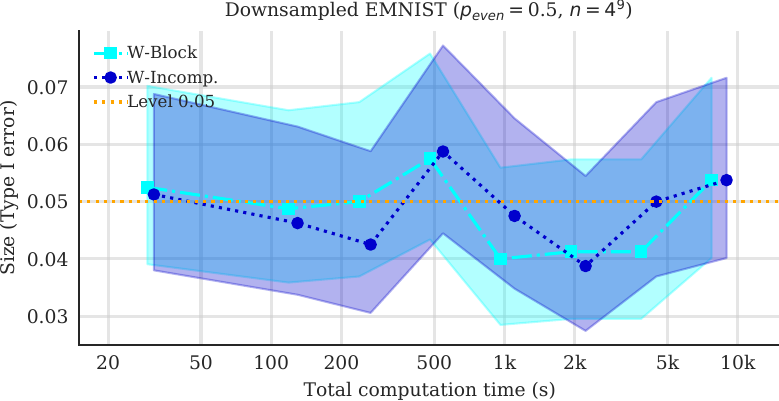} 
    \\
    \includegraphics[width=0.495\textwidth]{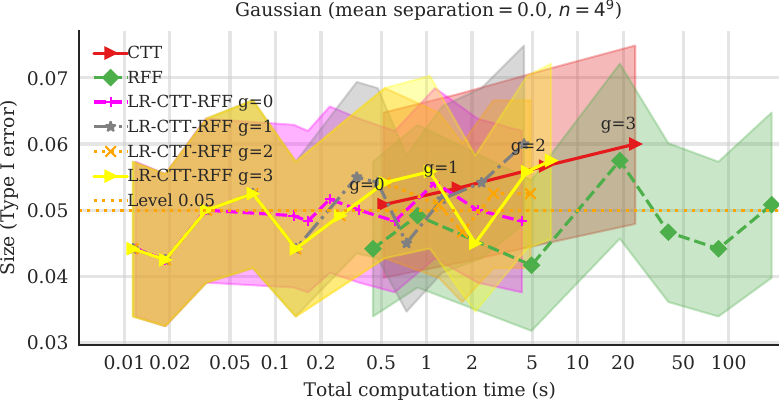}
    \includegraphics[width=0.495\textwidth]{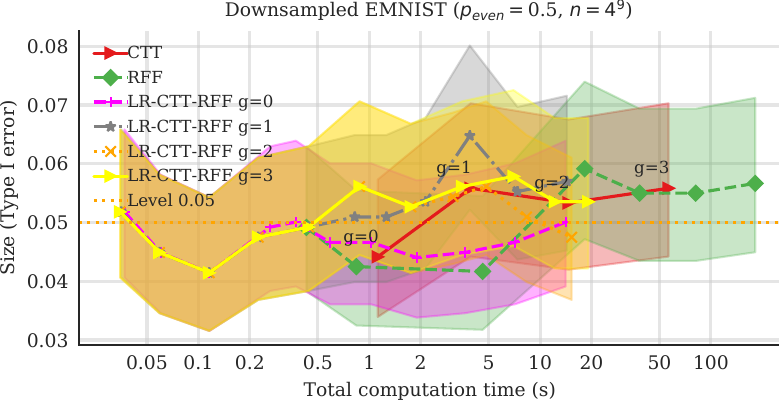} 
    \vspace{0.2cm}
    \caption{\tbf{Estimated test size} with 95\% Wilson confidence intervals in the \textsc{Gaussian} (\emph{left}) and \textsc{EMNIST} (\emph{right}) experimental settings of \cref{fig:gaussians_EMNIST}.
    \tbf{Top:} Asymptotic block and incomplete tests with $800$ (\emph{left}) and $400$ (\emph{right}) independent test repetitions. 
    \tbf{Middle:} Non-asymptotic wild bootstrap block and incomplete with $800$ independent test repetitions. 
    \tbf{Bottom:} Non-asymptotic \ctt, RFF, and \lctt with $1200$ independent test repetitions.
    } 
    \label{fig:gaussians_EMNIST_size_exact}
\end{figure}

\begin{figure}[h]
    \centering
    \includegraphics[width=0.495\textwidth]{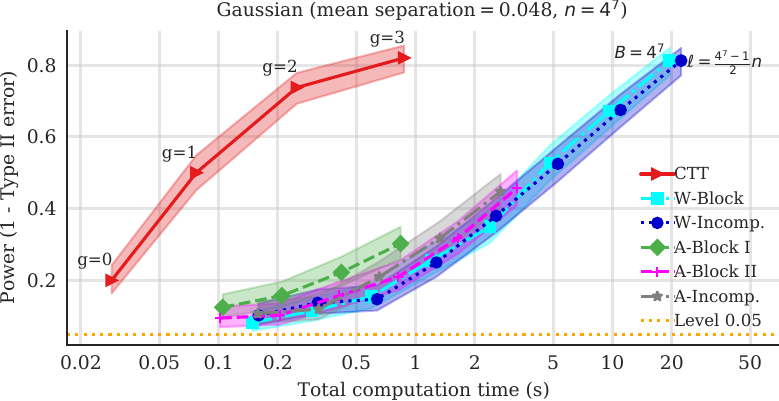}
    \includegraphics[width=0.48\textwidth]{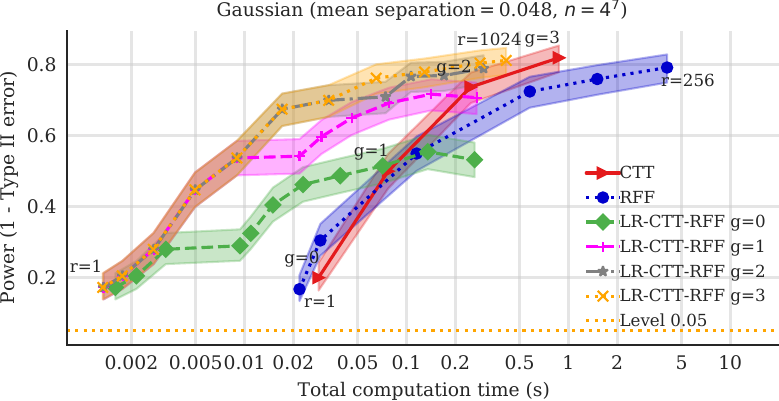}
    \\ \vspace{0.2cm}
    \includegraphics[width=0.495\textwidth]{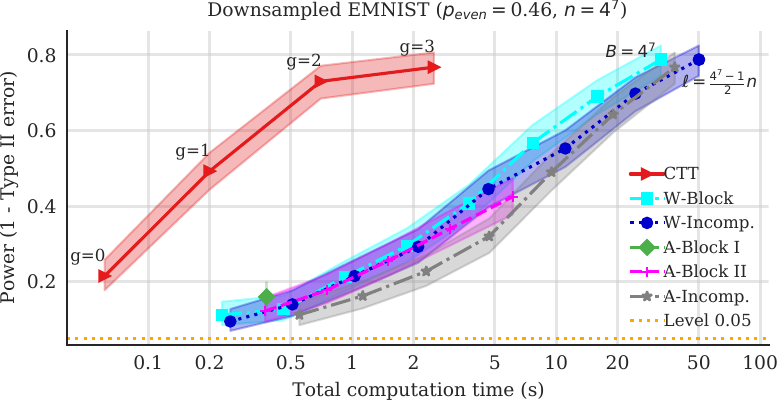}
    \includegraphics[width=0.495\textwidth]{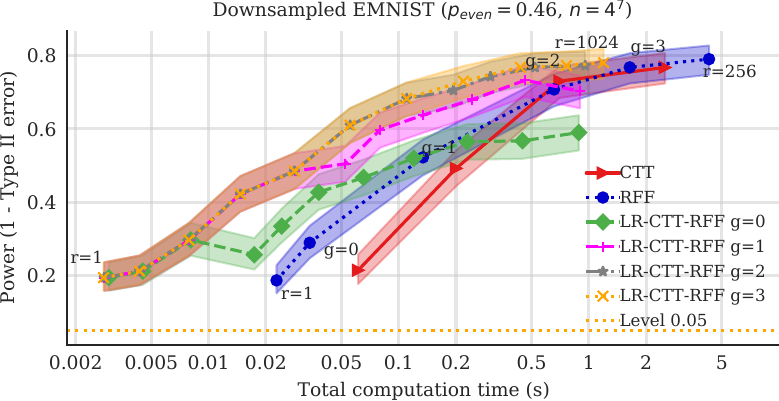}
    \caption{
    \tbf{Time-power trade-off curves} in the \textsc{Gaussian} and \textsc{EMNIST} experimental settings comparing 
    (\emph{left}) \ctt to five state-of-the-art approximate MMD tests based on subsampling 
    and (\emph{right}) \lctt to the state-of-the-art low-rank MMD test based on random Fourier features (RFF).
    These plots are like those in \cref{fig:gaussians_EMNIST}, but for a smaller sample size: $n=4^7$ instead of $n=4^9$.}
    \label{fig:power_47}
\end{figure}

\begin{figure}[h]
    \centering
    \includegraphics[width=0.48\textwidth]{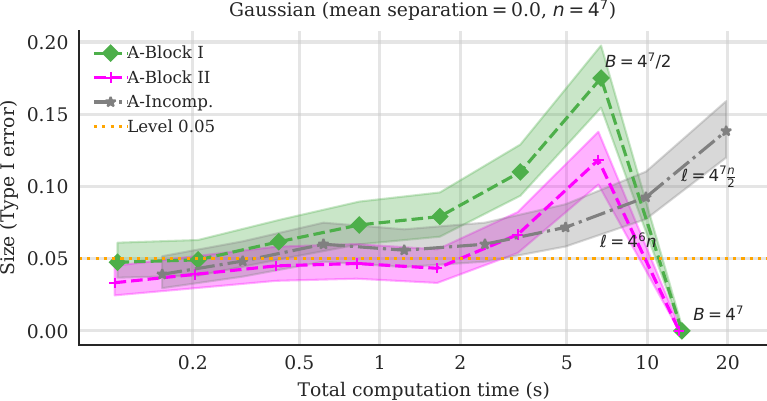}
    \includegraphics[width=0.48\textwidth]{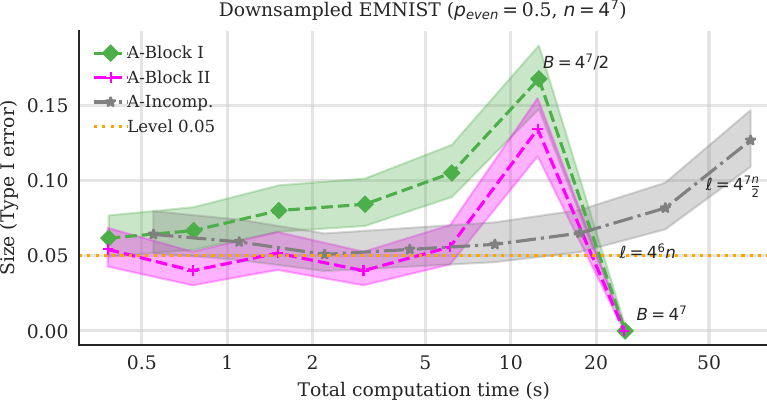}
    \\
    \includegraphics[width=0.495\textwidth]{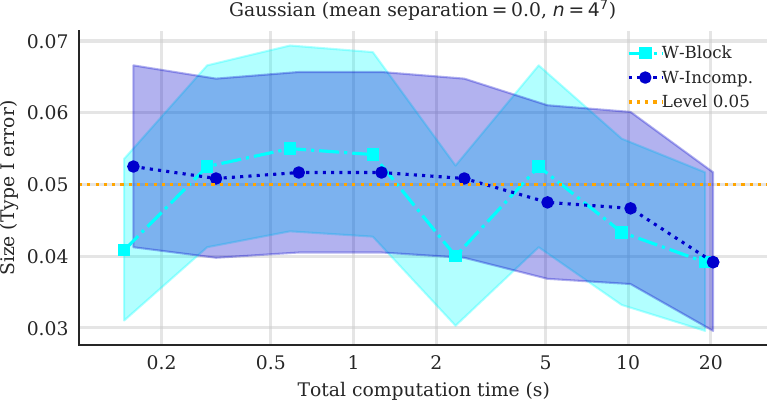}
    \includegraphics[width=0.495\textwidth]{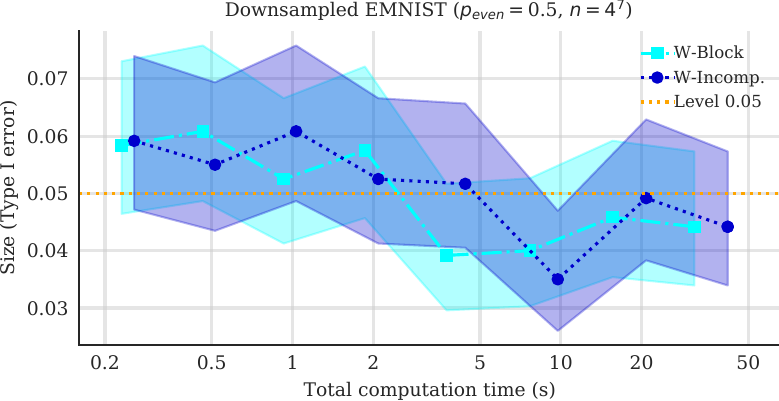}
    \\
    \includegraphics[width=0.495\textwidth]{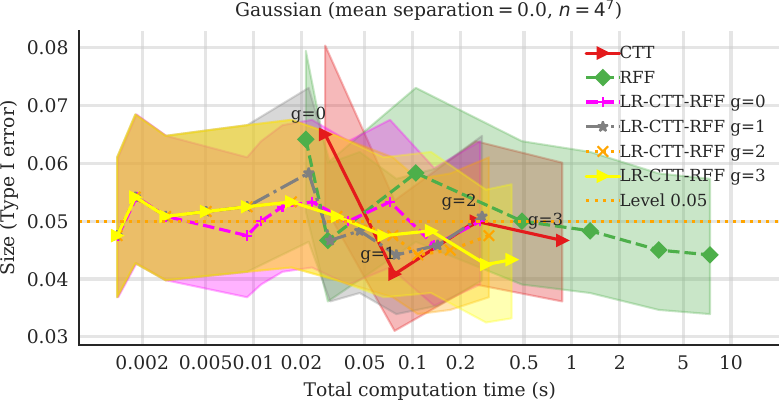}
    \includegraphics[width=0.495\textwidth]{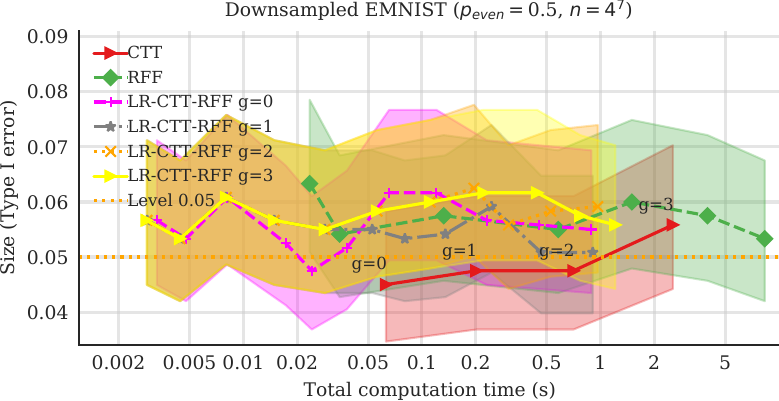} 
    \vspace{0.2cm}
    \caption{\tbf{Estimated test size} with 95\% Wilson confidence intervals in the \textsc{Gaussian} (\emph{right}) and \textsc{EMNIST} (\emph{left}) experimental settings of \cref{fig:power_47}, i.e. with $n=4^7$.
    \tbf{Top:} Asymptotic block and incomplete tests with $1200$ independent test repetitions. 
    \tbf{Middle:} Non-asymptotic wild bootstrap block and incomplete with $1200$ independent test repetitions. 
    \tbf{Bottom:} Non-asymptotic \ctt, RFF, and \lctt with $1200$ independent test repetitions.
    } 
    \label{fig:size_47}
\end{figure}

\end{document}